\documentclass[journal,twoside]{IEEEtran}
\usepackage{amsmath,amssymb,amsfonts}
\usepackage{array}
\usepackage[caption=false,font=normalsize,labelfont=sf,textfont=sf]{subfig}
\usepackage{textcomp}
\usepackage{stfloats}
\usepackage{url}
\usepackage{verbatim}
\usepackage{graphicx}
\usepackage{cite}
\usepackage[colorlinks,
            linkcolor=blue,
            anchorcolor=blue,
            citecolor=blue]{hyperref}

\usepackage[linesnumbered,ruled]{algorithm2e}
\SetNlSkip{0.5em}
\usepackage{color}
\usepackage{multirow}
\usepackage{mathtools}
\usepackage{enumitem}
\usepackage{makecell}
\usepackage{booktabs}
\usepackage[table]{xcolor}
\usepackage{threeparttable}
\newcolumntype{C}[1]{>{\centering\arraybackslash}m{#1}}

\graphicspath{{figures/}}

\allowdisplaybreaks[2]
\newtheorem{Lemma}{Lemma}
\newtheorem{Theorem}{Theorem}

\newtheorem{Remark}{Remark}

\begin{document}

\title{Normalizing Flow-Enhanced Message Passing for Multirobot Collaborative Localization}
\author{{Han Shen, \emph{Student Member, IEEE}, Guanghui Wen, \emph{Senior Member, IEEE},\\
Liangming Chen, \emph{Senior Member, IEEE}, and Ming Cao, \emph{Fellow, IEEE}

\vspace{-1ex}}

\thanks{This paper has been accepted for publication in \textit{IEEE Transactions on Robotics}. This work was supported in part by the National Natural Science Foundation of China through Grant Nos. 62325304, U2541220, U22B2046, 624B2037, in part by the Basic Research Program of Jiangsu under Grant No. BK20253020, and in part by the Jiangsu Provincial Scientific Research Center of Applied Mathematics under Grant No. BK20233002. \textit{(Corresponding author: Guanghui Wen.)}}
\thanks{Han Shen is with the Department of Systems Science, Southeast University, Nanjing 211189, China (e-mail: shenhan@seu.edu.cn).}
\thanks{Guanghui Wen is with the School of Automation, Southeast University, Nanjing 210096, China (e-mail: ghwen@seu.edu.cn).}
\thanks{Liangming Chen is with the School of Automation and Intelligent Manufacturing and Guangdong Provincial Key Laboratory of Fully Actuated System Control Theory and Technology, Southern University of Science and Technology, Shenzhen 518055, China (e-mail: chenlm6@sustech.edu.cn).}
\thanks{Ming Cao is with the Engineering and Technology Institute Groningen, Faculty of Science and Engineering, University of Groningen, AG 9747 Groningen, The Netherlands (e-mail: m.cao@rug.nl).}

\thanks{Source code is available at \href{https://github.com/han-sh/NF-enhanced-MP}\nolinkurl{https://github.com/han-sh/NF-enhanced-MP}.}

}

\markboth{IEEE Transactions on Robotics. Preprint Version. October, 2026}%
{Shen \MakeLowercase{\textit{et al.}}: Normalizing Flow-Enhanced Message Passing for Multirobot Collaborative Localization}

\makeatletter
\def\@IEEEpubidpullup{3\baselineskip}
\makeatother

\IEEEpubid{%
\parbox[b]{\textwidth}{%
\copyright{} 2026 IEEE. Personal use of this material is permitted. Permission from IEEE must be obtained for all other uses, in any current or future media, including reprinting/republishing this material for advertising or promotional purposes, creating new collective works, for resale or redistribution to servers or lists, or reuse of any copyrighted component of this work in other works.
}%
}

\maketitle

\begin{abstract}
  Accurate, robust, and adaptive localization is essential for various robotic operations. This paper proposes a parametric message passing (MP) algorithm for distributed collaborative localization, unifying Gaussian belief propagation (GBP) for robot state estimation and mean-field (MF) approximation for noise inference. Unlike conventional MP algorithms, the proposed formulation directly optimizes the natural parameters of beliefs, establishing a principled connection between MP and gradient estimators for handling non-conjugate terms arising from nonlinear measurement models. Beyond linearization and sampling, we further design a normalizing flow (NF)-based gradient estimator, enabling learnable sampling. End-to-end training tunes NF parameters according to the behavior of MP, improving estimation accuracy while reducing the required number of iterations. To support robot states involving rotations, the method is extended to Lie group spaces. Finally, the method is applied to a multirobot localization task fusing odometry, global navigation satellite system (GNSS) measurements, and inter-robot ultra wideband (UWB) ranging. Extensive simulations and experiments on autonomous surface vehicles (ASVs) demonstrate superior state estimation accuracy, enhanced robustness and adaptability under ranging outliers and varying noise statistics, as well as good generalization across different operating conditions.
\end{abstract}

\begin{IEEEkeywords}
  Multirobot systems, collaborative localization, state estimation, message passing.
\end{IEEEkeywords}

\section{Introduction}
\label{sec:introduction}
\IEEEPARstart{R}{obot} localization is a fundamental and essential task, providing pose estimates to support subsequent tasks such as decision-making, motion planning, and control. The accuracy, robustness to outliers, and adaptability to varying noise statistics are primary requirements for effective localization. With the development of multirobot systems, leveraging both onboard sensing and inter-robot relative measurements to realize collaborative localization has emerged as a promising approach for enhancing these capabilities. The theoretical feasibility has also been examined in recent studies \cite{chen1}, \cite{chen2}.

The backend state estimation for collaborative localization typically follows either optimization-based or probabilistic inference-based paradigms. Optimization-based methods mainly include distributed pose graph optimization \cite{PGO2}, \cite{PGO4} and factor graph optimization \cite{FGO2}, \cite{FGO3}. The states in the former method are primarily robot poses, while the latter method can additionally accommodate velocities, angular rates, IMU biases, and other states of interest. In \cite{ADMM}, pose graph and factor graph optimization are selected for far-field and near-field regimes, respectively. Probabilistic inference-based methods, such as belief propagation (BP) \cite{BP1}, \cite{BP2}, \cite{BP3}, \cite{BP4} and mean-field variational inference (MFVI) \cite{MFVI1}, can also be used to solve the estimation problem in a distributed manner. When MFVI is expressed in a message passing (MP) form, it corresponds to variational MP (VMP) \cite{bibVMP1}, \cite{bibVMP2}, which can then be placed alongside BP within the class of MP algorithms.

\IEEEpubidadjcol

BP and MF approximation exhibit complementary properties \cite{bibBPMF1}. MF factorizes the approximate posterior, yielding simple, often closed-form updates for conditionally conjugate exponential family models at the cost of neglecting posterior dependencies. BP, in contrast, preserves such dependencies but can become intractable even within some conjugate exponential family models. It is commonly used for Gaussian models in localization. Gaussian BP (GBP), a variant of BP which is restricted to Gaussian messages and beliefs, was extended to Lie groups in \cite{BP3}. Its robustness was enhanced using dynamic covariance scaling (DCS) \cite{DCS}. This framework was later applied to joint estimation of robot states and sensor extrinsic parameters \cite{BP4}. However, the DCS mechanism in \cite{BP3}, \cite{BP4} exhibits limited adaptability to uncertain or time-varying noise statistics. In contrast, by leveraging MF approximation's strengths in handling various conjugate exponential families, noise models beyond Gaussian assumptions can be incorporated. Then, the states and noise statistics can be jointly inferred to improve both robustness and adaptability. This has been extensively explored in single robot state estimation. The inverse Wishart models \cite{bibVBAKF2} have been adopted to improve the adaptability of Kalman filters (KFs) to uncertain or varying noise covariances. The Student's \textit{t} distribution was adopted to improve the robustness of KFs in \cite{bibVBRKF1}. To model time-varying noise that may alternate between Gaussian and heavy-tailed behaviors, \cite{bibVBRKF2} introduced a Gaussian-Student's \textit{t} mixture (GSTM) model with an adaptively estimated mixing probability. Its effectiveness has been demonstrated in various applications \cite{bibVBRKF3}, \cite{bibVBRKF4}. For collaborative localization, MFVI has been applied to jointly estimate states and uncertain noise statistics \cite{MFVI1}. However, the MF assumption in \cite{MFVI1} enforces independence among robot states, which can degrade accuracy.

To exploit the complementary strengths of BP and MF approximation, a unified MP framework was proposed in \cite{bibBPMF1}, where the update rules are obtained as fixed-point equations corresponding to the stationary points of a constrained region-based free energy approximation. This framework was subsequently applied to collaborative localization in \cite{bibBPMF2}. In \cite{bibBPMF2}, however, the MP algorithm was not designed with robustness as the primary objective; rather, the MF component incorporated relative measurement factors to preserve the same communication overhead as the VMP-based scheme. In both \cite{bibBPMF1} and \cite{bibBPMF2}, the factor nodes in the factor graph are partitioned into BP and MF parts, and beliefs of variables connected to MF factor nodes are approximated using a MF assumption. In contrast, \cite{bibBPMF3} directly partitions the variable nodes, rather than the factor nodes, into BP and MF parts, resulting in a more general update rule.

Traditional MP algorithms typically optimize over a predefined set of candidate distributions. Although their updates can be written in terms of distribution parameters, they arise implicitly from distribution updates rather than explicit parameter optimization. In contrast, natural gradient descent (NGD) methods \cite{bibNGVI1}, \cite{bibNGVI2}, \cite{bibNGVI3}, \cite{bibNGVI4} search for the optimal distribution by directly performing updates in the parameter space of the variational family. In \cite{bibNGVI1}, NGD was applied to conjugate exponential family. The work in \cite{bibNGVI2} proposed conjugate-computation VI, which uses closed-form updates for conjugate factors and stochastic gradients for non-conjugate ones. Here, a factor is termed conjugate with respect to a chosen exponential family candidate distribution if, as a function of the corresponding latent variable, it has the same exponential family form \cite{bibNGVI2}. Otherwise, it is non-conjugate. More recently, NGD has been extended to KFs \cite{bibNGVIKF1}, \cite{bibNGVIKF2} to improve adaptability and robustness, similar in spirit to \cite{bibVBAKF2}, \cite{bibVBRKF1}, \cite{bibVBRKF2}. However, \cite{bibNGVIKF1}, \cite{bibNGVIKF2} further focus on handling nonlinear dynamics and measurement models. In state estimation, nonlinear dynamic or measurement models commonly yield non-conjugate terms. With NGD, these nonlinear terms can be incorporated into stochastic gradients, effectively processed through Monte Carlo gradient estimators \cite{bibGE}, thereby improving estimation accuracy. Although these studies are successful in achieving parameterization and some of them offer principled mechanisms to integrate closed-form updates for conjugate terms with gradient estimators for non-conjugate terms, such capabilities are largely absent in existing BP algorithms. This limitation is even more pronounced in unified MP frameworks that combine BP with MF. This gap forms the first motivation of this paper.

Gradient estimators constitute core procedures of gradient-based inference methods for handling nonlinear dynamic and measurement models. The task of gradient estimators is to compute the gradient of an expectation with respect to the parameters of the distribution being integrated \cite{bibGE}. Over the past decades, several classical gradient estimators have been developed \cite{bibSFGE}, \cite{bibPGE}, \cite{bibMVGE}. These gradient estimators rely on Monte Carlo sampling. With the rise of deep learning, neural network-enhanced sampling has emerged as a promising direction. Neural importance sampling was introduced in \cite{bibNFS1}. The effectiveness of importance sampling critically depends on the choice of proposal distribution. In \cite{bibNFS1}, the proposal is constructed using a normalizing flow (NF) for improved expressiveness. Similar ideas have been explored in \cite{bibNFS2}, \cite{bibNFS3}. NFs \cite{bibNF1}, \cite{bibRealNVP}, \cite{bibNF2} offer several advantages. They provide expressive and flexible density models, and allow efficient sampling by transforming a simple base distribution through a sequence of invertible mappings. Since sampling-based gradient estimators rely on Monte Carlo integration, improving the proposal distribution directly improves estimator performance. However, the use of NFs for designing gradient estimators remains limited. Additionally, in our setting, gradient-related integrands vary with estimates and measurements. This requires the NF proposal to adapt to the current inference results. Developing NF-based gradient estimators, embedding them into MP, and enabling end-to-end training to substantially improve the state estimation accuracy and reduce algorithm iterations, constitute the second motivation of this paper.

In this paper, we propose an NF-enhanced MP algorithm to realize collaborative localization in a distributed manner. We first develop a parametric MP framework that unifies GBP and MF approximation by directly optimizing the natural parameters of exponential family candidate distributions, where GBP preserves dependencies among robot states and MF estimates noise statistics. Non-conjugate terms from nonlinear models are handled by gradient estimators, which we realize through three strategies: Linearization, sampling, and NF-based approaches. The NF is further optimized end-to-end according to the overall MP performance. Then, the method is extended to support states on Lie groups. Finally, the proposed algorithm is applied to multirobot collaborative localization, where each robot fuses measurements from onboard odometry and global navigation satellite system (GNSS), and inter-robot ranging data from ultra wideband (UWB). The contributions of this paper are summarized as follows:
\begin{enumerate}
  \item Compared with existing MP algorithms optimizing over a predefined set of candidate distributions, e.g., \cite{BP3}, \cite{bibVMP1}, \cite{bibVMP2}, \cite{bibBPMF1}, \cite{bibBPMF3}, the proposed MP algorithm directly optimizes over natural parameters of exponential family candidate distributions, establishing a principled connection between MP and gradient estimators, and further enabling explicit treatment of non-conjugate terms arising from nonlinear models through gradient estimators. In addition, the parametric structure is compatible with modern optimization and learning tools.
  \item To enable flexible and learnable sampling for gradient estimation, we develop an NF-enhanced gradient estimator, where the sampling distribution is conditioned on the current inference context, including state estimates, measurements, and uncertainty information. Together with end-to-end training, this significantly improves estimation accuracy and reduces iterations.
  \item By reformulating inference in terms of error states, the proposed MP algorithm is extended to operate on Lie groups, enhancing its suitability for real-world robotic applications.
  \item Extensive simulations and experiments on a multiple autonomous surface vehicle (ASV) system validate the accuracy, robustness, adaptability, and generalization of the proposed framework. An open-source implementation supporting user-defined factor graphs is also provided.
\end{enumerate}

The rest of this paper is organized as follows. Section \ref{sec:Preliminaries} introduces notations and preliminaries for MP, exponential family, NFs, and Lie groups. Section \ref{sec:MP} presents the proposed MP algorithm. Applications to collaborative localization are given in Section \ref{sec::multirobot_Collaborative_Localization}. Simulations and experiments are available in Section \ref{sec:simexp}. Section \ref{sec:con} concludes this paper.

\section{Preliminaries}
\label{sec:Preliminaries}
\subsection{Notations}
Let $\mathbf{0}_{n_{1}\times n_{2}}$ and $\mathbf{1}_{n_{1}\times n_{2}}$ denote $n_{1}\times n_{2}$ dimensional matrices with all elements being zeros and ones, respectively, and $\mathbf{I}_{n}$ denote an identity matrix with $n$ dimensions. We omit their subscripts for simplicity, when the matrix dimensions are either irrelevant to the discussion or can be uniquely inferred from the context. We use ${\rm det}(\mathbf{A})$ and ${\rm tr}({\mathbf{A}})$ to denote the determinant and trace of matrix $\mathbf{A}\in\mathbb{R}^{n\times n}$, respectively. If $\boldsymbol{\nu}\in\boldsymbol{\Phi}$, we write $\boldsymbol{\Phi}\setminus\boldsymbol{\nu}$ for $\boldsymbol{\Phi}\setminus\{\boldsymbol{\nu}\}$. Let ${\rm col}\left(\mathbf{x}_{1},\mathbf{x}_{2},\ldots,\mathbf{x}_{N}\right)=\left[\mathbf{x}_{1}^{\top},\mathbf{x}_{2}^{\top},\ldots,\mathbf{x}_{N}^{\top}\right]^{\top}$ denote a column vector consisting of $\mathbf{x}_{i}\in\mathbb{R}^{n_i},i=1,2,\ldots,N$. Let ${\rm diag}(\mathbf{\Lambda}_{1},\mathbf{\Lambda}_{2},\ldots,\mathbf{\Lambda}_{N})$ denote a block diagonal matrix whose main diagonal is specified by square matrices $\mathbf{\Lambda}_{1},\mathbf{\Lambda}_{2},\ldots,\mathbf{\Lambda}_{N}$. Operators $\odot$ and $\otimes$ are the Hadamard product and Kronecker product, respectively. The Kullback-Leibler (KL) divergence between distributions $q(\mathbf{x})$ and $p(\mathbf{x})$ is denoted by $D_{\rm KL}(q(\mathbf{x})||p(\mathbf{x}))$. Let ${\rm E}_{p(\mathbf{x})}[\cdot]$ denote the expectation over distribution $p(\mathbf{x})$.

\subsection{MP Algorithms}
Let $\mathbf{x}=\{\mathbf{x}_{i}|i\in\mathcal{I}\}$ denote latent variables and $\mathbf{z}$ denote an observation set. Their joint distribution admits the factorization $p(\mathbf{x},\mathbf{z})=\prod_{a\in\mathcal{A}}f_{a}(\mathbf{x}_{a})$, where function $f_{a}(\mathbf{x}_{a})$ is called a factor defined on a subset $\mathbf{x}_{a}\subseteq \mathbf{x}$, $\mathcal{A}$ is the index set of factors, and observations in $\mathbf{z}$ are treated as known parameters of factor $f_{a}(\mathbf{x}_{a})$. Such a factorization naturally arises from conditional-independence structures in probabilistic inference, where priors, state transitions, and measurement likelihoods involve only subsets of latent variables. This factorization can be represented by a factor graph, a bipartite graph whose nodes consist of variable nodes and factor nodes. There is one variable node $i\in \mathcal{I}$ for each latent variable $\mathbf{x}_{i}$, and one factor node $a\in \mathcal{A}$ for each factor $f_{a}(\mathbf{x}_{a})$. If $\mathbf{x}_{i}\in\mathbf{x}_{a}$, a variable node $i\in\mathcal{I}$ and a factor node $a\in\mathcal{A}$ are connected by an undirected edge. We use sets $\mathcal{N}(i)$ and $\mathcal{N}(a)$ to represent factor nodes and variable nodes connected to variable node $i\in\mathcal{I}$ and factor node $a\in\mathcal{A}$, respectively.

Define a candidate distribution $b(\mathbf{x})$, used to approximate the posterior $p(\mathbf{x}|\mathbf{z})$. The posterior can be inferred by minimizing the variational free energy
\begin{align}
  F\triangleq\int b(\mathbf{x})\log\frac{b(\mathbf{x})}{p(\mathbf{x},\mathbf{z})}{\rm d}\mathbf{x}\label{variational_free_energy}
\end{align}
with respect to $b(\mathbf{x})$. This follows from the identity $F=D_{\rm KL}(b(\mathbf{x})||p(\mathbf{x}|\mathbf{z}))-\log p(\mathbf{z})$, where KL divergence quantifies the discrepancy between two probability distributions, and $\log p(\mathbf{z})$ is a constant independent of $b(\mathbf{x})$. However, minimizing $F$ is intractable for many inference problems. To obtain a tractable approximation while retaining local dependencies among variables, we adopt the Bethe free energy \cite{bibBP}:
\begin{align}
    F_{\rm B}\triangleq&\sum_{a\in\mathcal{A}}\int b_{a}\left(\mathbf{x}_{a}\right)\log\frac{b_{a}\left(\mathbf{x}_{a}\right)}{f_{a}\left(\mathbf{x}_{a}\right)}{\rm d}\mathbf{x}_{a}\notag\\
    &-\sum_{i\in\mathcal{I}}\left(|\mathcal{N}(i)|-1\right)\int b_{i}\left(\mathbf{x}_{i}\right)\log b_{i}\left(\mathbf{x}_{i}\right){\rm d}\mathbf{x}_{i},\label{Bethe_free_energy}
\end{align}
where $b_{a}\left(\mathbf{x}_{a}\right)$ and $b_{i}\left(\mathbf{x}_{i}\right)$ are candidate distributions, called beliefs, used to approximate $p\left(\mathbf{x}_{a}|\mathbf{z}\right)$ and $p\left(\mathbf{x}_{i}|\mathbf{z}\right)$, respectively, and $|\mathcal{N}(i)|$ denotes the cardinality of the set $\mathcal{N}(i)$, i.e., the number of elements it contains. The Bethe free energy provides the classical variational foundation of BP. It becomes exact for tree-structured factor graphs \cite{bibBP}. Additionally, normalization constraints $\int b_{a}\left(\mathbf{x}_{a}\right){\rm d}\mathbf{x}_{a}=1,\forall a\in\mathcal{A}$ and marginalization constraints $\int b_{a}\left(\mathbf{x}_{a}\right){\rm d}\mathbf{x}_{a}\setminus\mathbf{x}_{i}=b_{i}\left(\mathbf{x}_{i}\right),\forall a\in\mathcal{A},i\in\mathcal{N}(a)$ should be satisfied. Then, we can construct a Lagrangian function based on cost function (\ref{Bethe_free_energy}) and equality constraints, whose stationary points are \cite{bibBP}
\begin{align}
  b_{a}\left(\mathbf{x}_{a}\right)\propto&f_{a}\left(\mathbf{x}_{a}\right)\prod_{i\in\mathcal{N}(a)}n_{i\to a}\left(\mathbf{x}_{i}\right), \forall a\in\mathcal{A}\label{beliefupdate_a}\\
  b_{i}\left(\mathbf{x}_{i}\right)\propto&\prod_{a\in\mathcal{N}(i)}m_{a\to i}\left(\mathbf{x}_{i}\right), \forall i\in\mathcal{I}\label{beliefupdate_i}
\end{align}
with functions called messages
\begin{align}
  m_{a\to i}\left(\mathbf{x}_{i}\right)=&\int f_{a}\left(\mathbf{x}_{a}\right)\prod_{j\in\mathcal{N}(a)\setminus i}n_{j\to a}\left(\mathbf{x}_{j}\right){\rm d}\mathbf{x}_{a}\setminus\mathbf{x}_{i},\label{message_factor2var_BP}\\
  n_{i\to a}\left(\mathbf{x}_{i}\right)=&\prod_{c\in\mathcal{N}(i)\setminus a}m_{c\to i}\left(\mathbf{x}_{i}\right).\label{message_var2factor}
\end{align}
Using fixed point iterations to solve (\ref{beliefupdate_a}) and (\ref{beliefupdate_i}) yields BP algorithm. When the factor graph satisfies a tree structure, BP is an exact inference algorithm. In the presence of cycles, BP generally achieves approximations of posteriors \cite{bibBP}.

When we substitute a factorization constraint, also known as the MF assumption, $b(\mathbf{x})=\prod_{i\in\mathcal{I}}b_{i}(\mathbf{x}_{i})$ into variational free energy (\ref{variational_free_energy}), and minimize (\ref{variational_free_energy}) with respect to $b_{i}(\mathbf{x}_{i})$, we obtain the stationary conditions of MFVI, given by \cite{bibMFVI}
\begin{align}
  \log{b_{i}\left(\mathbf{x}_{i}\right)}=&{\rm E}_{b_{-i}\left(\mathbf{x}_{-i}\right)}\left[\log{p(\mathbf{x},\mathbf{z})}\right]+c_{i},\forall i\in\mathcal{I}\label{MFVI}
\end{align}
where $\mathbf{x}_{-i}=\mathbf{x}\setminus\mathbf{x}_{i}$, $b_{-i}\left(\mathbf{x}_{-i}\right)=\prod_{j\in\mathcal{I}\setminus i}b_{j}\left(\mathbf{x}_{j}\right)$, and $c_{i}$ is a constant independent of $\mathbf{x}_{i}$. As summarized in \cite{bibVMP2}, \cite{bibBPMF1}, equation (\ref{MFVI}) can be expressed in a VMP form, which is updated based on messages exchanged between variables and factors, similar to BP. This passing process is carried out directly among robots, without relying on any central node, enabling MP algorithms to support distributed computation and decentralized fusion in collaborative localization.

\subsection{Exponential Family}
The exponential family is a class of distributions:
\begin{align}
  p(\mathbf{x};\boldsymbol{\lambda})\triangleq h(\mathbf{x})\exp(\left<\boldsymbol{\lambda},\mathcal{T}(\mathbf{x})\right>-A(\boldsymbol{\lambda})),\label{ef}
\end{align}
where $h(\mathbf{x})$ is the base measure that could be a constant, such as $1$, $\boldsymbol{\lambda}$ is natural parameter, $\mathcal{T}(\mathbf{x})$ is sufficient statistics, $A(\boldsymbol{\lambda})$ is the log partition function, and $\left<\cdot,\cdot\right>$ is an inner product. Let expectation parameter $\boldsymbol{\mu}\triangleq{\rm E}_{p(\mathbf{x};\boldsymbol{\lambda})}[\mathcal{T}(\mathbf{x})]$, which is also a function of $\boldsymbol{\lambda}$. For a minimal exponential family, it is a one-to-one mapping between $\boldsymbol{\lambda}$ and $\boldsymbol{\mu}$ \cite{bibNGVI4}. Various common distributions are in the exponential family. We give some examples in Appendix \ref{Examples_Exponential_Families}. These distributions will be used for later derivations. Let entropy $H(p(\mathbf{x};\boldsymbol{\lambda}))\triangleq-{\rm E}_{p(\mathbf{x};\boldsymbol{\lambda})}[\log p(\mathbf{x};\boldsymbol{\lambda})]$.

\subsection{Real NVP}
Given vectors $\boldsymbol{\epsilon}\in E$ and $\mathbf{y}\in Y$, a bijection $T:Y\rightarrow E$, and $\boldsymbol{\epsilon}=T(\mathbf{y})$, we have the change of variable formula:
\begin{align}
  p(\boldsymbol{\epsilon})=p(\mathbf{y})\left|{\rm det}\left(\frac{\partial T(\mathbf{y})}{\partial \mathbf{y}^{\top}}\right)\right|^{-1},\label{p2p}
\end{align}
where $p(\boldsymbol{\epsilon})$ and $p(\mathbf{y})$ denote the corresponding
probability densities over $E$ and $Y$, respectively.

NFs construct invertible mappings whose Jacobians are tractable, enabling probability densities to be transformed via the change of variable formula (\ref{p2p}). Among them, the Real NVP \cite{bibRealNVP} is particularly attractive due to its simple coupling-layer structure and the efficient computation of both forward mappings and log-determinants. The Real NVP consists of $L_{\rm NF}$ affine coupling layers. We use function $\mathbf{y}^{j}=F^{j}(\mathbf{y}^{j-1})$ to denote the $j$th affine coupling layer with $\mathbf{y}^{0}=\mathbf{y}$. In \cite{bibRealNVP}, affine coupling layer $\mathbf{y}^{j}=F^{j}(\mathbf{y}^{j-1})$ is specified by
\begin{align}
  \mathbf{y}^{j}=&\left(\mathbf{1}-\mathbf{b}^{j}\right)\odot\left(\mathbf{y}^{j-1}\odot\exp\left(s^{j}\left(\mathbf{b}^{j}\odot\mathbf{y}^{j-1}\right)\right)\right.\notag\\
  &\left.+t^{j}\left(\mathbf{b}^{j}\odot\mathbf{y}^{j-1}\right)\right)+\mathbf{b}^{j}\odot\mathbf{y}^{j-1},\label{acl}
\end{align}
where vector $\mathbf{b}^{j}$ is a binary mask. In our implementation, when the layer index $j$ is odd, $\mathbf{b}^{j}$ assigns the value $1$ to the odd positions and $0$ to the even positions. When $j$ is even, the assignment is reversed. Mappings $s^{j}$ and $t^{j}$ are multilayer perceptrons (MLPs). Letting $\mathbf{z}^{j}=(\mathbf{1}-\mathbf{b}^{j})\odot s^{j}\left(\mathbf{b}^{j}\odot\mathbf{y}^{j-1}\right)$,
\begin{align}
  \left|{\rm det}\left(\frac{\partial \mathbf{y}^{j}}{\partial (\mathbf{y}^{j-1})^{\top}}\right)\right|=\exp\left(\mathbf{1}^{\top}\mathbf{z}^{j}\right).\notag
\end{align}
The overall transformation of Real NVP is constructed by $T=F^{L_{\rm NF}}\circ\ldots\circ F^{1}$. Then,
\begin{align}
  \left|{\rm det}\left(\frac{\partial T(\mathbf{y})}{\partial \mathbf{y}^{\top}}\right)\right|=\prod_{j=1}^{L_{\rm NF}}\exp\left(\mathbf{1}^{\top}\mathbf{z}^{j}\right).\notag
\end{align}
Up to this point, we have specified the mapping $\boldsymbol{\epsilon}=T(\mathbf{y})$ and the coefficient in (\ref{p2p}). The Real NVP will later be used to construct proposal distributions for sampling.

\subsection{Lie Groups}
The special orthogonal group ${\rm SO}(n)$ is a typical Lie group, defined as ${\rm SO}(n)\triangleq\{\mathbf{X}\in\mathbb{R}^{n\times n}|\mathbf{X}^{\top}\mathbf{X}=\mathbf{I}_{n},{\rm det}(\mathbf{X})=1\}$. For $n=3$, ${\rm SO}(3)$ can be used as the domain of 3-D rotation matrices. For a rotation matrix $\mathbf{R}\in{\rm SO}(3)$ and its corresponding rotation vector $\mathbf{r}\in\mathbb{R}^{3}$, they can be mapped to each other by $\mathbf{R}\triangleq{\rm Exp}\left(\mathbf{r}\right)$, whose closed form is known as Rodrigues' formula. The local inverse of ${\rm Exp}\left(\mathbf{r}\right)$ is denoted by $\mathbf{r}\triangleq{\rm Log}\left(\mathbf{R}\right)$. To describe increments between elements on Lie groups, we define right operators $\oplus$ and $\ominus$:
\begin{align*}
  \mathbf{R}\oplus\boldsymbol{\tau}\triangleq\mathbf{R}{\rm Exp}\left(\boldsymbol{\tau}\right),\mathbf{R}_{2}\ominus\mathbf{R}_{1}\triangleq{\rm Log}\left(\mathbf{R}_{1}^{\top}\mathbf{R}_{2}\right),
\end{align*}
where $\mathbf{R},\mathbf{R}_{1},\mathbf{R}_{2}\in{\rm SO}(3)$ and $\boldsymbol{\tau}\in\mathbb{R}^{3}$. Define a product manifold $\mathcal{M}={\rm SO}(3)\times\mathbb{R}^{3}$. The operators $\boxplus$ and $\boxminus$ act right plus or minus to each block of the product manifold, respectively. For example, for product manifold $\mathcal{M}$ and $\mathbf{a},\mathbf{b}\in\mathbb{R}^{3}$,
\begin{align*}
  \left[\begin{matrix}
    \mathbf{R}\\
    \mathbf{a}
  \end{matrix}\right]\boxplus\left[\begin{matrix}
    \boldsymbol{\tau}\\
    \mathbf{b}
  \end{matrix}\right]=\left[\begin{matrix}
    \mathbf{R}\oplus\boldsymbol{\tau}\\
    \mathbf{a}+\mathbf{b}
  \end{matrix}\right],\
  \left[\begin{matrix}
    \mathbf{R}_{2}\\
    \mathbf{a}
  \end{matrix}\right]\boxminus\left[\begin{matrix}
    \mathbf{R}_{1}\\
    \mathbf{b}
  \end{matrix}\right]=\left[\begin{matrix}
    \mathbf{R}_{2}\ominus\mathbf{R}_{1}\\
    \mathbf{a}-\mathbf{b}
  \end{matrix}\right].
\end{align*}

Uncertainties on Lie groups can be defined by expressing perturbations in the tangent space and representing them in a Euclidean coordinate vector. Consider a product manifold $\mathcal{M}_{a}$ consisting of multiple ${\rm SO}(3)$ and $\mathbb{R}^{m}$. Given an estimate $\bar{\mathcal{X}}\in\mathcal{M}_{a}$, we define perturbation vector $\boldsymbol{\delta}$ with $p(\boldsymbol{\delta})={\rm N}(\boldsymbol{\delta};\mathbf{0},\boldsymbol{\Sigma})$, and construct random variable $\mathcal{X}=\bar{\mathcal{X}}\boxplus\boldsymbol{\delta}\in\mathcal{M}_{a}$. For convenience, we denote this tangent-space Gaussian perturbation model by ${\rm N}(\mathcal{X};\bar{\mathcal{X}},\boldsymbol{\Sigma})$. This notation does not imply that $\mathcal{X}$ follows a Euclidean Gaussian distribution, and its density is indirectly defined by $\boldsymbol{\delta}$.

\section{NF-Enhanced MP Algorithm}
\label{sec:MP}

The overall framework of the proposed NF-enhanced MP algorithm is presented in Fig. \ref{fig_Framework}. We begin by constructing a factor graph based on measurements and the assumed noise models. The variables are partitioned into BP and MF parts, and factors are categorized according to the types of variables they connect. Next, we propose a parametric MP framework that directly performs inference in the natural parameter space of candidate exponential family distributions. This framework iteratively updates messages between factors and variables, as well as beliefs of both factors and variables. Since the iterations involve gradient estimation, which is also the main step for handling non-conjugate terms arising from nonlinear measurement models, we design three types of gradient estimators to meet different requirements, including linearization, sampling, and NF-based approaches. The NF parameters are trained end-to-end by minimizing a loss function describing the discrepancy between ground truth and estimates. Finally, the proposed MP algorithm is extended to Lie group space.

\begin{figure}[!tb]
  \centering
  \includegraphics[scale=1]{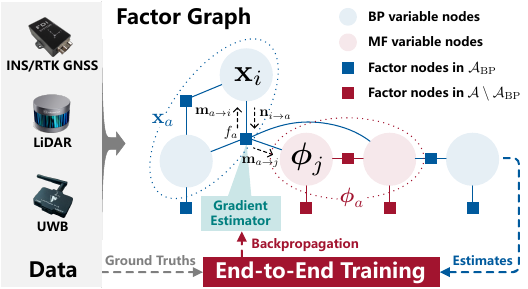}
  \caption{The proposed MP framework. The variables in BP and MF parts that serve as arguments of factor $f_{a}$ constitute $\mathbf{x}_{a}$ and $\boldsymbol{\phi}_{a}$, respectively. The factor nodes connected to BP variable nodes constitute set $\mathcal{A}_{\rm BP}\subseteq \mathcal{A}$, with $\mathcal{A}$ consisting of all factor nodes. The framework employs linearization, sampling, and NF-based gradient estimators and updates messages and beliefs to realize estimation. NF parameters are trained by minimizing the discrepancy between ground truth and estimates.}
  \label{fig_Framework}
\end{figure}

\subsection{Parametric MP Framework}
We divide the variable index set $\mathcal{I}$ into a BP part and a MF part, denoted by $\mathcal{I}_{\rm BP}$ and $\mathcal{I}_{\rm MF}$, respectively, and $\mathcal{I}_{\rm BP}\cap\mathcal{I}_{\rm MF}=\emptyset$, $\mathcal{I}=\mathcal{I}_{\rm BP}\cup\mathcal{I}_{\rm MF}$. The variables in BP and MF parts are denoted by $\mathbf{x}_{i},i\in\mathcal{I}_{\rm BP}$ and $\boldsymbol{\phi}_{i},i\in\mathcal{I}_{\rm MF}$, respectively. Let $\mathbf{x}=\{\mathbf{x}_{i}|i\in\mathcal{I}_{\rm BP}\}$, $\boldsymbol{\phi}=\{\boldsymbol{\phi}_{i}|i\in\mathcal{I}_{\rm MF}\}$, and $\mathbf{z}$ denote the measurement set. The joint distribution $p(\mathbf{x},\boldsymbol{\phi},\mathbf{z})=\prod_{a\in\mathcal{A}}f_{a}(\mathbf{x}_{a},\boldsymbol{\phi}_{a})$, where $f_{a}(\mathbf{x}_{a},\boldsymbol{\phi}_{a})$ is a factor, $\mathbf{x}_{a}$ and $\boldsymbol{\phi}_{a}$ are (possibly empty) subsets of $\mathbf{x}$ and $\boldsymbol{\phi}$, respectively. When $\mathbf{x}_{a}$ or $\boldsymbol{\phi}_{a}$ is empty, the corresponding argument is omitted and the factor reduces to a function of the remaining variable. Following the definition of a factor graph, a variable node $i\in\mathcal{I}_{\rm BP}$ or $i\in\mathcal{I}_{\rm MF}$ is connected to a factor node $a\in\mathcal{A}$ whenever the associated variable appears as an argument of that factor. We use sets $\mathcal{N}(a)$ and $\mathcal{N}(i)$ to denote nodes connected to factor node $a$ and variable node $i$, respectively. We further define sets $\mathcal{N}_{\rm BP}(a)=\mathcal{N}(a)\cap\mathcal{I}_{\rm BP}$ and $\mathcal{N}_{\rm MF}(a)=\mathcal{N}(a)\cap\mathcal{I}_{\rm MF}$. We use $\mathcal{A}_{\rm BP}=\{a\in\mathcal{A}|\mathcal{N}_{\rm BP}(a)\neq\emptyset\}$ to represent factor nodes connected to BP variable nodes. Fig. \ref{fig_Framework} uses an example to illustrate the above notations.

Invoking the Bethe free energy definition (\ref{Bethe_free_energy}) and rewriting it under the notation of this section, we have
\begin{align}
    F_{\rm B}=&\sum_{a\in\mathcal{A}}\int b_{a}\left(\mathbf{x}_{a},\boldsymbol{\phi}_{a}\right)\log\frac{b_{a}\left(\mathbf{x}_{a},\boldsymbol{\phi}_{a}\right)}{f_{a}\left(\mathbf{x}_{a},\boldsymbol{\phi}_{a}\right)}{\rm d}\mathbf{x}_{a}{\rm d}\boldsymbol{\phi}_{a}\notag\\
    &-\sum_{i\in\mathcal{I}_{\rm BP}}\left(|\mathcal{N}(i)|-1\right)\int b_{i}\left(\mathbf{x}_{i}\right)\log b_{i}\left(\mathbf{x}_{i}\right){\rm d}\mathbf{x}_{i}\notag\\
    &-\sum_{i\in\mathcal{I}_{\rm MF}}\left(|\mathcal{N}(i)|-1\right)\int b_{i}\left(\boldsymbol{\phi}_{i}\right)\log b_{i}\left(\boldsymbol{\phi}_{i}\right){\rm d}\boldsymbol{\phi}_{i},\label{Bethe_free_energy_BPMF}
\end{align}
where beliefs $b_{a}(\mathbf{x}_{a},\boldsymbol{\phi}_{a})$, $b_{i}(\mathbf{x}_{i})$, and $b_{i}(\boldsymbol{\phi}_{i})$ serve as candidate approximations to posteriors $p(\mathbf{x}_{a},\boldsymbol{\phi}_{a}|\mathbf{z})$, $p(\mathbf{x}_{i}|\mathbf{z})$, and $p(\boldsymbol{\phi}_{i}|\mathbf{z})$, respectively.

Following \cite{bibBPMF3}, to accommodate the advantages of BP and MF approximation, we introduce factorization constraints
\begin{align}
  \label{factorization_constraints}
  b_{a}\left(\mathbf{x}_{a},\boldsymbol{\phi}_{a}\right)=b_{a}\left(\mathbf{x}_{a}\right)\prod_{i\in\mathcal{N}_{\rm MF}(a)}b_{i}\left(\boldsymbol{\phi}_{i}\right),\forall a\in\mathcal{A}
\end{align}
When $\mathbf{x}_{a}$ or $\boldsymbol{\phi}_{a}$ is empty, $b_{a}(\mathbf{x}_{a})$ or $b_{i}(\boldsymbol{\phi}_{i})$ is interpreted as the identity element of the product, i.e., equals $1$.

Substituting factorization constraints (\ref{factorization_constraints}) into (\ref{Bethe_free_energy_BPMF}) and expressing certain terms by entropy, we have
\begin{align}
    F_{\rm B}=&-\sum_{a\in\mathcal{A}_{\rm BP}}H(b_{a}\left(\mathbf{x}_{a}\right))-\sum_{i\in\mathcal{I}_{\rm MF}}H(b_{i}\left(\boldsymbol{\phi}_{i}\right))-\sum_{a\in\mathcal{A}}U_{a}\notag\\
    &+\sum_{i\in\mathcal{I}_{\rm BP}}\left(|\mathcal{N}(i)|-1\right)H(b_{i}\left(\mathbf{x}_{i}\right)),\label{loss_function}
\end{align}
where
\begin{align}
  U_{a}=\int b_{a}\left(\mathbf{x}_{a}\right)b_{a}(\boldsymbol{\phi}_{a})\log f_{a}\left(\mathbf{x}_{a},\boldsymbol{\phi}_{a}\right){\rm d}\mathbf{x}_{a}{\rm d}\boldsymbol{\phi}_{a},\label{U_a}
\end{align}
and $b_{a}(\boldsymbol{\phi}_{a})=\prod_{i\in\mathcal{N}_{\rm MF}(a)}b_{i}\left(\boldsymbol{\phi}_{i}\right)$. We introduce marginalization constraints to maintain dependence between BP variables:
\begin{align}
  \int b_{a}\left(\mathbf{x}_{a}\right){\rm d}\mathbf{x}_{a}\setminus\mathbf{x}_{i}=b_{i}\left(\mathbf{x}_{i}\right),\forall a\in\mathcal{A}_{\rm BP},i\in\mathcal{N}_{\rm BP}(a).\label{marginalization_constraints}
\end{align}

\begin{Remark}
  In the subsequent collaborative localization task, $\mathbf{x}_{i},i\in\mathcal{I}_{\rm BP}$ and $\boldsymbol{\phi}_{i},i\in\mathcal{I}_{\rm MF}$ will be states and noise parameters, respectively. The factorization in (\ref{factorization_constraints}) facilitates the inference of noise parameters. Meanwhile, (\ref{marginalization_constraints}) preserves the dependence among states, ensuring high estimation accuracy.
\end{Remark}

Traditional MP algorithms view beliefs as optimization variables. In many practical problems, the beliefs can be modeled by an exponential family distribution, which is also the form targeted by many state estimators. We set $b_{a}(\mathbf{x}_{a}),\forall a\in\mathcal{A}_{\rm BP}$ as a minimal exponential family distribution with natural parameter $\boldsymbol{\lambda}_{a}$, sufficient statistics $\mathcal{T}_{a}(\mathbf{x}_{a})$, expectation parameter $\boldsymbol{\mu}_{a}$, and set $b_{i}(\mathbf{x}_{i})$ and $b_{i}(\boldsymbol{\phi}_{i})$ as minimal exponential family distributions with natural parameter $\boldsymbol{\lambda}_{i}$, sufficient statistics $\mathcal{T}_{i}(\mathbf{x}_{i})$ ($\mathcal{T}_{i}(\boldsymbol{\phi}_{i})$ for $b_{i}(\boldsymbol{\phi}_{i})$), and expectation parameter $\boldsymbol{\mu}_{i}$.

Since this paper focuses on state estimation, where states are generally modeled by Gaussian distributions, the beliefs in BP part, $b_{a}(\mathbf{x}_{a}),a\in\mathcal{A}_{\rm BP}$ and $b_{i}(\mathbf{x}_{i}),i\in\mathcal{I}_{\rm BP}$, are restricted to Gaussian distributions. For rigor, the subsequent consideration of $\mathbf{x}_{a}$ is a vector consisting of all $\mathbf{x}_{i},i\in\mathcal{N}_{\rm BP}(a)$, whose order can be arbitrary. Following the exponential family form of a Gaussian distribution in Appendix \ref{Examples_Exponential_Families}, marginalization constraints in (\ref{marginalization_constraints}) are equivalent to
\begin{align}
  \boldsymbol{\mu}_{i}=\mathbf{L}_{a,i}\boldsymbol{\mu}_{a},\forall a\in\mathcal{A}_{\rm BP},i\in\mathcal{N}_{\rm BP}(a),\label{eq_constraint}
\end{align}
where $\mathbf{L}_{a,i}={\rm diag}(\mathbf{\Gamma}_{a,i},\mathbf{\Gamma}_{a,i}\otimes\mathbf{\Gamma}_{a,i})$ and $\mathbf{\Gamma}_{a,i}$ satisfies $\mathbf{x}_{i}=\mathbf{\Gamma}_{a,i}\mathbf{x}_{a}$. 

For exponential family whose natural and expectation parameters are linked by a one-to-one mapping, an optimization problem posed in the natural parameter space can be reformulated in the expectation parameter space. This preserves the objective while facilitating gradient computation. Let $\boldsymbol{\mu}$ be a set consisting of all $\boldsymbol{\mu}_{a},a\in\mathcal{A}_{\rm BP}$ and $\boldsymbol{\mu}_{i},i\in\mathcal{I}$. We construct the optimization problem
\begin{align}
  \min_{\boldsymbol{\mu}}F_{\rm B},\quad{\rm s.t.}(\ref{eq_constraint})\label{optimization_problem}
\end{align}

The Lagrangian function of (\ref{optimization_problem}) is 
\begin{align}
  \label{Lagrangian}
    \mathcal{L}=&F_{\rm B}+\sum_{a\in\mathcal{A}_{\rm BP}}\sum_{i\in\mathcal{N}_{\rm BP}(a)}\boldsymbol{\lambda}_{a,i}^{\top}\left(\boldsymbol{\mu}_{i}-\mathbf{L}_{a,i}\boldsymbol{\mu}_{a}\right),
\end{align}
where $\boldsymbol{\lambda}_{a,i}$ is a Lagrange multiplier.

\begin{Theorem}
  \label{Theorem_stationary_points}
When minimal exponential family beliefs $b_{a}(\mathbf{x}_{a})$ for $a\in\mathcal{A}_{\rm BP}$, $b_{i}(\mathbf{x}_{i})$ for $i\in\mathcal{I}_{\rm BP}$, and $b_{i}(\boldsymbol{\phi}_{i})$ for $i\in\mathcal{I}_{\rm MF}$ have constant base measures, the stationary points of (\ref{Lagrangian}) are equivalent to fixed points fulfilling
\begin{align}
  \boldsymbol{\lambda}_{a}=&\nabla_{\boldsymbol{\mu}_{a}}U_{a}+\sum_{i\in\mathcal{N}_{\rm BP}(a)}\mathbf{L}_{a,i}^{\top}\mathbf{n}_{i\to a},\ \forall a\in\mathcal{A}_{\rm BP}\label{lambda_a_bpmf}\\
  \boldsymbol{\lambda}_{i}=&\sum_{a\in\mathcal{N}(i)}\mathbf{m}_{a\to i},\ \forall i\in\mathcal{I}\label{lambda_i_bpmf}
\end{align}
with messages given by vectors
\begin{align}
  \mathbf{n}_{i\to a}=&\sum_{c\in\mathcal{N}(i)\setminus a}\mathbf{m}_{c\to i},\ \forall i\in\mathcal{I}_{\rm BP},a\in\mathcal{N}(i)\label{m2n_bpmf}\\
  \mathbf{m}_{a\to i}=&\nabla_{\boldsymbol{\mu}_{i}}U_{a},\forall a\in\mathcal{A},i\in\mathcal{N}_{\rm MF}(a)\label{belief2m_mf_bpmf}\\
  \mathbf{m}_{a\to i}=&g\left(\mathbf{L}_{a,i}f\left(\boldsymbol{\lambda}_{a}\right)\right)-\mathbf{n}_{i\to a},\ \forall a\in\mathcal{A},i\in\mathcal{N}_{\rm BP}(a)\label{n2m_bpmf}
\end{align}
and vice versa, where function $f$ denotes the mapping from natural parameter to its corresponding expectation parameter, and $g$ is its inverse mapping.
\end{Theorem}
\textit{Proof:} See Appendix \ref{Proof_stationary_points}.
$\hfill\blacksquare$

The examples of exponential family given in Appendix \ref{Examples_Exponential_Families} all satisfy the prerequisites of Theorem \ref{Theorem_stationary_points}, admitting later usage of these distributions.  Inspired by \cite{bibBPMF1}, we use fixed point iterations to solve natural parameters and messages. Let $L$ denote the max iterations, and superscript $(l)$ denote the $l$th iteration. The proposed parametric MP is summarized in Algorithm \ref{Alg_BPMF}, also illustrated in Fig. \ref{fig_Framework}. Step \ref{alg_Initialize} is designed to satisfy equality constraints in (\ref{optimization_problem}), which means the initial mean of $\mathbf{x}_{a}$ consists of the mean of $\mathbf{x}_{i},i\in\mathcal{N}_{\rm BP}(a)$ and the initial covariance is a block diagonal matrix with diagonal blocks being covariance matrices of $\mathbf{x}_{i},i\in\mathcal{N}_{\rm BP}(a)$. Step \ref{alg_Initialize_msg} is used to initialize messages, following that in \cite{bibBPMF1}.

\begin{algorithm}[tb]
  \small
  \hsize=\linewidth
  \caption{Parametric MP}
  \label{Alg_BPMF}
  Initialize $\boldsymbol{\lambda}_{a}^{\left(0\right)}=\sum_{i\in\mathcal{N}_{\rm BP}(a)}\mathbf{L}_{a,i}^{\top}\boldsymbol{\lambda}_{i}^{\left(0\right)},\forall a\in\mathcal{A}_{\rm BP}$;\label{alg_Initialize}\\
  \For{$l=1:L$}
  {
    \eIf{$l=1$}
    {
      Initialize $\mathbf{n}_{i\to a}^{\left(1\right)}=\boldsymbol{\lambda}_{i}^{\left(0\right)},\forall i\in\mathcal{I}_{\rm BP},a\in\mathcal{N}(i)$;\label{alg_Initialize_msg}\\
    }
    {
      $\mathbf{n}_{i\to a}^{\left(l\right)}=\sum_{c\in\mathcal{N}(i)\setminus a}\mathbf{m}_{c\to i}^{\left(l-1\right)},\forall i\in\mathcal{I}_{\rm BP},a\in\mathcal{N}(i)$;\label{Alg_n}\\
    }
    $\mathbf{m}_{a\to i}^{(l)}=(\nabla_{\boldsymbol{\mu}_{i}}U_{a})^{(l)},\forall a\in\mathcal{A},i\in\mathcal{N}_{\rm MF}(a)$;\label{Alg_BPMF_grad_2}\\
    $\boldsymbol{\lambda}_{a}^{(l)}=(\nabla_{\boldsymbol{\mu}_{a}}U_{a})^{(l)}+\sum_{i\in\mathcal{N}_{\rm BP}(a)}\mathbf{L}_{a,i}^{\top}\mathbf{n}_{i\to a}^{(l)},\forall a\in\mathcal{A}_{\rm BP}$;\label{Alg_BPMF_grad_1}\\
    $\mathbf{m}_{a\to i}^{(l)}=g(\mathbf{L}_{a,i}f(\boldsymbol{\lambda}_{a}^{(l)}))-\mathbf{n}_{i\to a}^{\left(l\right)},\forall a\in\mathcal{A},i\in\mathcal{N}_{\rm BP}(a)$;\label{Alg_m}\\
    $\boldsymbol{\lambda}_{i}^{(l)}=\sum_{a\in\mathcal{N}(i)}\mathbf{m}_{a\to i}^{(l)},\forall i\in\mathcal{I}$;\label{Alg_BPMF_lambda_i_BP}\\
  }
\end{algorithm}

\begin{Remark}
In contrast to existing MP algorithms, e.g., \cite{BP3}, \cite{bibVMP1}, \cite{bibVMP2}, \cite{bibBPMF1}, \cite{bibBPMF3}, the proposed approach formulates MP directly in the parameter space of beliefs rather than optimizing the beliefs themselves. This enables conjugate factors to retain efficient closed-form updates, while non-conjugate factors, typically induced by nonlinear measurement models, are handled through dedicated gradient estimators. Such a parametric formulation balances computational efficiency and estimation accuracy. Moreover, the parametric formulation naturally interfaces with modern optimization and learning tools, such as gradient estimators and learnable samplers.
\end{Remark}

\begin{Remark}
The proposed framework is developed under the following basic modeling conditions. The joint distribution is assumed to admit the general factor graph factorization. The candidate beliefs are selected from minimal exponential families. For BP variables considered in state estimation, Gaussian beliefs are adopted, consistent with the widely used Gaussian approximation of state uncertainty. These conditions cover a broad class of state estimation problems.
\end{Remark}

\subsection{Gradient Estimators}
\label{subsec::gradient_estimators}
The next issue is how to evaluate gradients $(\nabla_{\boldsymbol{\mu}_{a}}U_{a})^{(l)}$ and $(\nabla_{\boldsymbol{\mu}_{i}}U_{a})^{(l)}$ in Algorithm \ref{Alg_BPMF}. We first denote $b_{i}(\mathbf{x}_{i})={\rm N}(\mathbf{x}_{i};\hat{\mathbf{x}}_{i},\mathbf{P}_{i}),\forall i\in\mathcal{I}_{\rm BP}$ and $b_{a}(\mathbf{x}_{a})={\rm N}(\mathbf{x}_{a};\hat{\mathbf{x}}_{a},\mathbf{P}_{a}),\forall a\in\mathcal{A}_{\rm BP}$ and denote their estimates at the $l$th iteration as $b_{i}^{(l)}(\mathbf{x}_{i})={\rm N}(\mathbf{x}_{i};\hat{\mathbf{x}}_{i}^{(l)},\mathbf{P}_{i}^{(l)})$ and $b_{a}^{(l)}(\mathbf{x}_{a})={\rm N}(\mathbf{x}_{a};\hat{\mathbf{x}}_{a}^{(l)},\mathbf{P}_{a}^{(l)})$, respectively. Based on the modeling in localization and a summary of the conjugate noise models \cite{bibVBAKF2}, \cite{bibVBRKF1}, \cite{bibVBRKF2}, for factor node $a\in\mathcal{A}_{\rm BP}$, the factor is restricted to
\begin{align}
    f_{a}\left(\mathbf{x}_{a},\boldsymbol{\phi}_{a}\right)=c({\boldsymbol{\phi}_{a}})\exp\left(-\frac{1}{2}\mathbf{r}_{a}^{\top}\left(\mathbf{x}_{a}\right)\mathbf{R}_{a}^{-1}({\boldsymbol{\phi}_{a}})\mathbf{r}_{a}\left(\mathbf{x}_{a}\right)\right),\label{factor_model}
\end{align}
where $\mathbf{r}_{a}\left(\mathbf{x}_{a}\right)$ is the residual model, $c({\boldsymbol{\phi}_{a}})>0$ and symmetric, positive definite matrix $\mathbf{R}_{a}({\boldsymbol{\phi}_{a}})$ are functions of $\boldsymbol{\phi}_{a}$. For localization, $\mathbf{r}_{a}\left(\mathbf{x}_{a}\right)$ describes the measurement model, and variable $\boldsymbol{\phi}_{a}$ represents measurement noise parameter.

Since the integral involved in $(\nabla_{\boldsymbol{\mu}_{a}}U_{a})^{(l)}$ generally becomes intractable when $\mathbf{r}_{a}(\mathbf{x}_{a})$ is nonlinear, we give linearization, sampling, NF-based gradient estimators. Linearization is computationally efficient but introduces local approximation errors. Sampling provides a conventional Monte Carlo estimator and forms a basis of NF-based estimator. The latter further enables learnable sampling.

For gradient $(\nabla_{\boldsymbol{\mu}_{i}}U_{a})^{(l)}$ in step \ref{Alg_BPMF_grad_2} of Algorithm \ref{Alg_BPMF}, the calculation varies for different models. This will be given later in applications in Section \ref{subsec::gradient_estimators_app}. However, the expectation $\mathbf{A}_{a}^{(l)}={\rm E}_{b_{a}^{(l-1)}(\mathbf{x}_{a})}[\mathbf{r}_{a}(\mathbf{x}_{a})\mathbf{r}_{a}^{\top}(\mathbf{x}_{a})]$ will be later used in the evaluation of $(\nabla_{\boldsymbol{\mu}_{i}}U_{a})^{(l)}$, and its calculation is related to gradient estimators designed in this subsection. The following two lemmas will be helpful in  gradient estimator design:

\begin{Lemma}
\label{Lemma_simplified_U}
For factor (\ref{factor_model}), $U_{a}$ in (\ref{U_a}) can be simplified as
\begin{align}
  U_{a}=\int b_{a}\left(\mathbf{x}_{a}\right)\log{\rm N}\left(\mathbf{r}_{a}\left(\mathbf{x}_{a}\right);\mathbf{0},\bar{\mathbf{R}}_{a}\right){\rm d}\mathbf{x}_{a}+c.\notag
\end{align}
where covariance $\bar{\mathbf{R}}_{a}={\rm E}^{-1}_{b_{a}(\boldsymbol{\phi}_{a})}[\mathbf{R}^{-1}_{a}\left(\boldsymbol{\phi}_{a}\right)]$, and $c$ is a constant independent of $\mathbf{x}_{a}$ and $\boldsymbol{\mu}_{a}$.
\end{Lemma}
\textit{Proof:} See Appendix \ref{Proof_simplified_U}.
$\hfill\blacksquare$

\begin{Lemma}
\label{Lemma_gradient_mu}
The gradient of $U_{a}$ in (\ref{U_a}) with respect to $\boldsymbol{\mu}_{a}$ can be calculated by
\begin{align}
  \nabla_{\boldsymbol{\mu}_{a}}U_{a}=\left[\begin{matrix}
    \frac{\partial U_{a}}{\partial \hat{\mathbf{x}}_{a}}-2\frac{\partial U_{a}}{\partial \mathbf{P}_{a}}\hat{\mathbf{x}}_{a}\\
    {\rm vec}\left(\frac{\partial U_{a}}{\partial \mathbf{P}_{a}}\right)
  \end{matrix}\right].\label{gradient_expectation}
\end{align}
\end{Lemma}
\textit{Proof:} See Appendix \ref{Proof_gradient_mu}.
$\hfill\blacksquare$

\subsubsection{Linearization}
Adopting first-order Taylor expansion, we linearize $\mathbf{r}_{a}\left(\mathbf{x}_{a}\right)$ as $\mathbf{r}_{a}\left(\mathbf{x}_{a}\right)=\mathbf{r}_{a}\left(\bar{\mathbf{x}}_{a}\right)+\mathbf{J}_{a}\left(\bar{\mathbf{x}}_{a}\right)(\mathbf{x}_{a}-\bar{\mathbf{x}}_{a})$, where the linearization point $\bar{\mathbf{x}}_{a}$ can be chosen as the latest estimation of $\mathbf{x}_{a}$, i.e., $\hat{\mathbf{x}}_{a}^{(l-1)}$, the mean of belief $b_{a}^{(l-1)}(\mathbf{x}_{a})$. The matrix $\mathbf{J}_{a}(\bar{\mathbf{x}}_{a})$ is the Jacobian of $\mathbf{r}_{a}(\mathbf{x}_{a})$ with respect to $\mathbf{x}_{a}$, evaluated at $\bar{\mathbf{x}}_{a}$. With the help of Lemma \ref{Lemma_simplified_U}, we have
\begin{align}
  \frac{\partial U_{a}}{\partial \hat{\mathbf{x}}_{a}}=&-\mathbf{J}_{a}^{\top}\left(\bar{\mathbf{x}}_{a}\right)\bar{\mathbf{R}}_{a}^{-1}\left(\mathbf{r}_{a}\left(\bar{\mathbf{x}}_{a}\right)+\mathbf{J}_{a}\left(\bar{\mathbf{x}}_{a}\right)\left(\hat{\mathbf{x}}_{a}-\bar{\mathbf{x}}_{a}\right)\right),\notag\\
  \frac{\partial U_{a}}{\partial \mathbf{P}_{a}}=&-\frac{1}{2}\mathbf{J}_{a}^{\top}\left(\bar{\mathbf{x}}_{a}\right)\bar{\mathbf{R}}_{a}^{-1}\mathbf{J}_{a}\left(\bar{\mathbf{x}}_{a}\right).\notag
\end{align}
When evaluating $(\nabla_{\boldsymbol{\mu}_{a}}U_{a})^{(l)}$, we apply these two equations at the linearization point $\bar{\mathbf{x}}_{a}=\hat{\mathbf{x}}_{a}^{(l-1)}$, and substitute $\hat{\mathbf{x}}_{a}=\hat{\mathbf{x}}_{a}^{(l-1)}$ and $\bar{\mathbf{R}}_{a}^{-1}=\boldsymbol{\Lambda}_{a}^{(l-1)}={\rm E}_{b_{a}^{(l-1)}(\boldsymbol{\phi}_{a})}[\mathbf{R}^{-1}_{a}\left(\boldsymbol{\phi}_{a}\right)]$, then
\begin{align}
  \left(\frac{\partial U_{a}}{\partial \hat{\mathbf{x}}_{a}}\right)^{(l)}=&-\mathbf{J}_{a}^{\top}(\hat{\mathbf{x}}_{a}^{(l-1)})\boldsymbol{\Lambda}_{a}^{(l-1)}\mathbf{r}_{a}(\hat{\mathbf{x}}_{a}^{(l-1)}),\label{Taylor_mean}\\
  \left(\frac{\partial U_{a}}{\partial \mathbf{P}_{a}}\right)^{(l)}=&-\frac{1}{2}\mathbf{J}_{a}^{\top}(\hat{\mathbf{x}}_{a}^{(l-1)})\boldsymbol{\Lambda}_{a}^{(l-1)}\mathbf{J}_{a}(\hat{\mathbf{x}}_{a}^{(l-1)}).\label{Taylor_cov}
\end{align}

In summary, to evaluate gradient $(\nabla_{\boldsymbol{\mu}_{a}}U_{a})^{(l)}$ by deterministic linearization, we first calculate (\ref{Taylor_mean}) and (\ref{Taylor_cov}), and then substitute them and $\hat{\mathbf{x}}_{a}=\hat{\mathbf{x}}_{a}^{(l-1)}$ into (\ref{gradient_expectation}). This also applies to conjugate factors with linear residual models $\mathbf{r}_{a}\left(\mathbf{x}_{a}\right)$. In this case, no linear approximation is needed, and the resulting update reduces to an exact closed-form solution.

Similarly, the expectation $\mathbf{A}_{a}^{(l)}$ can be updated by
\begin{align}
  \mathbf{A}_{a}^{(l)}=&\mathbf{J}_{a}(\hat{\mathbf{x}}_{a}^{(l-1)})\mathbf{P}_{a}^{(l-1)}\mathbf{J}_{a}^{\top}(\hat{\mathbf{x}}_{a}^{(l-1)})+\mathbf{r}_{a}(\hat{\mathbf{x}}_{a}^{(l-1)})\mathbf{r}_{a}^{\top}(\hat{\mathbf{x}}_{a}^{(l-1)}).\label{A_Taylor}
\end{align}

\subsubsection{Sampling} Combining Lemma \ref{Lemma_simplified_U} with Bonnet's theorem and Price's theorem \cite{bibPML}, which relate gradients of expectations to expectations of gradients, we obtain
\begin{align}
  \frac{\partial U_{a}}{\partial \hat{\mathbf{x}}_{a}}=&\int b_{a}\left(\mathbf{x}_{a}\right)\nabla_{\mathbf{x}_{a}}\log{\rm N}\left(\mathbf{r}_{a}\left(\mathbf{x}_{a}\right);\mathbf{0},\bar{\mathbf{R}}_{a}\right){\rm d}\mathbf{x}_{a},\label{gradient_mean}\\
  \frac{\partial U_{a}}{\partial \mathbf{P}_{a}}=&\frac{1}{2}\int b_{a}\left(\mathbf{x}_{a}\right)\nabla_{\mathbf{x}_{a}}^{2}\log{\rm N}\left(\mathbf{r}_{a}\left(\mathbf{x}_{a}\right);\mathbf{0},\bar{\mathbf{R}}_{a}\right){\rm d}\mathbf{x}_{a},\label{gradient_cov}
\end{align}
with
\begin{align}
  \nabla_{\mathbf{x}_{a}}\log{\rm N}\left(\mathbf{r}_{a}\left(\mathbf{x}_{a}\right);\mathbf{0},\bar{\mathbf{R}}_{a}\right)=&-\mathbf{J}_{a}^{\top}\left(\mathbf{x}_{a}\right)\bar{\mathbf{R}}_{a}^{-1}\mathbf{r}_{a}\left(\mathbf{x}_{a}\right),\notag\\
  \nabla^{2}_{\mathbf{x}_{a}}\log{\rm N}\left(\mathbf{r}_{a}\left(\mathbf{x}_{a}\right);\mathbf{0},\bar{\mathbf{R}}_{a}\right)=&-\mathbf{J}_{a}^{\top}\left(\mathbf{x}_{a}\right)\bar{\mathbf{R}}_{a}^{-1}\mathbf{J}_{a}\left(\mathbf{x}_{a}\right)-\tilde{\mathbf{H}}_{a},\notag
\end{align}
where $\tilde{\mathbf{H}}_{a}=\sum_{d=1}^{n_{r_{a}}}G_{a,d}(\mathbf{x}_{a})\mathbf{H}_{a,d}(\mathbf{x}_{a})$, $G_{a,d}(\mathbf{x}_{a})$ is the $d$th element of vector $\bar{\mathbf{R}}_{a}^{-1}\mathbf{r}_{a}(\mathbf{x}_{a})$, $\mathbf{H}_{a,d}(\mathbf{x}_{a})$ is the Hessian of the $d$th element of $\mathbf{r}_{a}(\mathbf{x}_{a})$ with respect to $\mathbf{x}_{a}$, and $n_{r_{a}}$ is the dimension of $\mathbf{r}_{a}(\mathbf{x}_{a})$. Directly computing $\tilde{\mathbf{H}}_{a}$ requires the Hessians of all residuals and is computationally expensive. Following the Gauss-Newton approximation \cite{bibestrob}, we neglect $\tilde{\mathbf{H}}_{a}$, thereby avoiding costly second-order derivatives while preserving positive-semidefinite information contributions that help maintain valid covariances during MP iterations.

Both $\frac{\partial U_{a}}{\partial \hat{\mathbf{x}}_{a}}$ and $\frac{\partial U_{a}}{\partial \mathbf{P}_{a}}$ can be written in the unified form $\int b_{a}\left(\mathbf{x}_{a}\right)\mathcal{F}\left(\mathbf{x}_{a}\right){\rm d}\mathbf{x}_{a}$. We use the reparameterization trick:
\begin{align}
  \int b_{a}\left(\mathbf{x}_{a}\right)\mathcal{F}\left(\mathbf{x}_{a}\right){\rm d}\mathbf{x}_{a}=&\int q\left(\boldsymbol{\epsilon}\right)\mathcal{H}_{a}\left(\boldsymbol{\epsilon}\right){\rm d}\boldsymbol{\epsilon},\label{reparameterization_trick}
\end{align}
where $q(\boldsymbol{\epsilon})={\rm N}(\boldsymbol{\epsilon};\mathbf{0},\mathbf{I})$, $\mathcal{H}_{a}(\boldsymbol{\epsilon})=\mathcal{F}(\hat{\mathbf{x}}_{a}+\mathbf{S}_{a}\boldsymbol{\epsilon})$, and $\mathbf{S}_{a}\mathbf{S}_{a}^{\top}=\mathbf{P}_{a}$. Applying Monte Carlo integration to (\ref{gradient_mean}) and (\ref{gradient_cov}),
\begin{align}
  \left(\frac{\partial U_{a}}{\partial \hat{\mathbf{x}}_{a}}\right)^{(l)}=&-\frac{1}{S}\sum_{s=1}^{S}\mathbf{J}_{a}^{\top}\big(\mathbf{x}_{a,s}^{(l-1)}\big)\boldsymbol{\Lambda}_{a}^{(l-1)}\mathbf{r}_{a}\big(\mathbf{x}_{a,s}^{(l-1)}\big),\label{sampling_mean}\\
  \left(\frac{\partial U_{a}}{\partial \mathbf{P}_{a}}\right)^{(l)}=&-\frac{1}{2S}\sum_{s=1}^{S}\mathbf{J}_{a}^{\top}\big(\mathbf{x}_{a,s}^{(l-1)})\boldsymbol{\Lambda}_{a}^{(l-1)}\mathbf{J}_{a}(\mathbf{x}_{a,s}^{(l-1)}\big),\label{sampling_cov}
\end{align}
where $\mathbf{x}_{a,s}^{(l-1)}=\hat{\mathbf{x}}_{a}^{(l-1)}+\mathbf{S}_{a}^{(l-1)}\boldsymbol{\epsilon}_{a,s}$, $\mathbf{S}_{a}^{(l-1)}(\mathbf{S}_{a}^{(l-1)})^{\top}=\mathbf{P}_{a}^{(l-1)}$, $\boldsymbol{\epsilon}_{a,s}$ is sampled from ${\rm N}(\boldsymbol{\epsilon}_{a,s};\mathbf{0},\mathbf{I})$, and $S$ is the number of Monte Carlo samples.

In summary, to evaluate gradient $(\nabla_{\boldsymbol{\mu}_{a}}U_{a})^{(l)}$ by sampling, we calculate (\ref{sampling_mean}) and (\ref{sampling_cov}), and substitute them and $\hat{\mathbf{x}}_{a}=\hat{\mathbf{x}}_{a}^{(l-1)}$ into (\ref{gradient_expectation}). For expectation $\mathbf{A}_{a}^{(l)}$,
\begin{align}
  \mathbf{A}_{a}^{(l)}=&\frac{1}{S}\sum_{s=1}^{S}\mathbf{r}_{a}(\mathbf{x}_{a,s}^{(l-1)})\mathbf{r}_{a}^{\top}(\mathbf{x}_{a,s}^{(l-1)}).\label{A_sampling}
\end{align}

\subsubsection{NF}
Still consider equation (\ref{reparameterization_trick}). We adopt an importance sampling formulation and draw samples from an alternative proposal density $p\left(\boldsymbol{\epsilon}\right)$. Then,
\begin{align}
  \int b_{a}\left(\mathbf{x}_{a}\right)\mathcal{F}\left(\mathbf{x}_{a}\right){\rm d}\mathbf{x}_{a}=&\int \tilde{w}\left(\boldsymbol{\epsilon}\right)p\left(\boldsymbol{\epsilon}\right)\mathcal{H}_{a}\left(\boldsymbol{\epsilon}\right){\rm d}\boldsymbol{\epsilon},\label{importance_sampling}
\end{align}
where $\tilde{w}\left(\boldsymbol{\epsilon}\right)=\frac{q\left(\boldsymbol{\epsilon}\right)}{p\left(\boldsymbol{\epsilon}\right)}$.

The density $p\left(\boldsymbol{\epsilon}\right)$ is modeled by Real NVP, which defines a bijective transformation $\boldsymbol{\epsilon}=T_{\boldsymbol{\theta}}(\mathbf{y})$ and $\mathbf{y}$ is sampled from $p(\mathbf{y})={\rm N}(\mathbf{y};\mathbf{0},\mathbf{I})$. The proposal $p\left(\boldsymbol{\epsilon}\right)$ is then induced by this transformation. In practice, $T_{\boldsymbol{\theta}}(\mathbf{y})$ is conditioned on features, such as sensor measurements, beliefs, and noise covariances, allowing the transformation to adapt to available information. These quantities are encoded in a condition vector $\boldsymbol{\theta}$, which will be detailed in Section \ref{sec::multirobot_Collaborative_Localization}. By applying the change of variables to (\ref{importance_sampling}) and treating $\mathbf{y}$ as the integration variable,
\begin{align}
  \int b_{a}\left(\mathbf{x}_{a}\right)\mathcal{F}\left(\mathbf{x}_{a}\right){\rm d}\mathbf{x}_{a}=\!&\int w\left(\mathbf{y}\right)p\left(\mathbf{y}\right)\mathcal{H}_{a}\left(T_{\boldsymbol{\theta}}\left(\mathbf{y}\right)\right){\rm d}\mathbf{y},\label{nn_sampling}
\end{align}
where
\begin{align}
  w\left(\mathbf{y}\right)=&\frac{q\left(T_{\boldsymbol{\theta}}(\mathbf{y})\right)\left|{\rm det}\left(\frac{\partial T_{\boldsymbol{\theta}}(\mathbf{y})}{\partial \mathbf{y}^{\top}}\right)\right|}{p(\mathbf{y})}\notag\\
  =&\exp\left(-0.5\left(\Vert T_{\boldsymbol{\theta}}(\mathbf{y})\Vert^{2}_{2}-\Vert\mathbf{y}\Vert^{2}_{2}\right)\right)\left|{\rm det}\left(\frac{\partial T_{\boldsymbol{\theta}}(\mathbf{y})}{\partial \mathbf{y}^{\top}}\right)\right|.\notag
\end{align}

Applying (\ref{nn_sampling}) to compute $\frac{\partial U_{a}}{\partial \hat{\mathbf{x}}_{a}}$ and $\frac{\partial U_{a}}{\partial \mathbf{P}_{a}}$, we have
\begin{align}
    \left(\frac{\partial U_{a}}{\partial \hat{\mathbf{x}}_{a}}\right)^{\!(l)}\!=\!&-\!\frac{1}{S}\sum_{s=1}^{S}w(\mathbf{y}_{a,s})\mathbf{J}_{a}^{\top}\big(\mathbf{x}_{a,s}^{(l-1)}\big)\boldsymbol{\Lambda}_{a}^{(l-1)}\mathbf{r}_{a}\big(\mathbf{x}_{a,s}^{(l-1)}\big),\label{NN_mean}\\
  \left(\frac{\partial U_{a}}{\partial \mathbf{P}_{a}}\right)^{\!(l)}\!=\!&-\!\frac{1}{2S}\sum_{s=1}^{S}w(\mathbf{y}_{a,s})\mathbf{J}_{a}^{\top}\big(\mathbf{x}_{a,s}^{(l-1)}\big)\boldsymbol{\Lambda}_{a}^{(l-1)}\mathbf{J}_{a}\big(\mathbf{x}_{a,s}^{(l-1)}\big),\label{NN_cov}
\end{align}
where $\mathbf{x}_{a,s}^{(l-1)}=\hat{\mathbf{x}}_{a}^{(l-1)}+\mathbf{S}_{a}^{(l-1)}T_{\boldsymbol{\theta}^{(l)}}(\mathbf{y}_{a,s})$, $\mathbf{S}_{a}^{(l-1)}(\mathbf{S}_{a}^{(l-1)})^{\top}=\mathbf{P}_{a}^{(l-1)}$, and $\mathbf{y}_{a,s}$ is sampled from ${\rm N}(\mathbf{y}_{a,s};\mathbf{0},\mathbf{I})$.

For the expectation $\mathbf{A}_{a}^{(l)}$,
\begin{align}
  \mathbf{A}_{a}^{(l)}=&\frac{1}{S}\sum_{s=1}^{S}w(\mathbf{y}_{a,s})\mathbf{r}_{a}(\mathbf{x}_{a,s}^{(l-1)})\mathbf{r}_{a}^{\top}(\mathbf{x}_{a,s}^{(l-1)}).\label{A_NN}
\end{align}

\begin{Remark}
One of the key roles of NF is to introduce a learnable proposal $p(\boldsymbol{\epsilon})$. Although (\ref{nn_sampling}) no longer displays the proposal density $p(\boldsymbol{\epsilon})$ explicitly, it is equivalent to (\ref{importance_sampling}). Then, (\ref{NN_mean}), (\ref{NN_cov}), (\ref{A_NN}) provide unbiased estimates of mean gradient ${\rm E}_{b_{a}^{(l-1)}(\mathbf{x}_{a})}[-\mathbf{J}_{a}^{\top}\left(\mathbf{x}_{a}\right)\boldsymbol{\Lambda}_{a}^{(l-1)}\mathbf{r}_{a}\left(\mathbf{x}_{a}\right)]$, covariance gradient ${\rm E}_{b_{a}^{(l-1)}(\mathbf{x}_{a})}[-\frac{1}{2}\mathbf{J}_{a}^{\top}\left(\mathbf{x}_{a}\right)\boldsymbol{\Lambda}_{a}^{(l-1)}\mathbf{J}_{a}\left(\mathbf{x}_{a}\right)]$, and expectation ${\rm E}_{b^{(l-1)}(\mathbf{x}_{a})}[\mathbf{r}_{a}(\mathbf{x}_{a})\mathbf{r}_{a}^{\top}(\mathbf{x}_{a})]$, respectively. This establishes an interpretable justification for introducing NF. Beyond this, NF is further optimized through end-to-end training within the full MP algorithm. This allows the network to learn richer, task-specific mappings that are difficult to design analytically.
\end{Remark}

\subsection{Extension to Lie Groups}
We further consider the case in which variables in BP part are all defined on Lie groups, which is common for robot state estimation. For $i\in\mathcal{I}_{\rm BP}$, its variable is denoted by $\mathcal{X}_{i}\in\mathcal{M}_{i}$, where $\mathcal{M}_{i}$ is a product manifold. For factor node $a\in\mathcal{A}_{\rm BP}$, the variable $\mathcal{X}_{a}\in\mathcal{M}_{a}$ consists of all $\mathcal{X}_{i},i\in\mathcal{N}_{\rm BP}(a)$, with product manifold $\mathcal{M}_{a}$ consisting of all $\mathcal{M}_{i},i\in\mathcal{N}_{\rm BP}(a)$. The factors are denoted by $\bar{f}_{a}\left(\mathcal{X}_{a},\boldsymbol{\phi}_{a}\right),\forall a\in\mathcal{A}$.

Inspired by ``lift-solve-retract" scheme \cite{bibLie}, the inference problem can be represented locally in a Euclidean space. Assume that we have initial estimates of beliefs, denoted by $b_{i}\left(\mathcal{X}_{i}\right)={\rm N}\big(\mathcal{X}_{i};\hat{\mathcal{X}}_{i},\boldsymbol{\Sigma}_{i}\big)$. We use $\hat{\mathcal{X}}_{i},i\in\mathcal{I}_{\rm BP}$ as tangent points. Let $\bar{\mathcal{X}}_{i}=\hat{\mathcal{X}}_{i}$ and $\bar{\mathcal{X}}_{a}=(\bar{\mathcal{X}}_{i})_{i\in\mathcal{N}_{\rm BP}(a)}$. Define retraction functions $\mathcal{X}_{i}=\mathcal{R}_{\bar{\mathcal{X}}_{i}}(\delta\mathbf{x}_{i})\triangleq\bar{\mathcal{X}}_{i}\boxplus\delta\mathbf{x}_{i},\forall i\in\mathcal{I}_{\rm BP}$ and $\mathcal{X}_{a}=\mathcal{R}_{\bar{\mathcal{X}}_{a}}(\delta\mathbf{x}_{a})\triangleq\bar{\mathcal{X}}_{a}\boxplus\delta\mathbf{x}_{a},\forall a\in\mathcal{A}_{\rm BP}$. Then, we can use vectors $\delta\mathbf{x}_{i}$ and $\delta\mathbf{x}_{a}$ in Euclidean space as latent variables. Define $f_{a}\left(\delta\mathbf{x}_{a},\boldsymbol{\phi}_{a}\right)\triangleq\bar{f}_{a}\left(\bar{\mathcal{X}}_{a}\boxplus\delta\mathbf{x}_{a},\boldsymbol{\phi}_{a}\right)$. When using $f_{a}\left(\delta\mathbf{x}_{a},\boldsymbol{\phi}_{a}\right)$ to construct factor graph and viewing $\delta\mathbf{x}_{a}$, $\delta\mathbf{x}_{i}$, and $\boldsymbol{\phi}_{i}$ as variables, the inference problems on Lie group can be converted to that in Euclidean space. Then, Algorithm \ref{Alg_BPMF} can be used to infer beliefs of $\delta\mathbf{x}_{a}$, $\delta\mathbf{x}_{i}$, and $\boldsymbol{\phi}_{i}$. Combining $b_{i}\left(\mathcal{X}_{i}\right)={\rm N}\big(\mathcal{X}_{i};\hat{\mathcal{X}}_{i},\boldsymbol{\Sigma}_{i}\big)$ with $\mathcal{X}_{i}=\mathcal{R}_{\bar{\mathcal{X}}_{i}}(\delta\mathbf{x}_{i})$, the initial belief of $\delta\mathbf{x}_{i}$ is set as
\begin{align}
  b^{(0)}_{i}(\delta\mathbf{x}_{i})=&{\rm N}(\delta\mathbf{x}_{i};\mathbf{0},\mathbf{P}_{i}^{(0)}),\label{initial_b_x_i}
\end{align}
where $\mathbf{P}_{i}^{(0)}=\boldsymbol{\Sigma}_{i}$. After running Algorithm \ref{Alg_BPMF}, we can obtain $b^{(L)}_{i}(\delta\mathbf{x}_{i})={\rm N}(\delta\mathbf{x}_{i};\delta\hat{\mathbf{x}}_{i}^{(L)},\mathbf{P}_{i}^{(L)})$. Finally, belief $b_{i}\left(\mathcal{X}_{i}\right)={\rm N}\big(\mathcal{X}_{i};\hat{\mathcal{X}}_{i},\boldsymbol{\Sigma}_{i}\big)$ can be updated by
\begin{align}
  \label{alg_warp}
  \hat{\mathcal{X}}_{i}=\mathcal{R}_{\bar{\mathcal{X}}_{i}}(\delta\hat{\mathbf{x}}_{i}^{(L)}),\boldsymbol{\Sigma}_{i}=\mathbf{J}_{r}(\delta\hat{\mathbf{x}}_{i}^{(L)})\mathbf{P}_{i}^{(L)}\mathbf{J}_{r}^{\top}(\delta\hat{\mathbf{x}}_{i}^{(L)}),
\end{align}
where $\mathbf{J}_{r}(\delta\hat{\mathbf{x}}_{i}^{(L)})$ is the right Jacobian matrix on $\mathcal{M}_{i}$. The MP on Lie groups is summarized in Algorithm \ref{Alg2}.

\begin{algorithm}[tb]
  \small
  \hsize=\linewidth
  \caption{MP on Lie Groups}
  \label{Alg2}
  Initialize $\bar{\mathcal{X}}_{a}=(\bar{\mathcal{X}}_{i})_{i\in\mathcal{N}_{\rm BP}(a)}$;\\
  Initialize $b^{(0)}_{i}(\delta\mathbf{x}_{i}),\forall i\in\mathcal{I}_{\rm BP}$ by (\ref{initial_b_x_i});\\
  View $\delta\mathbf{x}_{a}$, $\delta\mathbf{x}_{i}$, and $\boldsymbol{\phi}_{i}$ as variables and run Algorithm \ref{Alg_BPMF};\\
  Update $b_{i}\left(\mathcal{X}_{i}\right)={\rm N}\big(\mathcal{X}_{i};\hat{\mathcal{X}}_{i},\boldsymbol{\Sigma}_{i}\big),\forall i\in\mathcal{I}_{\rm BP}$ by (\ref{alg_warp});
\end{algorithm}

\section{Multirobot Collaborative Localization}
\label{sec::multirobot_Collaborative_Localization}
In this section, we give the application of the proposed MP algorithm to multirobot collaborative localization. Each robot is equipped with odometry, GNSS, and UWB. The UWB measures inter-robot distances. We first introduce measurement models and noise assumptions. Then, we present the specific expressions of $(\nabla_{\boldsymbol{\mu}_{i}}U_{a})^{(l)}$ in step \ref{Alg_BPMF_grad_2} of Algorithm \ref{Alg_BPMF}. Finally, we discuss some implementation details.

\subsection{Measurement Models and Noise Assumptions}
The world frame is the same for all robots. Denote the state of robot $n\in\mathcal{N}=\{1,2,\ldots,N\}$ at time step $k$ as $\mathcal{X}_{k,n}=(\mathbf{R}_{k,n},\mathbf{p}_{k,n})\in\mathcal{M}={\rm SO}(3)\times\mathbb{R}^{3}$, where $\mathbf{R}_{k,n}\in{\rm SO}(3)$ and $\mathbf{p}_{k,n}\in\mathbb{R}^{3}$ represent the orientation and position of robot $n$ in the world frame, respectively. The body frame of each robot is attached to its odometry. We view error state $\delta \mathbf{x}_{k,n}=\mathcal{X}_{k,n}\boxminus\bar{\mathcal{X}}_{k,n}$ as variables in factor graphs, where $\bar{\mathcal{X}}_{k,n}$ is set as the previous estimate of $\mathcal{X}_{k,n}$. The estimation framework is designed based on sliding window estimation. At time step $k$, states from $k_0$ to $k$ require simultaneous estimation.

\subsubsection{Odometry Model}
The odometry measures the relative motion of robot $n$ from time step $k-1$ to $k$. We model its measurement $\mathbf{z}_{k,n}^{\mathcal{O}}\in\mathcal{M}$ as
\begin{align}
  \mathbf{z}_{k,n}^{\mathcal{O}}=\mathbf{h}^{\mathcal{O}}\left(\mathcal{X}_{k-1,n},\mathcal{X}_{k,n}\right)\boxplus\mathbf{v}^{\mathcal{O}}_{k,n},\label{measurement_model_odo}
\end{align}
where noise $\mathbf{v}^{\mathcal{O}}_{k,n}\in\mathbb{R}^{6}$ is zero-mean Gaussian with covariance $\mathbf{P}^{\mathcal{O}}_{k,n}$, and relative motion model
\begin{align*}
  \mathbf{h}^{\mathcal{O}}\left(\mathcal{X}_{k-1,n},\mathcal{X}_{k,n}\right)=\left[\begin{matrix}
    \mathbf{R}_{k-1,n}^{\top}\mathbf{R}_{k,n}\\
    \mathbf{R}_{k-1,n}^{\top}(\mathbf{p}_{k,n}-\mathbf{p}_{k-1,n})
  \end{matrix}\right].
\end{align*}
Define residual $\mathbf{r}^{\mathcal{O}}_{k,n}=\mathbf{z}_{k,n}^{\mathcal{O}}\boxminus\mathbf{h}^{\mathcal{O}}\left(\mathcal{X}_{k-1,n},\mathcal{X}_{k,n}\right)$. To obtain a Gaussian likelihood, we model a virtual measurement $\tilde{\mathbf{z}}_{k,n}^{\mathcal{O}}$ by
\begin{align}
  p\big(\tilde{\mathbf{z}}_{k,n}^{\mathcal{O}}|\delta\mathbf{x}_{k-1,n},\delta\mathbf{x}_{k,n},\mathbf{P}^{\mathcal{O}}_{k,n}\big)={\rm N}\big(\tilde{\mathbf{z}}_{k,n}^{\mathcal{O}};\mathbf{r}^{\mathcal{O}}_{k,n}, \mathbf{P}^{\mathcal{O}}_{k,n}\big),\notag
\end{align} 
and fix its observed value to $\tilde{\mathbf{z}}_{k,n}^{\mathcal{O}}=\mathbf{0}$.  Denote odometry factor
\begin{align}
  f_{k,n}^{\mathcal{O}}(\delta\mathbf{x}_{k-1,n},\delta\mathbf{x}_{k,n},\mathbf{P}^{\mathcal{O}}_{k,n})=p(\tilde{\mathbf{z}}^{\mathcal{O}}_{k,n}|\delta\mathbf{x}_{k-1,n},\delta\mathbf{x}_{k,n},\mathbf{P}^{\mathcal{O}}_{k,n}).\label{factor_odo}
\end{align}

To estimate the uncertain odometry noise covariances and enhance adaptability, we view $\mathbf{P}^{\mathcal{O}}_{k,n}$ as a latent variable and model its prior as inverse Wishart distribution $p(\mathbf{P}^{\mathcal{O}}_{k,n})={\rm IW}(\mathbf{P}^{\mathcal{O}}_{k,n};t^{\mathcal{O}}_{k,n},\mathbf{T}^{\mathcal{O}}_{k,n})$. Denote factor
\begin{align}
  f_{k,n}^{\mathcal{C}}(\mathbf{P}^{\mathcal{O}}_{k,n})=p(\mathbf{P}^{\mathcal{O}}_{k,n}).\label{factor_odo_R}
\end{align}
Denote the estimated belief of $\mathbf{P}^{\mathcal{O}}_{k-1,n}$ at time step $k-1$ as $b(\mathbf{P}^{\mathcal{O}}_{k-1,n})={\rm IW}(\mathbf{P}^{\mathcal{O}}_{k-1,n};\hat{t}^{\mathcal{O}}_{k-1,n},\hat{\mathbf{T}}^{\mathcal{O}}_{k-1,n})$. Motivated by the heuristic mechanism in \cite{bibVBAKF2}, we set $t^{\mathcal{O}}_{k,n}=\rho \hat{t}^{\mathcal{O}}_{k-1,n}$ and $\mathbf{T}^{\mathcal{O}}_{k,n}=\rho \hat{\mathbf{T}}^{\mathcal{O}}_{k-1,n}$, where $\rho\in(0,1]$ is a forgetting factor. This can retain the previous estimates of the noise covariances while accommodating time-varying noise covariances.

\subsubsection{GNSS Model}
Denote the measurement from GNSS as $\mathbf{z}^{\mathcal{G}}_{k,n}\in\mathbb{R}^{3}$. The GNSS measurement is modeled as
\begin{align}
  \mathbf{z}^{\mathcal{G}}_{k,n}=\mathbf{R}_{k,n}\mathbf{p}^{b_{n}}_{g_{n}}+\mathbf{p}_{k,n}+\mathbf{v}^{\mathcal{G}}_{k,n},\label{measurement_model_gnss}
\end{align}
where $\mathbf{v}^{\mathcal{G}}_{k,n}$ is Gaussian distributed noise with zero mean and covariance matrix $\mathbf{P}^{\mathcal{G}}_{k,n}$, and $\mathbf{p}^{b_{n}}_{g_{n}}$ is the GNSS antenna position with respect to the body frame of robot $n$, i.e., the extrinsic parameter. Then, $p(\mathbf{z}^{\mathcal{G}}_{k,n}|\delta\mathbf{x}_{k,n})={\rm N}(\mathbf{z}^{\mathcal{G}}_{k,n};\mathbf{R}_{k,n}\mathbf{p}^{b_{n}}_{g_{n}}+\mathbf{p}_{k,n},\mathbf{P}^{\mathcal{G}}_{k,n})$. Denote the GNSS factor
\begin{align}
  f^{\mathcal{G}}_{k,n}(\delta\mathbf{x}_{k,n})=p(\mathbf{z}^{\mathcal{G}}_{k,n}|\delta\mathbf{x}_{k,n}).\label{factor_gnss}
\end{align}
Denote the residual $\mathbf{r}^{\mathcal{G}}_{k,n}=\mathbf{z}^{\mathcal{G}}_{k,n}-\mathbf{R}_{k,n}\mathbf{p}^{b_{n}}_{g_{n}}-\mathbf{p}_{k,n}$.

\subsubsection{UWB Model}
Let $z_{k,n}^{\mathcal{U},m}\in\mathbb{R}$ denote the ranging between UWB modules mounted on robots $n$ and $m$. We have
\begin{align}
  z_{k,n}^{\mathcal{U},m}=\left\Vert\Delta\mathbf{p}_{k,n}^{m}\right\Vert+v_{k,n}^{\mathcal{U},m},\label{measurement_model_uwb}
\end{align}
where the relative position $\Delta\mathbf{p}_{k,n}^{m}=\mathbf{R}_{k,n}\mathbf{p}^{b_{n}}_{u_{n}}+\mathbf{p}_{k,n}-\mathbf{R}_{k,m}\mathbf{p}^{b_{m}}_{u_{m}}-\mathbf{p}_{k,m}$, extrinsic parameters $\mathbf{p}^{b_{n}}_{u_{n}}$ and $\mathbf{p}^{b_{m}}_{u_{m}}$ are positions of UWB modules in the body frame of robot $n$ and $m$, respectively. Due to non-line-of-sight and multi-path propagation, UWB ranging noise can alternate between approximately Gaussian behavior and distinctly heavy-tailed characteristics. Thus, noise $v_{k,n}^{\mathcal{U},m}$ is modeled by GSTM \cite{bibVBRKF2}. The unknown mixing weight $0\leq\pi_{k,n}^{m}\leq 1$ specifies the relative contribution of the Gaussian and Student's \textit{t} components. Conditioned on $\pi_{k,n}^{m}$, the density of $v_{k,n}^{\mathcal{U},m}$ is
\begin{align}
  \label{GSTMcondpi}
  p(v_{k,n}^{\mathcal{U},m}|\pi_{k,n}^{m})=&\pi_{k,n}^{m}{\rm N}(v_{k,n}^{\mathcal{U},m};0,P_{k,n}^{\mathcal{U},m})\notag\\
  &+(1-\pi_{k,n}^{m}){\rm St}(v_{k,n}^{\mathcal{U},m};0,P_{k,n,0}^{\mathcal{U},m},\nu_{k,n}^{m}),
\end{align}
where $P_{k,n}^{\mathcal{U},m}$, $P_{k,n,0}^{\mathcal{U},m}$, and $\nu_{k,n}^{m}$ are user-specified parameters. The Student's \textit{t} distribution can be expressed as
\begin{align}
  {\rm St}\big(v_{k,n}^{\mathcal{U},m};0,P_{k,n,0}^{\mathcal{U},m},\nu_{k,n}^{m}\big)=&\int{{\rm N}\big(v_{k,n}^{\mathcal{U},m};0,P_{k,n,0}^{\mathcal{U},m}/\xi_{k,n}^{m}\big)}\notag\\
  &\times{\rm G}\left(\xi_{k,n}^{m};\nu_{k,n}^{m}/2,\nu_{k,n}^{m}/2\right){\rm d}\xi_{k,n}^{m},\notag
\end{align}
with $\xi_{k,n}^{m}$ serving as an auxiliary latent variable. The prior of the mixing weight $\pi_{k,n}^{m}$ follows a Beta distribution:
\begin{align}
  p\left(\pi_{k,n}^{m}\right)=&{\rm Beta}\left(\pi_{k,n}^{m};a_{k,n}^{m},1-a_{k,n}^{m}\right),\notag
\end{align}
where $0<a_{k,n}^{m}<1$ is a prior parameter. Introducing a binary indicator $y_{k,n}^{m}\in\left\{0,1\right\}$, the conditional density in (\ref{GSTMcondpi}) admits the hierarchical representation in (\ref{hierarchical_model}). Denote UWB factor
\begin{align}
  &f^{\mathcal{U},m}_{k,n}(\delta\mathbf{x}_{k,n},\delta\mathbf{x}_{k,m},\xi_{k,n}^{m},y_{k,n}^{m})\notag\\
  =&p\big(z_{k,n}^{\mathcal{U},m}|\delta\mathbf{x}_{k,n},\delta\mathbf{x}_{k,m},\xi_{k,n}^{m},y_{k,n}^{m}\big),\label{factor_uwb}
\end{align}
and noise parameter factors
\begin{align}
&f^{\xi,m}_{k,n}(\xi_{k,n}^{m})=p(\xi_{k,n}^{m}),f^{\pi,m}_{k,n}(\pi_{k,n}^{m})=p(\pi_{k,n}^{m}),\notag\\
&f^{y,m}_{k,n}(y_{k,n}^{m},\pi_{k,n}^{m})=p(y_{k,n}^{m}|\pi_{k,n}^{m}).\label{factor_uwb_R}
\end{align}
Define residual $r^{\mathcal{U},m}_{k,n}=z^{\mathcal{U},m}_{k,n}-\Vert\Delta\mathbf{p}_{k,n}^{m}\Vert$.

\begin{figure*}[t]
\begin{align}
  &p(z_{k,n}^{\mathcal{U},m}|\delta\mathbf{x}_{k,n},\delta\mathbf{x}_{k,m},\xi_{k,n}^{m},y_{k,n}^{m})={\rm N}(z_{k,n}^{\mathcal{U},m};\Vert\Delta\mathbf{p}_{k,n}^{m}\Vert,P_{k,n}^{\mathcal{U},m})^{y_{k,n}^{m}}{\rm N}(z_{k,n}^{\mathcal{U},m};\Vert\Delta\mathbf{p}_{k,n}^{m}\Vert,P_{k,n,0}^{\mathcal{U},m}/\xi_{k,n}^{m})^{1-y_{k,n}^{m}},\notag\\
  &p\left(y_{k,n}^{m}|\pi_{k,n}^{m}\right)={\rm Bern}\left(y_{k,n}^{m};\pi_{k,n}^{m}\right),p\left(\xi_{k,n}^{m}\right)={\rm G}\left(\xi_{k,n}^{m};\nu_{k,n}^{m}/2,\nu_{k,n}^{m}/2\right).\label{hierarchical_model}
\end{align}
  \hrule
  \vspace{-5pt}
\end{figure*}

\subsubsection{Prior of States}
Let $\mathcal{N}_{k,n}$ denote the set of indices of robots that can be ranged by UWB on robot $n$ at time step $k$. Define measurement sets $\mathbf{z}_{k,n}^{\mathcal{U}}=\{z_{k,n}^{\mathcal{U},m}\}_{m\in\mathcal{N}_{k,n}}$, $\mathbf{z}_{k,n}=\{\tilde{\mathbf{z}}_{k,n}^{\mathcal{O}}\}\bigcup\{\mathbf{z}_{k,n}^{\mathcal{G}}\}\bigcup\mathbf{z}_{k,n}^{\mathcal{U}}$, and $\mathbf{z}_{s:k}=\bigcup_{t=s}^{k}\bigcup_{n=1}^{N}\mathbf{z}_{t,n}$, describing ranging obtained by robot $n$ at time step $k$, measurements obtained by robot $n$ at time step $k$, all measurements from time step $s$ to $k$, respectively. The posterior $p(\mathcal{X}_{k_0,n}|\mathbf{z}_{1:k_0})={\rm N}(\mathcal{X}_{k_0,n};\hat{\mathcal{X}}_{k_0|k_0,n},\mathbf{P}^{\mathcal{P}}_{k_0|k_0,n})$ could be calculated at time step $k_{0}$. Denote $\delta\bar{\mathbf{x}}_{k_0,n}=\hat{\mathcal{X}}_{k_0|k_0,n}\boxminus\bar{\mathcal{X}}_{k_0,n}$, where $\bar{\mathcal{X}}_{k_0,n}$ is the current tangent point of $\mathcal{X}_{k_0,n}$. According to \cite{BP3}, the posterior of error state $\delta\mathbf{x}_{k_0,n}=\mathcal{X}_{k_0,n}\boxminus\bar{\mathcal{X}}_{k_0,n}$ can be approximated as $p(\delta\mathbf{x}_{k_0,n}|\mathbf{z}_{1:k_0})={\rm N}(\delta\mathbf{x}_{k_0,n};\delta\bar{\mathbf{x}}_{k_0,n},\bar{\mathbf{P}}^{\mathcal{P}}_{k_0|k_0,n})$, where $\bar{\mathbf{P}}^{\mathcal{P}}_{k_0|k_0,n}=\mathbf{J}_{r}^{-1}(\delta\bar{\mathbf{x}}_{k_0,n})\mathbf{P}^{\mathcal{P}}_{k_0|k_0,n}\mathbf{J}_{r}^{-\top}(\delta\bar{\mathbf{x}}_{k_0,n})$. Denote prior factor
\begin{align}
  f^{\mathcal{P}}_{k_0,n}(\delta\mathbf{x}_{k_0,n})=p(\delta\mathbf{x}_{k_0,n}|\mathbf{z}_{1:k_0}).\label{factor_prior}
\end{align}

\subsubsection{Factor Graph}
Define $\boldsymbol{\Theta}_{k,n}=\bigcup_{m\in\mathcal{N}_{k,n}}\boldsymbol{\Theta}_{k,n}^{m}$, where $\boldsymbol{\Theta}_{k,n}^{m}=\{\xi_{k,n}^{m},y_{k,n}^{m},\pi_{k,n}^{m}\}$ contains all noise parameters in GSTM. Let latent variable sets
\begin{align}
  \boldsymbol{\Phi}_{s:k,n}=&\bigcup_{t=s}^{k}\left\{\delta\mathbf{x}_{t,n}\right\}\cup\bigcup_{t=s+1}^{k}\left\{\mathbf{P}_{t,n}^{\mathcal{O}}\right\}\cup\bigcup_{t=s+1}^{k}\boldsymbol{\Theta}_{t,n},
\end{align}
and $\boldsymbol{\Phi}=\bigcup_{n=1}^{N}\boldsymbol{\Phi}_{k_{0}:k,n}$. The set $\boldsymbol{\Phi}$ consists of all latent variables that need to be estimated at time step $k$. The posterior $p(\boldsymbol{\Phi}|\mathbf{z}_{1:k})$ can be calculated by $p(\boldsymbol{\Phi}|\mathbf{z}_{1:k})\propto f(\boldsymbol{\Phi})$, where $f(\boldsymbol{\Phi})=p(\boldsymbol{\Phi},\mathbf{z}_{k_{0}+1:k}|\mathbf{z}_{1:k_{0}})$. For the ease of distributed design, we approximate $p(\{\delta\mathbf{x}_{k_{0},n}\}_{n\in\mathcal{N}}|\mathbf{z}_{1:k_0})=\prod_{n=1}^{N}f_{k_0,n}^{\mathcal{P}}$. With noninformative prior $p(\delta\mathbf{x}_{s,n})\propto1$ for all $s=k_{0}+1,\ldots,k$ and $n=1,\ldots,N$, $f(\boldsymbol{\Phi})$ can be written as
\begin{align*}
  f\left(\boldsymbol{\Phi}\right)\propto&\prod_{n=1}^{N}\Big(f_{k_0,n}^{\mathcal{P}}\prod_{s=k_0+1}^{k}\Big(f_{s,n}^{\mathcal{O}}f_{s,n}^{\mathcal{C}}f_{s,n}^{\mathcal{G}}\notag\\
  &\times\prod_{m\in\mathcal{N}_{s,n}}f_{s,n}^{\mathcal{U},m}f_{s,n}^{\xi,m}f_{s,n}^{\pi,m}f_{s,n}^{y,m}\Big)\Big),
\end{align*}
The factor graph is given in Fig. \ref{fig_Factorgraph}. We assign all state variables to BP part, while the remaining variables, i.e., noise parameters, are all assigned to MF part. In this case, the factor nodes connected to BP variable nodes all satisfy (\ref{factor_model}).

\begin{figure}[tb]
  \centering
  \includegraphics[scale=1.09]{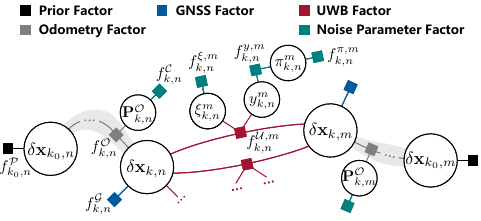}
  \caption{Factor graph for the multirobot collaborative localization problem, including odometry factors (\ref{factor_odo}), GNSS factors (\ref{factor_gnss}), UWB factors (\ref{factor_uwb}), prior factors (\ref{factor_prior}), and noise parameter factors (\ref{factor_odo_R}), (\ref{factor_uwb_R}). The circles are variable nodes, where the larger ones correspond to error states, and the smaller ones represent uncertain noise parameters. The squares depict factor nodes. The gray shaded regions schematically illustrate state estimation uncertainty.}
  \label{fig_Factorgraph}
\end{figure}

\subsection{Gradient Estimators}
\label{subsec::gradient_estimators_app}
For iteration $l$, denote beliefs 
\begin{align*}
  b^{(l)}(y_{k,n}^{m})=&{\rm Bern}(y_{k,n}^{m};\hat{y}_{k,n}^{m,(l)}),\notag\\
  b^{(l)}(\pi_{k,n}^{m})=&{\rm Beta}(\pi_{k,n}^{m};\hat{a}_{k,n}^{m,(l)},\hat{b}_{k,n}^{m,(l)}),\notag\\
  b^{(l)}(\xi_{k,n}^{m})=&{\rm G}(\xi_{k,n}^{m};g_{k,n}^{m,(l)},f_{k,n}^{m,(l)}),\notag\\
  b^{(l)}(\mathbf{P}^{\mathcal{O}}_{k,n})=&{\rm IW}(\mathbf{P}^{\mathcal{O}}_{k,n};\hat{t}^{\mathcal{O},(l)}_{k,n},\hat{\mathbf{T}}^{\mathcal{O},(l)}_{k,n}),
\end{align*}
all of which adopt exponential family forms given in Appendix \ref{Examples_Exponential_Families}. Then, noise parameter factors $f_{k,n}^{\mathcal{C}}$, $f_{k,n}^{\xi,m}$, $f_{k,n}^{\pi,m}$, and $f_{k,n}^{y,m}$ are conjugate factors. The measurement factors $f_{k,n}^{\mathcal{O}}$, $f_{k,n}^{\mathcal{G}}$, $f_{k,n}^{\mathcal{U},m}$ are non-conjugate for error states while prior factor $f_{k_0,n}^{\mathcal{P}}$ is conjugate. For $f_{k,n}^{\mathcal{O}}$, $f_{k,n}^{\mathcal{G}}$, $f_{k,n}^{\mathcal{U},m}$, and $f_{k_0,n}^{\mathcal{P}}$, we have provided the calculation of $(\nabla_{\boldsymbol{\mu}_{a}}U_{a})^{(l)}$ used in step \ref{Alg_BPMF_grad_1} of Algorithm \ref{Alg_BPMF}.

The gradients of conjugate factors $f_{k,n}^{\mathcal{C}}$, $f_{k,n}^{\xi,m}$, $f_{k,n}^{\pi,m}$, and $f_{k,n}^{y,m}$ have closed-form expressions. Substituting specific factor models into (\ref{U_a}), for the messages from $f^{\mathcal{C}}_{k,n}$ to $\mathbf{P}^{\mathcal{O}}_{k,n}$, from $f^{\xi,m}_{k,n}$ to $\xi_{k,n}^{m}$, and from $f^{\pi,m}_{k,n}$ to $\pi_{k,n}^{m}$, the gradient $(\nabla_{\boldsymbol{\mu}_{i}}U_{a})^{(l)}$ should be natural parameters of $f^{\mathcal{C}}_{k,n}$, $f^{\xi,m}_{k,n}$, and $f^{\pi,m}_{k,n}$, respectively. For the message from $f^{y,m}_{k,n}$ to $y_{k,n}^{m}$,
\begin{align*}
  (\nabla_{\boldsymbol{\mu}_{i}}U_{a})^{(l)}={\rm E}[\log{\pi}_{k,n}^{m}]-{\rm E}[\log(1-{\pi}_{k,n}^{m})].
\end{align*}
For the message from $f^{y,m}_{k,n}$ to $\pi_{k,n}^{m}$,
\begin{align*}
  (\nabla_{\boldsymbol{\mu}_{i}}U_{a})^{(l)}=\left[\begin{matrix}
        {\rm E}[y_{k,n}^{m}]\\
        1-{\rm E}[y_{k,n}^{m}]
    \end{matrix}\right].
\end{align*}
For the message from $f^{\mathcal{O}}_{k,n}$ to $\mathbf{P}^{\mathcal{O}}_{k,n}$, 
\begin{align*}
    (\nabla_{\boldsymbol{\mu}_{i}}U_{a})^{(l)}=-0.5\left[\begin{matrix}
        1\\
        {\rm vec}\big(\mathbf{A}_{k,n}^{(l)}\big)
    \end{matrix}\right],
\end{align*}
where $\mathbf{A}_{k,n}^{(l)}={\rm E}[\mathbf{r}_{k,n}^{\mathcal{O}}(\mathbf{r}_{k,n}^{\mathcal{O}})^{\top}]$. From $f^{\mathcal{U},m}_{k,n}$ to $y_{k,n}^{m}$,
\begin{align*}
  (\nabla_{\boldsymbol{\mu}_{i}}U_{a})^{(l)}=&0.5\left(B_{k,n}^{m,(l)}{\rm E}[\xi_{k,n}^{m}]/P_{k,n,0}^{\mathcal{U},m}-{\rm E}[\log\xi_{k,n}^{m}]\right.\notag\\
  &\left.-B_{k,n}^{m,(l)}/P_{k,n}^{\mathcal{U},m}+\log P_{k,n,0}^{\mathcal{U},m}-\log P_{k,n}^{\mathcal{U},m}\right),
\end{align*}
where $B_{k,n}^{m,(l)}={\rm E}[r_{k,n}^{\mathcal{U},m}(r_{k,n}^{\mathcal{U},m})^{\top}]$. From $f^{\mathcal{U},m}_{k,n}$ to $\xi_{k,n}^{m}$,
\begin{align*}
  (\nabla_{\boldsymbol{\mu}_{i}}U_{a})^{(l)}=0.5\left[\begin{matrix}
      (1-{\rm E}[y_{k,n}^{m}])\\
      -(1-{\rm E}[y_{k,n}^{m}])B_{k,n}^{m,(l)}/P_{k,n,0}^{\mathcal{U},m}
    \end{matrix}\right].
\end{align*}

Expectations $\mathbf{A}_{k,n}^{(l)}$ and $B_{k,n}^{m,(l)}$ are calculated by (\ref{A_Taylor}), (\ref{A_sampling}), or (\ref{A_NN}). The other expectations involved above are given by
\begin{align*}
  &{\rm E}[\log{\pi}_{k,n}^{m}]=\psi(\hat{a}_{k,n}^{m,(l-1)})-\psi(\hat{a}_{k,n}^{m,(l-1)}+\hat{b}_{k,n}^{m,(l-1)}),\notag\\
  &{\rm E}[\log(1-{\pi}_{k,n}^{m})]=\psi(\hat{b}_{k,n}^{m,(l-1)})-\psi(\hat{a}_{k,n}^{m,(l-1)}+\hat{b}_{k,n}^{m,(l-1)}),\notag\\
  &{\rm E}[y_{k,n}^{m}]=\hat{y}_{k,n}^{m,(l-1)},{\rm E}[\xi_{k,n}^{m}]=g_{k,n}^{m,(l-1)}/f_{k,n}^{m,(l-1)},\notag\\
  &{\rm E}[\log\xi_{k,n}^{m}]=\psi(g_{k,n}^{m,(l-1)})-\log f_{k,n}^{m,(l-1)},
\end{align*}
where $\psi(\cdot)$ is digamma function. 

\subsection{Some Implementation Details}
\subsubsection{The Conditions Used in Real NVP} We apply the NF-based gradient estimator for UWB factors. The condition $\boldsymbol{\theta}^{(l)}$ in mapping $T_{\boldsymbol{\theta}^{(l)}}(\mathbf{y})$ consists of relative quantities of state estimates of two robots, measurements, belief covariances, and noise covariances. Since translational components of states lie in an unbounded Euclidean space and are difficult for networks to learn reliably, we use the relative quantities. Specifically, $\boldsymbol{\theta}^{(l)}={\rm col}(\boldsymbol{\delta}_{k,n}^{m,(l)},z^{\mathcal{U},m}_{k,n},\mathcal{C}(\mathbf{P}_{a}^{(l-1)}),\mathcal{C}(\boldsymbol{\Lambda}_{a}^{(l-1)}))$. The relative quantity $\boldsymbol{\delta}_{k,n}^{m,(l)}$ is given as
\begin{align}
  \boldsymbol{\delta}_{k,n}^{m,(l)}=(\bar{\mathcal{X}}_{k,n}\boxplus\delta \hat{\mathbf{x}}_{k,n}^{(l-1)})\boxminus(\bar{\mathcal{X}}_{k,m}\boxplus\delta \hat{\mathbf{x}}_{k,m}^{(l-1)}),\notag
\end{align}
where $\bar{\mathcal{X}}_{k,n}$ and $\bar{\mathcal{X}}_{k,m}$ are tangent points corresponding to the states of robot $n$ and $m$, respectively, $\delta \hat{\mathbf{x}}_{k,n}^{(l-1)}$ and $\delta \hat{\mathbf{x}}_{k,m}^{(l-1)}$ are factor belief means of $\delta\mathbf{x}_{k,n}$ and $\delta\mathbf{x}_{k,m}$ at the $(l-1)$th iteration, respectively. The function $\mathcal{C}(\cdot)$ is used to extract features of covariances. For an input $\boldsymbol{\Sigma}\in\mathbb{R}^{d\times d}$,
\begin{align}
  \mathcal{C}(\boldsymbol{\Sigma})=&{\rm col}\left(\boldsymbol{\omega},\boldsymbol{\chi}\right),\notag
\end{align}
where $\boldsymbol{\omega}={\rm col}(\log(\sigma_i+\varepsilon))_{1\leq i\leq d}$, $\boldsymbol{\chi}={\rm col}(\chi_{ij})_{1\leq i<j\leq d}$, $\varepsilon=10^{-8}$ is a small constant to ensure numerical stability,
\begin{align}
\sigma_i &= \sqrt{\Sigma_{ii}},\ R_{ij} = \frac{\Sigma_{ij}}{\sigma_i \sigma_j + \varepsilon},\notag\\
\chi_{ij} &={\rm atanh}(R_{ij})= \frac{1}{2}\log\left(
\frac{1+R_{ij}}{1-R_{ij}}\right),\notag
\end{align}
and $\Sigma_{ij}$ is the $(i,j)$th entry of $\boldsymbol{\Sigma}$. To provide a more informative representation of scale differences, the logarithm is applied to $\sigma_i$. The term $R_{ij}$ represents the correlation coefficient. To map the bounded interval $(-1,1)$ of $R_{ij}$ to the entire real axis, the Fisher z-transform is used to yield $\chi_{ij}$.

\subsubsection{Construction of Real NVP}
Each algorithm iteration uses a separate Real NVP. Within the same iteration, however, all factors originating from the same sensor type share the same Real NVP module. The Real NVP consists of $4$ affine coupling layers. The input of each MLP is the concatenation of the original input vector $\mathbf{b}^{j}\odot\mathbf{y}^{j-1}$ in (\ref{acl}) and the condition vector $\boldsymbol{\theta}$. The MLP consists of $3$ fully connected layers. The first layer maps the input to a dimension given by the integer-truncated average of MLP input and output dimensions. The second layer maps it to the output dimension of MLP. The third layer keeps the output dimension. ReLU activation functions are applied after the first two layers. To ensure numerical stability during training, the logarithm of $w(\mathbf{y}_{a,s})$ in (\ref{NN_mean}), (\ref{NN_cov}), (\ref{A_NN}) is clamped to the range $[-10,10]$, while in evaluation this clamp is removed to ensure the NF-based gradient estimator is strictly consistent with that in (\ref{NN_mean}), (\ref{NN_cov}), and (\ref{A_NN}).

\subsubsection{End-to-End Training}
We train NF parameters in an end-to-end manner. NF parameters are updated every $\Delta k$ time steps. The loss function at time step $k=\tau\Delta k$, with $\tau\in\mathbb{N}^{+}$, is defined as
\begin{align*}
  \mathcal{L}_{k}=\frac{1}{BN\Delta k}\sum_{l=1}^{L}\sum_{b=1}^{B}\sum_{n=1}^{N}\sum_{s=k-\Delta k+1}^{k}\varrho (l)\big\Vert\mathcal{X}_{s,n}^{b}\boxminus\hat{\mathcal{X}}_{s,n}^{(l),b}\big\Vert^{2}_{2},
\end{align*}
where $B$ is the mini-batch size, $\hat{\mathcal{X}}_{s,n}^{(l),b}$ is the state estimate of the $n$th robot at time step $s$ obtained from MP at iteration $l$ for the $b$th sample in the mini-batch, $\mathcal{X}_{s,n}^{b}$ is the corresponding ground truth, $\varrho (l)=\exp(-0.05(L-l))$ is a monotonically increasing weight with respect to the iteration index $l$, designed to emphasize the accuracy of later MP iterations.

\subsubsection{Distributed Computation and Decentralized Fusion}
The state variables can be allocated to the robots that own the corresponding states, while factors and their associated noise parameters are assigned to the robots that obtain the corresponding measurements. The belief and message updates are carried out locally, enabling distributed computation. When a factor and its connected variables reside on different robots, the messages between them are exchanged across robots. Specifically, cross-robot communication includes variable to factor messages $\mathbf{n}_{i\to a}^{(l)}$, factor to variable messages $\mathbf{m}_{a\to i}^{(l)}$, as well as tangent points of states connected to cross-robot factors. The received messages are used for local message updates, while tangent points are required to evaluate cross-robot residuals and Jacobians and to construct the NF conditions. This process does not rely on any central node and thus achieves a completely distributed fusion scheme.

\section{Simulations and Experiments}
\label{sec:simexp}

\subsection{Methods Included in Verification}

GTSAM \cite{bibGTSAM}: A widely used centralized factor graph optimizer. Only odometry, GNSS, UWB, prior factors, and state variables are modeled. To enhance robustness, Huber m-estimators with threshold set as $1$ were constructed for UWB and odometry factors in our implementation.

MFVI \cite{bibMFVI}: Standard MFVI that uses the factor graph in Section \ref{sec::multirobot_Collaborative_Localization} with MF assumption that all variables are independent. The nonlinear residuals are handled by linearization.

GBP-L, GBP-S, and GBP-NF: Simplified versions of our method, Algorithm \ref{Alg_BPMF}, using linearization, sampling, and NF-based gradient estimators, respectively. Only odometry, GNSS, UWB, prior factors, and state variables are modeled.

GBP-DCS-n \cite{BP3} and GBP-DCS: GBP-DCS-n combines existing nonlinear GBP with DCS. GBP-DCS is a robust extension of GBP-L, using DCS \cite{DCS} to adjust noise covariances in the presence of outliers. For both, DCS is used for UWB and odometry factors, with algorithm parameter set as in \cite{BP3}.

MP-L, MP-S, and MP-NF: Full versions of Algorithm \ref{Alg_BPMF}, using the factor graph in Section \ref{sec::multirobot_Collaborative_Localization}, with linearization, sampling, and NF-based gradient estimators, respectively.

In summary, GTSAM, MFVI, and GBP-DCS-n are baselines, while GBP-L, GBP-S, GBP-NF, GBP-DCS, MP-L, MP-S are variants of the proposed MP-NF. To balance computation overhead and the handling of nonlinear dynamics, for GBP-S, GBP-NF, MP-S, and MP-NF, only UWB factors use sampling or NF-based gradient estimators, the odometry, GNSS, and prior factors still adopt linearization-based gradient estimators. For Lie group verification, all methods, except GTSAM, employ Algorithm \ref{Alg2} to enable inference on Lie groups, whereas GTSAM uses its built-in Lie group support.

\subsection{Euclidean Space Simulation Verification}
We first verify the effectiveness of proposed methods in a Euclidean space. Compared to Lie groups, this avoids the need for manifold and Euclidean space mappings, thereby providing a more direct means of verifying inference algorithms. The factor graph follows that in Section \ref{sec::multirobot_Collaborative_Localization}, except that orientation parts are removed from states and factors. Specifically, in this subsection the state $\mathcal{X}_{k,n}$ is reduced to position, $\mathcal{X}_{k,n}=\mathbf{p}_{k,n}\in\mathbb{R}^{3}$, and measurement models (\ref{measurement_model_odo}), (\ref{measurement_model_gnss}), and (\ref{measurement_model_uwb}) are respectively reduced to
\begin{align*}
  \mathbf{z}^{\mathcal{O}}_{k,n}=&\mathbf{p}_{k,n}-\mathbf{p}_{k-1,n}+\mathbf{v}^{\mathcal{O}}_{k,n},\notag\\
  \mathbf{z}^{\mathcal{G}}_{k,n}=&\mathbf{p}_{k,n}+\mathbf{v}^{\mathcal{G}}_{k,n},\notag\\
  z^{\mathcal{U},m}_{k,n}=&\left\Vert\mathbf{p}_{k,n}-\mathbf{p}_{k,m}\right\Vert+v_{k,n}^{\mathcal{U},m}.
\end{align*}
Other components in factor graph follow the same formulation as in Section \ref{sec::multirobot_Collaborative_Localization}, while orientation-related terms are omitted.

During the evaluation stage, the ground truths of states are generated as $\mathcal{X}_{k,n}=\mathcal{X}_{k-1,n}+\Delta\mathbf{p}_{k,n}$, where $\Delta\mathbf{p}_{k,n}\in\mathbb{R}^{3}$ and initial positions $\mathcal{X}_{0,n}\in\mathbb{R}^{3}$ consist of elements uniformly distributed on $[0,2)$ and $[-10,10)$, respectively. Prior $p(\mathcal{X}_{0,n})={\rm N}(\mathcal{X}_{0,n};\hat{\mathcal{X}}_{0|0,n},\mathbf{P}^{\mathcal{P}}_{0|0,n})$, where $\hat{\mathcal{X}}_{0|0,n}=\mathcal{X}_{0,n}+\tilde{\mathbf{x}}_{0|0,n}$, $\tilde{\mathbf{x}}_{0|0,n}\sim{\rm N}(\mathbf{0}_{3\times 1},\mathbf{P}^{\mathcal{P}}_{0|0,n})$, and $\mathbf{P}^{\mathcal{P}}_{0|0,n}=10^{-1}\mathbf{I}_{3}$. Robot number $N=4$. For odometry and GNSS, noise covariances are set as $\mathbf{P}^{\mathcal{O}}_{k,n}=(1+0.1\sin\frac{k}{20})10^{-2}\mathbf{I}_{3}$ and $\mathbf{P}^{\mathcal{G}}_{k,n}=\mathbf{I}_{3}$, respectively. For UWB, noise is generated by
\begin{align}
  v_{k,n}^{\mathcal{U},m}\sim\left\{\begin{array}{ll}
    {\rm N}(0,10^{-2}), & 1\leq k\leq20\\
    \left\{\begin{array}{ll}
      {\rm N}(0,1), & {\rm w.p.} 0.05\\
      {\rm N}(0,10^{-2}), & {\rm w.p.} 0.95
    \end{array}\right. & 21\leq k\leq40\\
    {\rm N}(0,10^{-2}), & 41\leq k\leq60\\
    \left\{\begin{array}{ll}
      {\rm N}(0,4), & {\rm w.p.} 0.05\\
      {\rm N}(0,10^{-2}), & {\rm w.p.} 0.95
    \end{array}\right. & 61\leq k\leq80\\
    {\rm N}(0,10^{-2}), & 81\leq k\leq100\\
  \end{array}\right.\label{noise_setting}
\end{align}
where $\rm w.p.$ denotes ``with probability." For all algorithms, maximum number of iterations $L=5$, sliding window size is set as $3$, and $\mathbf{P}^{\mathcal{G}}_{k,n}$ is known exactly. For GTSAM and all GBP, the odometry and UWB noise covariance are set as $\mathbf{P}^{\mathcal{O}}_{k,n}=10^{-2}\mathbf{I}_{3}$ and $P^{\mathcal{U},m}_{k,n}=10^{-2}$, respectively. For MFVI, MP-L, MP-S, and MP-NF, the odometry noise prior parameters $t^{\mathcal{O}}_{1,n}=5$, $\mathbf{T}^{\mathcal{O}}_{1,n}=t^{\mathcal{O}}_{1,n}\times10^{-2}\mathbf{I}_{3}$, $\rho=0.99$, UWB noise prior parameters $a_{k,n}^{m}=0.8$, $\nu_{k,n}^{m}=7$, $P^{\mathcal{U},m}_{k,n}=10^{-2}$ and $P^{\mathcal{U},m}_{k,n,0}=4P^{\mathcal{U},m}_{k,n}$. The initial beliefs of $\mathbf{P}^{\mathcal{O}}_{k,n}$, $\xi_{k,n}^{m}$, and $\pi_{k,n}^{m}$ are set as their priors, while the initial beliefs of $y_{k,n}^{m}$ are set as ${\rm Bern}(y_{k,n}^{m};1)$. The number of samples of sampling and NF-based algorithms is $4$. Each simulation runs for $T=50$ independent Monte Carlo runs.

The training environment is set similar to above evaluation environment, with the following modifications to assess generalization across diverse scenarios: The components of position increment $\Delta\mathbf{p}_{k,n}$ are uniformly distributed on $[-1,1)$. The odometry noise covariance is $\mathbf{P}^{\mathcal{O}}_{k,n}=(1+0.1\sin\frac{k}{20})10^{-2+r_{k,n}}\mathbf{I}_{3}$, where $r_{k,n}$ is uniformly distributed on $[-3,-1)$. The number of samples is $16$. For MP-NF, $P^{\mathcal{U},m}_{k,n,0}=10^{-2}$. The ranging noise $v_{k,n}^{\mathcal{U},m}$ is Gaussian distributed with variance $10^{-2}$ w.p. $0.95$, and with variance $1$ w.p. $0.05$. The networks are trained using the Adam optimizer provided by PyTorch \cite{bibPyTorch} with a learning rate of $10^{-3}$ and weight decay of $10^{-4}$. The training dataset contains $640$ simulated sequences, each with $100$ time steps. The mini-batch size is $64$. At the end of each time step, the current beliefs and the state estimates and covariances propagated to subsequent time steps are detached from the computational graph, preventing gradient propagation across time steps. NF parameters are updated using data from the most recent $10$ time steps. After completing all $100$ time steps, $5\%$ of the sequences in the mini-batch are replaced and the mini-batch is shuffled. This process is repeated $20$ times.

Training and evaluation are both run on a personal computer with Intel Core i5-14600K CPU @ 3.50 GHz, NVIDIA GeForce RTX 4070 GPU, and 32 GB RAM. The algorithms are implemented in Python. Our training code is primarily based on PyTorch \cite{bibPyTorch} and PyTorch3D \cite{bibPyTorch3D}, while the evaluation relies on NumPy \cite{bibNumPy} and SciPy \cite{bibSciPy}. NF inference in evaluation is executed through ONNX Runtime on CPU. The training process of GBP-NF and MP-NF only takes $9.94$ minutes and $21.04$ minutes to complete, respectively.

The root mean square error (RMSE) is defined as
\begin{align}
  {\rm RMSE}(k)=\textstyle\sqrt{\frac{1}{TN}\sum_{t=1}^{T}\sum_{n=1}^{N}\Vert\mathcal{X}_{k,n}^{t}-\hat{\mathcal{X}}_{k,n}^{t}\Vert_{2}^{2}},\notag
\end{align}
where superscript $t$ indicates the $t$th simulation, $\mathcal{X}_{k,n}^{t}$ and $\hat{\mathcal{X}}_{k,n}^{t}$ are ground truth and estimated state of robot $n$ at time step $k$ in $t$th simulation. The RMSEs of all algorithms are presented in Fig. \ref{fig_simulation_1}. When ranging outliers occur, the estimation errors of GBP-L and GBP-S increase markedly, and GBP-NF also exhibits a noticeable rise in the second shaded region. In contrast, MP-L, MP-S, and MP-NF maintain relatively stable estimation errors in the shaded regions, demonstrating improved robustness against ranging outliers. Although MFVI employs robust design, the MF assumption that variables are all independent is overly restrictive, resulting in the highest error. GBP-DCS and GBP-DCS-n exhibit similar performance, indicating that parameterization itself has limited impact on the algorithm's effectiveness. It mainly serves as a conceptual bridge in the design, while the performance is primarily determined by gradient estimators it introduces. With the incorporation of NFs and end-to-end training, MP-NF achieves the highest estimation accuracy, significantly outperforming algorithms of the same class and even surpassing centralized GTSAM. Its task-oriented training allows the learned sampling strategy to account for how gradient estimation propagate through subsequent message updates, thereby improving the overall estimation performance, even better than GTSAM.

\begin{table*}[t]
\centering
    \scriptsize
  \caption{The Localization Errors and Average Time Per Iteration in Euclidean Space Verification}
  \vspace{-1ex}
  \label{tab_sim1}
  \setlength{\tabcolsep}{3.2mm}{
    \renewcommand{\arraystretch}{1.1}
    \setlength{\aboverulesep}{0pt}
  \setlength{\belowrulesep}{0pt}
\begin{tabular}{ccccccccccc}
  \Xhline{0.75pt}\noalign{\vskip 0.3mm}
\multirow{2}{*}{Metrics}  & \multicolumn{3}{c}{Baselines} & \multicolumn{6}{c}{Proposed Variants} & \multirow{2}{*}{MP-NF (ours)}\\\cmidrule(lr){2-4} \cmidrule(lr){5-10}\noalign{\vskip 0.3mm}
 & GTSAM & MFVI & GBP-DCS-n & GBP-DCS & GBP-L & MP-L & GBP-S & MP-S & GBP-NF &  \\\noalign{\vskip 0.3mm}
\hline
\noalign{\vskip 0.3mm}
$\rm ARMSE$ [m] & 0.439 & 0.695 & 0.534 & 0.534 & 0.602 & 0.514 & 0.606 & 0.513 & 0.439 & \textbf{0.425} \\
Avg. Time [ms] & 0.732 & 0.828 & 0.566 & 0.574 & \textbf{0.544} & 0.936 & 0.689 & 0.988 & 1.152 & 1.190 \\\noalign{\vskip 0.3mm}\Xhline{0.75pt}
\end{tabular}}
\end{table*}

\begin{figure}[tb]
  \centering
  \includegraphics[scale=1]{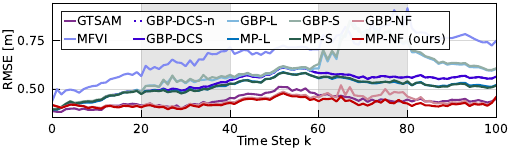}
  \caption{Estimation errors in Euclidean space verification. The two shaded regions mark occurrences of UWB ranging outliers. The curves of GBP-DCS and GBP-DCS-n, as well as those of MP-L and MP-S, are nearly identical.}
  \label{fig_simulation_1}
\end{figure}

\begin{figure}[tb]
  \centering
  \includegraphics[scale=1]{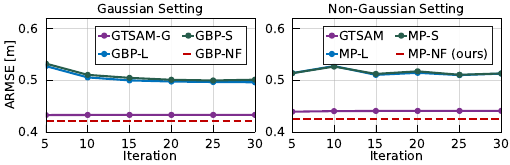}
  \caption{Estimation errors versus the number of algorithm iterations. The red dashed lines indicate the ARMSEs achieved by GBP-NF and MP-NF after 5 iterations. The curves of MP-L and MP-S are close.}
  \label{fig_simulation_2}
\end{figure}

To provide a more comprehensive quantitative assessment of each algorithm's performance, we present the average RMSE (ARMSE) and average runtime for each algorithm iteration in Table \ref{tab_sim1}. The ARMSE is calculated by
\begin{align}
  {\rm ARMSE}=\textstyle\sqrt{\frac{1}{K+1}\sum_{k=0}^{K}{\rm RMSE}(k)^{2}},\notag
\end{align}
where $K$ is the number of time steps. Except for GTSAM, all algorithms considered are distributed, and thus their average time is measured relative to each robot. In contrast, GTSAM operates on a single computational unit, so its average runtime is not distributed across robots. As shown in Table \ref{tab_sim1}, MP-NF achieves the highest accuracy, followed by GBP-NF and GTSAM. Compared with other MP algorithms, MP-NF and GBP-NF demonstrate a substantial improvement in accuracy, underscoring the advantages of NF-based enhancement techniques. Furthermore, algorithms that integrate GBP with MF, i.e., MP-L, MP-S, and MP-NF, achieve higher estimation accuracy than those relying solely on GBP, highlighting that the joint design of BP and MF can significantly improve robustness. Although GBP-NF does not explicitly incorporate MF to achieve robust design, it still exhibits improved robustness.

To further validate the effectiveness of NF-enhanced approaches, we increase the number of algorithm iterations for methods without NF. To ensure a fair comparison and eliminate the influence of outliers and uncertain covariance matrices, GTSAM, GBP-L, GBP-S, and GBP-NF are tested under Gaussian measurement noises with known covariances. Specifically, the ranging noise $v_{k,n}^{\mathcal{U},m}\sim {\rm N}(0,10^{-2})$, and the odometry covariance is fixed as  $\mathbf{P}^{\mathcal{O}}_{k,n}=10^{-2}\mathbf{I}_{3}$, known to all methods in this test. The Huber m-estimators in GTSAM are removed, denoted by GTSAM-G. The GTSAM with Huber m-estimators, MP-L, MP-S, and MP-NF are still evaluated in the non-Gaussian setting used previously. The errors with respect to iterations in both Gaussian and non-Gaussian settings are shown in Fig. \ref{fig_simulation_2}. Even as the number of iterations increases and the error plateaus, the estimation errors of methods without NF remain higher than those of GBP-NF or MP-NF, underscoring the critical role of NF enhancement in achieving superior accuracy and reducing iterations. Although Table \ref{tab_sim1} shows that GBP-NF and MP-NF have higher average runtime, the results in Fig. \ref{fig_simulation_2} indicate that introducing NF is worthwhile, as they achieve sufficiently high accuracy within only $5$ iterations.

\subsection{Lie Group Simulation Verification}

We next validate the performance of proposed methods on a 3D state space with rotational and translational components. The state of robot $n$ at time step $k$ is $\mathcal{X}_{k,n}=(\mathbf{R}_{k,n},\mathbf{p}_{k,n})$ as in Section \ref{sec::multirobot_Collaborative_Localization}. In evaluation stage, ground truths of rotations and positions are generated as $\mathbf{R}_{k,n}=\mathbf{R}_{k-1,n}\Delta\mathbf{R}_{k,n}$ and $\mathbf{p}_{k,n}=\mathbf{p}_{k-1,n}+\mathbf{R}_{k-1,n}\Delta\mathbf{p}_{k,n}$, respectively, where increment $\Delta\mathbf{R}_{k,n}={\rm Exp}(\Delta\mathbf{r}_{k,n})$ and $\Delta\mathbf{r}_{k,n}$, $\Delta\mathbf{p}_{k,n}\in\mathbb{R}^{3}$ consist of elements uniformly distributed on $[-0.02\pi,0.02\pi)$ and $[-1,1)$. The initial orientation $\mathbf{R}_{0,n}={\rm Exp}(\mathbf{r}_{0,n})$. The initial rotation vector $\mathbf{r}_{0,n}\in\mathbb{R}^{3}$ and position $\mathbf{p}_{0,n}\in\mathbb{R}^{3}$ consist of elements uniformly distributed on $[-\pi,\pi)$ and $[-10,10)$. For all algorithms, the prior $p(\mathcal{X}_{0,n})={\rm N}(\mathcal{X}_{0,n};\hat{\mathcal{X}}_{0|0,n},\mathbf{P}^{\mathcal{P}}_{0|0,n})$ with $\hat{\mathcal{X}}_{0|0,n}=\mathcal{X}_{0,n}\boxplus\tilde{\mathbf{x}}_{0|0,n}$, $\tilde{\mathbf{x}}_{0|0,n}\sim{\rm N}(\mathbf{0}_{6\times 1},\mathbf{P}^{\mathcal{P}}_{0|0,n})$, and $\mathbf{P}^{\mathcal{P}}_{0|0,n}={\rm diag}(10^{-4}\mathbf{I}_{3},10^{-1}\mathbf{I}_{3})$. For odometry, the measurement noise covariances are set as $\mathbf{P}^{\mathcal{O}}_{k,n}=(1+0.1\sin\frac{k}{20}){\rm diag}(10^{-6}\mathbf{I}_{3},10^{-2}\mathbf{I}_{3})$. For GNSS and UWB, noises are set as that in Euclidean space verification. Robot number $N=4$. The extrinsic parameters $\mathbf{p}^{b_n}_{g_n}=[0.6,0,-0.2]^{\top}$ and $\mathbf{p}^{b_n}_{u_n}=[0,0,-0.4]^{\top}$.

\begin{figure}[tb]
  \centering
  \includegraphics[scale=1]{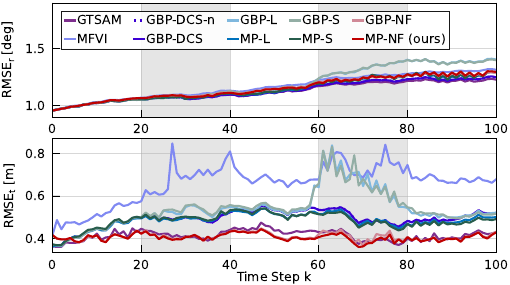}
  \caption{Rotational and translational errors. The two shaded regions mark occurrences of UWB ranging outliers. The curves of GBP-DCS and GBP-DCS-n, as well as those of MP-L and MP-S, are nearly identical.}
  \label{fig_simulation_4}
\end{figure}

\begin{figure}[tb]
  \centering
  \includegraphics[scale=1]{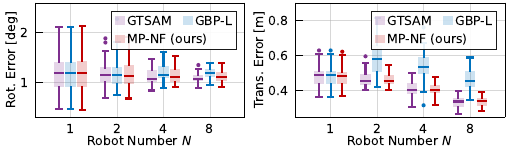}
  \caption{The box charts of rotational and translational errors with respect to different robot number $N$.}
  \label{fig_simulation_8}
\end{figure}

For all algorithms, maximum number of iterations $L=5$ and sliding window size is $3$. For GTSAM and all GBP, noise covariances of odometry, GNSS, and UWB are set as $\mathbf{P}^{\mathcal{O}}_{k,n}={\rm diag}(10^{-6}\mathbf{I}_{3},10^{-2}\mathbf{I}_{3})$, $\mathbf{P}^{\mathcal{G}}_{k,n}=\mathbf{I}_{3}$, and $P^{\mathcal{U},m}_{k,n}=10^{-2}$, respectively. For MP-L, MP-S, and MP-NF, odometry noise prior parameters $t^{\mathcal{O}}_{1,n}=8$, $\mathbf{T}^{\mathcal{O}}_{1,n}=t^{\mathcal{O}}_{1,n}{\rm diag}(10^{-6}\mathbf{I}_{3},10^{-2}\mathbf{I}_{3})$, forgetting parameter $\rho=0.99$, UWB noise prior parameters $a_{k,n}^{m}=0.8$, $\nu_{k,n}^{m}=7$, $P^{\mathcal{U},m}_{k,n}=10^{-2}$, and $P^{\mathcal{U},m}_{k,n,0}=4P^{\mathcal{U},m}_{k,n}$. The initial beliefs of $\mathbf{P}^{\mathcal{O}}_{k,n}$, $\xi_{k,n}^{m}$, and $\pi_{k,n}^{m}$ are set as their prior distributions, while the initial beliefs of $y_{k,n}^{m}$ are set as ${\rm Bern}(y_{k,n}^{m};1)$. The number of samples is set as $4$. Each simulation is repeated $T=30$.

For training, the odometry noise covariance is set as $\mathbf{P}^{\mathcal{O}}_{k,n}=(1+0.1\sin\frac{k}{20}){\rm diag}(10^{-6}\mathbf{I}_{3},10^{-2}\mathbf{I}_{3})\times10^{r_{k,n}}$, where $r_{k,n}$ is uniformly distributed on $[-3,-1)$. For MP-NF, $P^{\mathcal{U},m}_{k,n,0}=10^{-2}$. The ranging noise $v_{k,n}^{\mathcal{U},m}$ is Gaussian distributed with variance $10^{-2}$ w.p. $0.95$, and with variance $1$ w.p. $0.05$. The number of samples is $16$. Robot rotation increment $\Delta\mathbf{R}_{k,n}={\rm Exp}(\Delta\mathbf{r}_{k,n})$ and $\Delta\mathbf{r}_{k,n}$ consist of elements uniformly distributed on $[-0.01\pi,0.01\pi)$. The other settings of environment and algorithm parameters are the same as those in evaluation. The training method is the same as that in Euclidean space verification, except that NF parameters are updated every $5$ time steps. The training of GBP-NF and MP-NF only takes $28.41$ minutes and $36.82$ minutes, respectively.

\begin{table*}[t]
\centering
    \scriptsize
  \caption{The ARMSEs and Average Time Per Iteration in Lie Group Verification}
  \vspace{-1ex}
  \label{tab_sim4}
  \setlength{\tabcolsep}{2.9mm}{
    \renewcommand{\arraystretch}{1.1}
    \setlength{\aboverulesep}{0pt}
  \setlength{\belowrulesep}{0pt}
\begin{tabular}{ccccccccccc}
  \Xhline{0.75pt}\noalign{\vskip 0.3mm}
\multirow{2}{*}{Metrics}  & \multicolumn{3}{c}{Baselines} & \multicolumn{6}{c}{Proposed Variants} & \multirow{2}{*}{MP-NF (ours)}\\\cmidrule(lr){2-4} \cmidrule(lr){5-10}\noalign{\vskip 0.3mm}
 & GTSAM & MFVI & GBP-DCS-n & GBP-DCS & GBP-L & MP-L & GBP-S & MP-S & GBP-NF &  \\\noalign{\vskip 0.3mm}
\hline\noalign{\vskip 0.3mm}
$\rm ARMSE_r$ [deg] & \textbf{1.125} & 1.171 & 1.132 & 1.131 & 1.199 & 1.146 & 1.199 & 1.147 & 1.156 & 1.155 \\
$\rm ARMSE_t$ [m] & 0.414 & 0.665 & 0.494 & 0.494 & 0.544 & 0.485 & 0.546 & 0.484 & 0.409 & \textbf{0.405} \\
Avg. Time [ms] & 2.708 & 1.443 & 1.323 & 1.388 & \textbf{1.234} & 1.608 & 1.535 & 1.685 & 1.752 & 1.758 \\\noalign{\vskip 0.3mm}\Xhline{0.75pt}
\end{tabular}}
\end{table*}

\begin{figure*}[tb]
  \centering
  \includegraphics[scale=1]{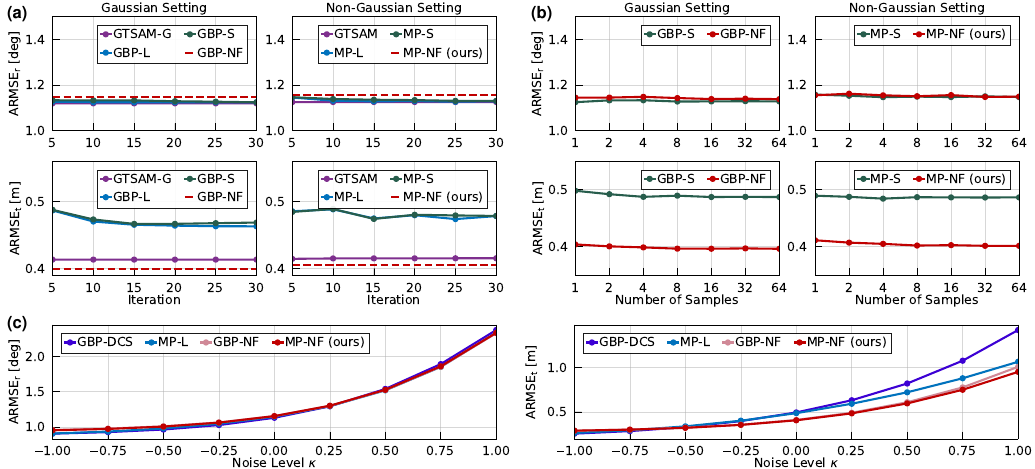}
  \caption{Sensitivity analysis. (a) Rotational and translational errors versus the number of algorithm iterations. The red dashed lines indicate the ARMSEs achieved by GBP-NF and MP-NF after 5 iterations. (b) Rotational and translational errors versus the number of samples. For (a) and (b), the left and right columns correspond to the Gaussian and non-Gaussian settings, respectively. (c) Rotational and translational errors for different noise level $\kappa$.}
  \label{fig_simulation_567}
\end{figure*}

The rotational and translational RMSEs are respectively defined as: 
\begin{align}
  {\rm RMSE_r}(k)=&\textstyle\sqrt{\frac{1}{TN}\sum_{t=1}^{T}\sum_{n=1}^{N}\Vert\mathbf{R}_{k,n}^{t}\ominus  \hat{\mathbf{R}}_{k,n}^{t}\Vert_{2}^{2}},\notag\\
  {\rm RMSE_t}(k)=&\textstyle\sqrt{\frac{1}{TN}\sum_{t=1}^{T}\sum_{n=1}^{N}\Vert\mathbf{p}_{k,n}^{t}-\hat{\mathbf{p}}_{k,n}^{t}\Vert_{2}^{2}},\notag
\end{align}
where superscript $t$ indicates the $t$th simulation, $\mathbf{R}_{k,n}$ and $\mathbf{p}_{k,n}$ are ground truths of attitude and position of robot $n$ at time step $k$, and $\hat{\mathbf{R}}_{k,n}$ and $\hat{\mathbf{p}}_{k,n}$ are estimates from algorithms. The rotational and translational RMSEs of all algorithms are shown in Fig. \ref{fig_simulation_4}. The rotational errors across different algorithms are similar, with the maximum discrepancy remaining below $0.25$ deg. For translational errors, consistent with the phenomena observed in Euclidean space simulations, errors of GBP-L and GBP-S increase significantly in the presence of ranging outliers. In contrast, the corresponding MP-L and MP-S exhibit substantially improved robustness to ranging outliers, demonstrating the benefit of the robust treatment introduced in MP framework. MP-NF achieves the highest accuracy. In Table \ref{tab_sim4}, we report the rotational and translational ARMSEs and average time per iteration. The rotational estimation performances are comparable across algorithms. For translational estimation, MP-NF achieves the highest accuracy, followed by GBP-NF, highlighting the benefits of NF-based enhancement. GBP-L achieves the lowest computational overhead owing to distributed computation design. The increase in runtime for GBP-NF and MP-NF is also minor.

\begin{table*}[t]
\centering
    \scriptsize
  \caption{The ARMSEs and Average Time Per Iteration Using Nonlinear Ranging Model (\ref{new_nonlinear})}
  \vspace{-1ex}
  \label{tab_sim9}
  \setlength{\tabcolsep}{2.5mm}{
    \renewcommand{\arraystretch}{1.1}
    \setlength{\aboverulesep}{0pt}
  \setlength{\belowrulesep}{0pt}
\begin{tabular}{cccccccccccc}
  \Xhline{0.75pt}\noalign{\vskip 0.3mm}
\multirow{2}{*}{Noise Cases} & \multirow{2}{*}{Metrics}  & \multicolumn{3}{c}{Baselines} & \multicolumn{6}{c}{Proposed Variants} & \multirow{2}{*}{MP-NF (ours)}\\\cmidrule(lr){3-5} \cmidrule(lr){6-11}\noalign{\vskip 0.3mm}
& & GTSAM & MFVI & GBP-DCS-n & GBP-DCS & GBP-L & MP-L & GBP-S & MP-S & GBP-NF &  \\\noalign{\vskip 0.3mm}\hline\noalign{\vskip 0.3mm}
\multirow{3}{*}{non-Gaussian} & $\rm ARMSE_r$ [deg] & \textbf{1.134} & 1.173 & 1.135 & 1.135 & 1.203 & 1.150 & 1.208 & 1.155 & 1.159 & 1.165 \\
&$\rm ARMSE_t$ [m] & \textbf{0.412} & 0.550 & 0.468 & 0.468 & 0.529 & 0.459 & 0.527 & 0.461 & 0.428 & 0.422 \\
&Avg. Time [ms] & 3.679 & 1.571 & 1.321 & 1.344 & \textbf{1.259} & 1.488 & 1.451 & 1.843 & 1.792 & 1.888 \\\noalign{\vskip 0.3mm}\hline\noalign{\vskip 0.3mm}
\multirow{3}{*}{Gaussian}
&$\rm ARMSE_r$ [deg] & 1.130 & 1.166 & 1.130 & \textbf{1.129} & 1.130 & 1.146 & 1.134 & 1.152 & 1.150 & 1.149 \\
&$\rm ARMSE_t$ [m] & \textbf{0.406} & 0.537 & 0.462 & 0.462 & 0.461 & 0.457 & 0.461 & 0.456 & 0.417 & 0.414 \\
&Avg. Time [ms] & 3.553 & 1.673 & 1.285 & 1.283 & \textbf{1.234} & 1.608 & 1.499 & 1.686 & 1.752 & 1.787 \\
\noalign{\vskip 0.3mm}\Xhline{0.75pt}
\end{tabular}}
\end{table*}

We evaluate the performance of the algorithms under varying numbers of robots. Fig. \ref{fig_simulation_8} presents box charts of the rotational and translational errors of GTSAM, GBP-L, and MP-NF with respect to robot numbers. The box charts use RMSEs $\sqrt{\frac{1}{(K+1)N}\sum_{k=0}^{K}\sum_{n=1}^{N}\Vert\mathbf{R}_{k,n}^{t}\ominus  \hat{\mathbf{R}}_{k,n}^{t}\Vert_{2}^{2}}$ and $\sqrt{\frac{1}{(K+1)N}\sum_{k=0}^{K}\sum_{n=1}^{N}\Vert\mathbf{p}_{k,n}^{t}-\hat{\mathbf{p}}_{k,n}^{t}\Vert_{2}^{2}}$ from each simulation run as data. The case $N=1$ corresponds to single robot localization relying solely on odometry and GNSS. For GTSAM and MP-NF, the position error decreases noticeably as the number of robots increases, indicating that collaborative localization framework provides significant benefits for enhancing accuracy. The performance of MP-NF and GTSAM is comparable. In contrast, GBP-L exhibits increasing error, suggesting that enhancing accuracy through collaboration requires algorithms to be robust against potential outliers.

We further analyze the performance of algorithms across different iterations. GTSAM-G (GTSAM without Huber m-estimators), GBP-L, GBP-S, and GBP-NF are tested under Gaussian noises with known covariances as that in Euclidean space verification except that $\mathbf{P}^{\mathcal{O}}_{k,n}={\rm diag}(10^{-6}\mathbf{I}_{3},10^{-2}\mathbf{I}_{3})$. GTSAM with Huber m-estimators, MP-L, MP-S, and MP-NF are evaluated in the non-Gaussian setting used before. The errors with respect to iterations are illustrated in Fig. \ref{fig_simulation_567} (a). For translation, GBP-NF and MP-NF reach  accuracy levels that other methods cannot attain even after many more iterations.

We also examine the choice of sampling points. Fig. \ref{fig_simulation_567} (b) illustrates the ARMSEs of the sampling-based algorithm slightly decrease as the number of sampling points increases. However, excessive sampling does not result in a noticeable improvement in accuracy. For the nonlinear UWB measurement model considered in this paper, setting the number of sampling points to $4$ is sufficient.

\begin{figure*}[tb]
  \centering
  \includegraphics[scale=1]{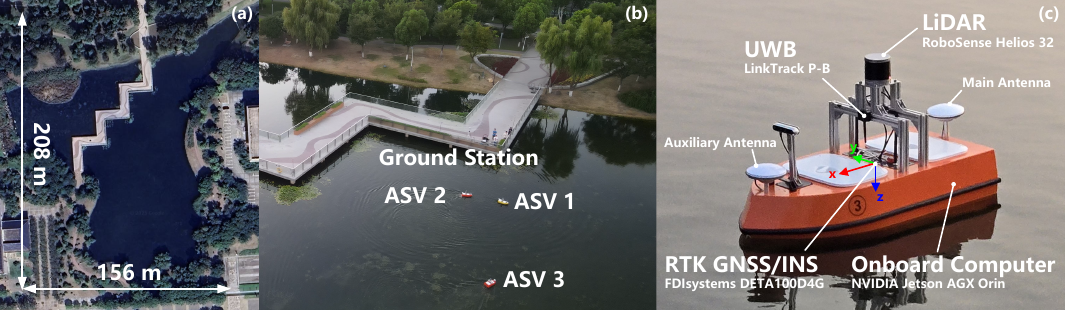}
  \caption{The experimental field and self-developed ASVs employed for dataset collection. (a) Satellite imagery of the test site. (b) Aerial photograph captured by a drone. (c) Close-up view of the ASV with annotations highlighting its body frame and key components.}
  \label{field}
\end{figure*}

We then assess the adaptability of the proposed methods. The odometry noise covariance is adjusted as $\mathbf{P}^{\mathcal{O}}_{k,n}=(1+0.1\sin\frac{k}{20})\times {\rm diag}(10^{-6}\mathbf{I}_{3},10^{-2}\mathbf{I}_{3}) \times10^{\kappa}$, where noise level $\kappa$ ranges from $-1$ to $1$. All algorithms use the same parameter settings as before. This configuration simulates varying degrees of misrecognition of noise covariance matrices. As can be seen from Fig. \ref{fig_simulation_567} (c), for rotation, the performance of different algorithms is nearly identical. For position, GBP-DCS exhibits the most pronounced error growth as $\kappa$ increases. Comparing GBP-NF with MP-NF, the error of GBP-NF increases more rapidly. Together with the relatively stable performance of MP-L with respect to GBP-DCS, these results demonstrate that the incorporation of odometry noise parameter factors enables MP-NF and MP-L to better adapt to increasingly uncertain covariance matrices.

To further evaluate the applicability of the proposed method beyond UWB ranging, we consider an RSSI-type log-distance inter-robot measurement model:
\begin{align}
  z_{k,n}^{\mathcal{U},m}=\log\left\Vert\Delta\mathbf{p}_{k,n}^{m}\right\Vert+v_{k,n}^{\mathcal{U},m}.\label{new_nonlinear}
\end{align}
Both Gaussian and non-Gaussian noise settings are considered. For the non-Gaussian case, the noise $v_{k,n}^{\mathcal{U},m}$ follows the same mixture setting as (\ref{noise_setting}), while its noise variances and the corresponding algorithm settings are set to $0.01$ of their original values. The Real NVP consists of $5$ affine coupling layers. GBP-NF and MP-NF are re-trained under the new measurement model. During training, the logarithm of weight $w(\mathbf{y}_{a,s})$ is clamped to range $[-8,8]$, while the other settings remain unchanged. The results are summarized in Table \ref{tab_sim9}. GTSAM achieves the best overall translational accuracy, while MP-NF attains the lowest translational ARMSE among the distributed methods, with only approximately $2.4\%$ and $2.0\%$ higher translational errors than GTSAM under non-Gaussian and Gaussian noise, respectively. Compared with MP-S, MP-NF reduces the translational ARMSE by approximately $8.5\%$ and $9.2\%$ under non-Gaussian and Gaussian noise, respectively, and also slightly outperforms GBP-NF. The proposed framework has potential to achieve reliable performance across different measurement models.

\subsection{ASV Field Dataset Verification}

\begin{figure}[tb]
  \centering
  \includegraphics[scale=1]{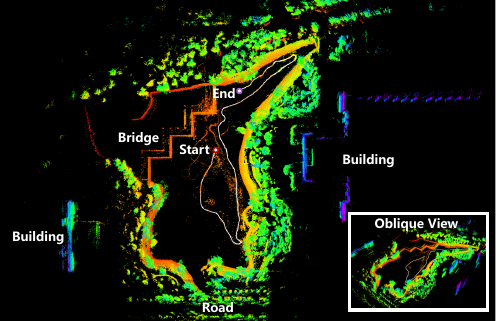}
  \caption{The estimated trajectory (white curve) and maps from LeGO-LOAM.}
  \label{fig_LO}
\end{figure}

\begin{figure*}[tb]
  \centering
  \includegraphics[scale=1]{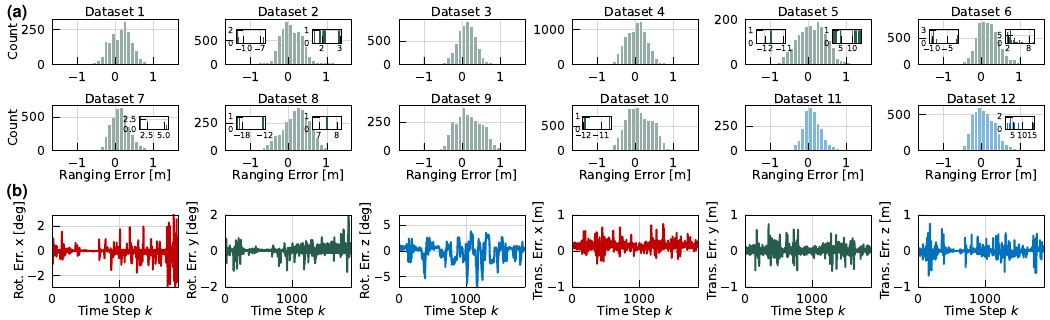}
  \caption{Sensor noise analysis. (a) The histograms of UWB ranging errors. The first ten histograms, highlighted in green, are obtained from the evaluation datasets, whereas the final two, shown in blue, are from the training datasets. (b) The rotational and translational errors of odometry in dataset $10$.}
  \label{fig_experimet_error}
\end{figure*}

\begin{table}[!t]
\centering
    \scriptsize
  \caption{Time Steps of Datasets in Experiments}
  \vspace{-1ex}
  \label{tab_data_size}
  \setlength{\tabcolsep}{2.3mm}{
    \renewcommand{\arraystretch}{1.1}
\begin{tabular}{cccccc}\Xhline{0.75pt}\noalign{\vskip 0.3mm}
Dataset & Time Steps & Dataset & Time Steps & Dataset & Time Steps \\\noalign{\vskip 0.3mm}\hline\noalign{\vskip 0.3mm}
1 & 343 & 5 & 724 & 9 & 703 \\
2 & 1216 & 6 & 1337 & 10 & 1880  \\
3 & 970 & 7 & 705 & 11 & 425 \\
4 & 1504 & 8 & 875 & 12 & 853 \\\noalign{\vskip 0.3mm}\Xhline{0.75pt}
\end{tabular}}
\end{table}

\begin{table*}[!t]
\centering
\begin{threeparttable}
    \scriptsize
  \caption{The Rotational and Translational ARMSEs in Experiments}
  \vspace{-1ex}
  \label{tab_exp1_r}
    \renewcommand{\arraystretch}{1.1}
        \setlength{\aboverulesep}{0pt}
  \setlength{\belowrulesep}{0pt}
\begin{tabular}{
C{6mm}
C{9mm}
C{9mm}
C{12mm}
C{9mm}
C{10mm}
C{10mm}
C{10mm}
C{13.4mm}
C{10mm}
C{12.2mm}
C{9.5mm}
}
\Xhline{0.75pt}\noalign{\vskip 0.3mm}
\multirow{2}{*}{Dataset} & \multicolumn{2}{c}{Baselines} & \multicolumn{7}{c}{Proposed Variants} & \multicolumn{2}{c}{Ours}\\\cmidrule(lr){2-3} \cmidrule(lr){4-10} \cmidrule(lr){11-12}\noalign{\vskip 0.3mm}
 & GTSAM & MFVI & GBP-DCS & GBP-L & MP-L & GBP-S & MP-S & GBP-NF-SW & GBP-NF & MP-NF-SW & MP-NF \\\noalign{\vskip 0.3mm}\hline\noalign{\vskip 0.5mm}

\hline\noalign{\vskip 0.3mm}
\multicolumn{12}{c}{Rotational ARMSEs [deg]}\\\noalign{\vskip 0.3mm}\hline\noalign{\vskip 0.3mm}
1 & 7.310 & \textbf{7.191} & 7.441 & 7.438 & 7.552 & 7.459 & 7.576 & 7.409 & 7.369 & 7.371 & 7.408 \\
\rowcolor{black!8}2 & 13.810 & 16.514 & 13.886 & 18.533 & 17.786 & 18.348 & 16.970 & 13.826 & 13.925 & 13.996 & \textbf{13.756} \\
3 & 11.501 & 12.939 & 11.717 & 11.714 & 11.805 & 11.810 & 11.720 & 11.330 & 11.325 & 11.259 & \textbf{11.108} \\
4 & 11.576 & 15.320 & 11.833 & 11.824 & 13.100 & 12.138 & 12.649 & 11.556 & 11.532 & \textbf{11.177} & 11.272 \\
\rowcolor{black!8}5 & 14.728 & 17.303 & 15.398 & 18.578 & 15.792 & 19.200 & 15.544 & 14.333 & 14.054 & 14.009 & \textbf{13.994} \\
\rowcolor{black!8}6 & 10.442 & 16.163 & 11.483 & 12.814 & 19.554 & 15.250 & 14.688 & 10.574 & 10.343 & 10.221 & \textbf{10.169} \\
\rowcolor{black!8}7 & 14.077 & 15.157 & \textbf{13.803} & 13.818 & 14.189 & 13.925 & 14.124 & 14.214 & 14.069 & 14.087 & 14.143 \\
\rowcolor{black!8}8 & \textbf{11.858} & 15.115 & 12.396 & 14.457 & 13.494 & 14.411 & 13.635 & 12.424 & 12.272 & 12.274 & 12.097 \\
9 & 7.951 & 9.016 & 8.359 & 8.333 & 8.852 & 8.453 & 8.797 & \textbf{7.837} & 7.849 & 7.838 & 7.855 \\
\rowcolor{black!8}10 & \textbf{7.324} & 12.004 & 8.004 & 8.195 & 8.507 & 8.633 & 8.816 & 7.337 & 7.400 & 7.339 & 7.350 \\\noalign{\vskip 0.3mm}\hline\noalign{\vskip 0.3mm}
Total & 11.219 & 14.379 & 11.614 & 13.032 & 13.900 & 13.509 & 12.915 & 11.244 & 11.182 & 11.108 & \textbf{11.056} \\ \noalign{\vskip 0.3mm}\hline\noalign{\vskip 0.5mm}

\hline\noalign{\vskip 0.3mm}
\multicolumn{12}{c}{Translational ARMSEs [m]}\\
\noalign{\vskip 0.3mm}\hline\noalign{\vskip 0.3mm}
1 & 0.579 & 0.763 & 0.631 & 0.635 & 0.656 & 0.638 & 0.644 & 0.570 & 0.559 & 0.558 & \textbf{0.557} \\
\rowcolor{black!8}2 & 0.585 & 0.947 & 0.624 & 0.775 & 0.687 & 0.770 & 0.675 & 0.569 & 0.571 & 0.576 & \textbf{0.560} \\
3 & 0.538 & 0.740 & 0.553 & 0.554 & 0.581 & 0.553 & 0.567 & 0.520 & 0.516 & 0.510 & \textbf{0.510} \\
4 & 0.540 & 0.909 & 0.561 & 0.561 & 0.593 & 0.557 & 0.585 & 0.542 & 0.540 & 0.534 & \textbf{0.534} \\
\rowcolor{black!8}5 & 0.670 & 0.844 & 0.624 & 1.194 & 0.837 & 1.138 & 0.788 & 0.594 & 0.575 & \textbf{0.573} & 0.576 \\
\rowcolor{black!8}6 & 0.630 & 1.073 & 0.618 & 0.886 & 0.777 & 0.901 & 0.753 & 0.582 & 0.562 & \textbf{0.557} & 0.558 \\
\rowcolor{black!8}7 & 0.552 & 0.772 & 0.573 & 0.583 & 0.608 & 0.571 & 0.598 & 0.548 & 0.546 & \textbf{0.540} & 0.549 \\
\rowcolor{black!8}8 & 0.627 & 0.863 & 0.605 & 1.121 & 0.696 & 1.132 & 0.738 & 0.575 & 0.561 & 0.558 & \textbf{0.557} \\
9 & 0.541 & 0.840 & 0.557 & 0.558 & 0.588 & 0.565 & 0.589 & 0.543 & 0.539 & \textbf{0.533} & 0.536 \\
\rowcolor{black!8}10 & 0.547 & 0.957 & 0.577 & 0.682 & 0.609 & 0.682 & 0.605 & 0.549 & 0.545 & 0.543 & \textbf{0.541} \\\noalign{\vskip 0.3mm}\hline\noalign{\vskip 0.3mm}
Total & 0.578 & 0.905 & 0.590 & 0.774 & 0.664 & 0.771 & 0.654 & 0.558 & 0.551 & 0.548 & \textbf{0.546} \\\noalign{\vskip 0.3mm}\Xhline{0.75pt}
\end{tabular}

\begin{tablenotes}[flushleft]
\scriptsize
\item[] \textit{Note:} Shaded rows indicate datasets containing evident UWB ranging outliers.
\end{tablenotes}
\end{threeparttable}
\end{table*}

In field experiments, three self-developed ASVs were used for data collection. The experimental field and ASVs are shown in Fig. \ref{field}. The test area is an irregularly shaped lake with a bridge spanning across it. The shoreline and the bridge introduce interference to the UWB ranging between the ASVs. Each ASV is equipped with an NVIDIA Jetson AGX Orin with $12$-core CPU, $2048$-core GPU, and $64$ GB RAM as an onboard computer. The STM32F407 micro-controller (MCU) drives the dual propellers of ASV. Communication between the ASVs and the shore-based computer relies on $1.4$ GHz wireless MESH modules, supporting data rates up to $90$ Mbps. Each ASV carries an FDIsystems DETA100D4G module integrating dual-antenna real-time kinematic (RTK) GNSS and an inertial navigation system (INS), delivering ground truth positioning with $0.8$ cm + $1$ ppm accuracy. The body frame of each ASV is fixed on its INS. Since we collected ground truths when the integrated navigation system was in RTK mode, we simulated GNSS data with lower accuracy mode by adding zero-mean Gaussian noises with covariance $\mathbf{P}^{\mathcal{G}}_{k,n}={\rm diag}(1,1,1.5)$ to the ground truth. The extrinsic parameter $\mathbf{p}^{b_n}_{g_n}=[-0.363,0,-0.159]^{\top}$ m. A RoboSense Helios 32 LiDAR is installed on each ASV to acquire point clouds for LiDAR odometry (LO), implemented using the LeGO-LOAM \cite{biblegoloam}. The results of LO are transformed to the body frame. As shown in Fig. \ref{fig_LO}, the resulting trajectories are stable and consistent, providing reliable inputs for subsequent state estimation. Inter-ASV ranging is achieved via LinkTrackP-B UWB units with extrinsic parameter $\mathbf{p}^{b_n}_{u_n}=[0.0946,0,-0.2028]^{\top}$ m, providing up to $500$ m range with typical accuracy of $0.1$ m. The sampling frequency of all measurement data is fixed at $10$ Hz. The collected data were divided into twelve datasets, with the first ten designated for evaluation and the remaining two for training. The time steps of each dataset are available in Table \ref{tab_data_size}. The histograms of UWB ranging errors are presented in Fig. \ref{fig_experimet_error} (a). While the ranging errors in certain datasets approximately follow a Gaussian distribution, others exhibit substantial outliers. The rotational and translational odometry measurement errors of the representative dataset $10$ are shown in Fig. \ref{fig_experimet_error} (b). The odometry errors exhibit evident temporal variations, with relatively small dispersion during some periods and substantially larger dispersion during others. This time-varying behavior observed in the real world data motivates the use of an adaptive noise covariance strategy for the odometry factors.

\begin{figure*}[tb]
  \centering
  \includegraphics[scale=1]{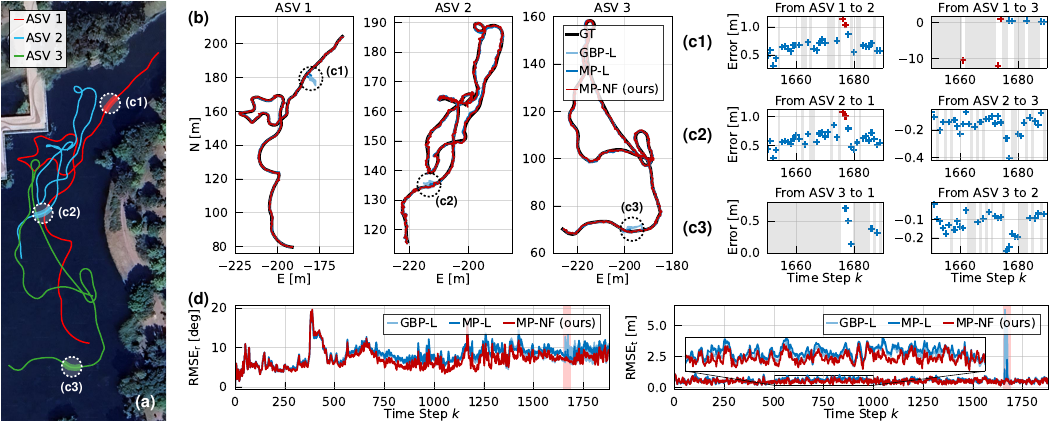}
  \caption{Results from dataset 10. (a) Ground truth trajectories of ASVs overlaid on the satellite map. The trajectories between time steps $1650$ and $1690$ are highlighted. (b) Ground truths and estimated ASV trajectories. GT curves denote ground truth trajectories. (c1), (c2), and (c3) The errors of UWB ranging measured by ASV 1, ASV 2, ASV 3 from time step $1650$ to $1690$, respectively. Errors exceeding $1$ m are highlighted in red, while the remaining errors are marked in blue. Shaded regions indicate measurement loss.  (d) Rotational and translational RMSEs, with shaded regions marking time step $1650$ to $1690$.}
  \label{fig_experimet_traj}
\end{figure*}

\begin{figure}[tb]
  \centering
  \includegraphics[scale=1]{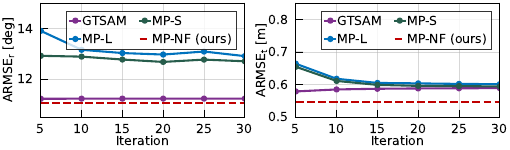}
  \caption{The rotational and translational errors with respect to iterations in experiments. MP-NF reaches its reported ARMSEs within 5 iterations.}
  \label{fig_experimet_2}
\end{figure}

We further introduce GBP-NF-SW, and MP-NF-SW. The only difference is that GBP-NF and MP-NF further fine-tune NF parameters trained in simulations on experiment data, whereas GBP-NF-SW and MP-NF-SW use parameters trained in simulations without fine-tuning. Sequences are extracted from dataset 11 and 12 with a length of $100$ time steps, and the final portion of data not extracted is discarded. The mini-batch size is set as $8$. NF parameters are updated every $5$ time steps. After $100$ steps, $2$ sequences are replaced and the mini-batch data is shuffled. This process is repeated $8$ times.

For all algorithms, maximum number of iterations $L=5$, sliding window size is set as 3, and GNSS noise covariance is set as $\mathbf{P}^{\mathcal{G}}_{k,n}={\rm diag}(1,1,1.5)$. For GTSAM and all GBP, the odometry and UWB noise covariance are set as $\mathbf{P}^{\mathcal{O}}_{k,n}={\rm diag}(10^{-5}\mathbf{I}_{3},5\times10^{-2}\mathbf{I}_{3})$ and $P^{\mathcal{U},m}_{k,n}=10^{-2}$, respectively. For MP-L, MP-S, MP-NF-SW, and MP-NF, the odometry noise prior parameters $t^{\mathcal{O}}_{1,n}=50$, $\mathbf{T}^{\mathcal{O}}_{1,n}=t^{\mathcal{O}}_{1,n}{\rm diag}(10^{-5}\mathbf{I}_{3},5\times10^{-2}\mathbf{I}_{3})$, forgetting parameter $\rho=0.999$, UWB noise prior parameters $a_{k,n}^{m}=0.8$, $\nu_{k,n}^{m}=7$, $P^{\mathcal{U},m}_{k,n}=10^{-2}$, and $P^{\mathcal{U},m}_{k,n,0}=4P^{\mathcal{U},m}_{k,n}$. The number of samples for sampling and NF-based algorithms is $4$.

The rotational and translational ARMSEs are reported in Table \ref{tab_exp1_r}. Although GBP-NF-SW and MP-NF-SW employ NF parameters trained on simulation data and extrinsic parameters are different in simulations and experiments, their performance differs only marginally from GBP-NF and MP-NF. This indicates that the proposed methods possess strong sim-to-real generalization capabilities, allowing rapid deployment on physical systems even when trained solely in simulation. A detailed analysis of algorithm performance across various datasets, combined with the UWB ranging outliers illustrated in Fig. \ref{fig_experimet_error} (a), shows that in datasets containing outliers, methods without robust design generally suffer significantly larger estimation errors. In contrast, the robust MP variants generally maintain more stable performance under these challenging conditions. This demonstrates that the robust treatment effectively mitigates the adverse effects of ranging outliers.

Fig. \ref{fig_experimet_traj} presents the results from dataset 10. In the presence of ranging outliers, the trajectories estimated by GBP-L exhibit pronounced deviations at certain points during time step $1650$ to $1690$, followed by MP-L, whereas MP-NF maintains closer alignment with the ground truth. Even outside this specific time interval, MP-NF consistently achieves the highest estimation accuracy overall. The time-series analysis above further demonstrates that the proposed MP-NF provides both superior estimation accuracy and improved robustness.

\begin{figure}[tb]
  \centering
  \includegraphics[scale=1]{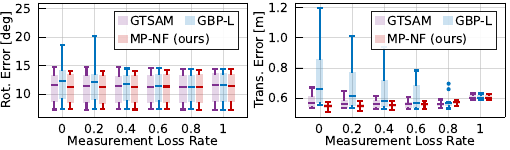}
  \caption{The box charts of rotational and translational errors with respect to different UWB measurement loss rates.}
  \label{fig_experimet_3}
\end{figure}

We further analyze the estimation performance with increased iterations. The errors of GTSAM, MP-L, and MP-S across different iterations, together with MP-NF under 5 iterations, are shown in Fig. \ref{fig_experimet_2}. As the number of iterations increases, the errors of MP-S and MP-L gradually decrease. However, neither they nor GTSAM reach the accuracy achieved by MP-NF with only 5 iterations.

Finally, we analyze how the algorithms behave as the UWB measurement loss rate increases. A portion of the UWB measurements is randomly discarded, and all algorithms are subsequently executed under these degraded conditions. The results are shown in Fig. \ref{fig_experimet_3}. As the loss rate increases, MP-NF generally maintains higher position accuracy than GTSAM and GBP-L, although its error exhibits an upward trend. When the loss rate reaches $1$, all algorithms degenerate into fusing only LO and GNSS. In contrast, GBP-L shows a noticeable decrease in estimation error as the loss rate increases. This indicates that, in practical engineering scenarios, relying on robust collaborative localization can improve estimation accuracy, whereas the absence of robustness may even degrade performance. Moreover, since the experimental dataset itself contains UWB dropouts and we also additionally remove measurements for validation, the underlying factor graph is substantially altered. The consistently reliable performance of MP-NF under these varying graph structures demonstrates its generalization capability.

\section{Conclusion}
\label{sec:con}
This paper has presented a unified MP algorithm that integrates GBP and MF approximation, where the former maintains dependencies among robot states, and the latter estimates noise statistics to enhance robustness and adaptability. The MP algorithm optimizes over the natural parameter space of candidate distributions, while non-conjugate terms arising from nonlinear measurement models have been addressed through gradient estimators constructed by linearization, sampling, or NF-based techniques. End-to-end training for NF parameters has enhanced estimation accuracy. The MP algorithm has also been extended to Lie groups. The algorithm has been validated in multirobot localization involving odometry, GNSS, and UWB, with both simulations and ASV experiments demonstrating the improvements in accuracy, robustness, adaptability, and good sim-to-real generalization across different operating conditions. Future work will generalize the NF-enhanced gradient estimators to broader classes of inference algorithms and investigate more tailored network architectures and feature representations for the learned gradient estimator to further improve the inference accuracy under various scenarios.

\appendices
\section{Examples of Exponential Family Distributions}
\label{Examples_Exponential_Families}
The base measures of the following distributions are all set as $1$. The Gaussian distribution ${\rm N}(\mathbf{x};\mathbf{m},\boldsymbol{\Sigma})$ with mean vector $\mathbf{m}\in\mathbb{R}^{n}$ and covariance matrix $\boldsymbol{\Sigma}\in\mathbb{S}_{++}^{n}$ can be written in a minimal exponential family form by defining natural parameter $\boldsymbol{\lambda}$ and sufficient statistics $\mathcal{T}(\mathbf{x})$ in tuple notation \cite{bibNGVI4}, \cite{bibNGVIKF1}:
\begin{align}
  \boldsymbol{\lambda}=\left[\begin{matrix}
        \boldsymbol{\Sigma}^{-1}\mathbf{m}\\
        -\frac{1}{2}\boldsymbol{\Sigma}^{-1}
    \end{matrix}\right],\ \mathcal{T}(\mathbf{x})=\left[\begin{matrix}
        \mathbf{x}\\
        \mathbf{x}\mathbf{x}^{\top}
    \end{matrix}\right].\notag
\end{align}
For notational clarity and to facilitate the derivations, we use vectorized representations of matrix natural parameters and sufficient statistics:
\begin{align}
  \boldsymbol{\lambda}=\left[\begin{matrix}
        \boldsymbol{\Sigma}^{-1}\mathbf{m}\\
        -\frac{1}{2}{\rm vec}\left(\boldsymbol{\Sigma}^{-1}\right)
    \end{matrix}\right],\ \mathcal{T}(\mathbf{x})=\left[\begin{matrix}
        \mathbf{x}\\
        {\rm vec}\left(\mathbf{x}\mathbf{x}^{\top}\right)
    \end{matrix}\right].\notag
\end{align}
The vectorized form of expectation parameters is
\begin{align}
  \boldsymbol{\mu}=\left[\begin{matrix}
        \mathbf{m}\\
        {\rm vec}\left(\boldsymbol{\Sigma}+\mathbf{m}\mathbf{m}^{\top}\right)
    \end{matrix}\right].\notag
\end{align}
Since the matrix-valued components are symmetric, their vectorized representations are understood to lie in ${\rm vec}(\mathbb{S}^{n})$. This vectorization is purely notational and preserves the underlying matrix inner products and operations.

The inverse Wishart distribution ${\rm IW}(\mathbf{R};t,\mathbf{T})$ with $\mathbf{T}\in\mathbb{S}_{++}^{n}$, and degrees of freedom $t>n+1$ also has a minimal exponential family form using tuple notation \cite{bibNGVIKF1}. We adopt its vectorized form
\begin{align}
  \boldsymbol{\lambda}=\left[\begin{matrix}
    -0.5(t+n+1)\\
    -0.5{\rm vec}(\mathbf{T})
  \end{matrix}\right],\ \mathcal{T}(\mathbf{R})=\left[\begin{matrix}
        \log{\rm det}(\mathbf{R})\\
        {\rm vec}\left(\mathbf{R}^{-1}\right)
    \end{matrix}\right].\notag
\end{align}

The Beta distribution ${\rm Beta}(\pi;a,b)$ with $a,b>0$ can be written in minimal exponential family form by defining
\begin{align}
  \boldsymbol{\lambda}=\left[\begin{matrix}
    a-1\\
    b-1
  \end{matrix}\right],\
  \mathcal{T}(\pi)=\left[\begin{matrix}
    \log \pi\\
    \log(1-\pi)
  \end{matrix}\right].\notag
\end{align}

The Bernoulli distribution ${\rm Bern}(y;\alpha)$ with $y\in\{0,1\}$ and $\alpha\in(0,1)$ can be written in minimal exponential family by
\begin{align}
  \boldsymbol{\lambda}=\log\frac{\alpha}{1-\alpha},\ \mathcal{T}(y)=y.\notag
\end{align}

The Gamma distribution ${\rm G}(x;\alpha,\beta)$ with $x,\alpha,\beta>0$ can be written in minimal exponential family by defining
\begin{align}
  \boldsymbol{\lambda}=\left[\begin{matrix}
    \alpha-1\\
    -\beta
  \end{matrix}\right],\
  \mathcal{T}(x)=\left[\begin{matrix}
    \log x\\
    x
  \end{matrix}\right].\notag
\end{align}

\section{Proof of Theorem {\ref{Theorem_stationary_points}}}
\label{Proof_stationary_points}
By setting $\nabla_{\boldsymbol{\mu}_{a}}\mathcal{L}=\mathbf{0},\forall a\in\mathcal{A}_{\rm BP}$, $\nabla_{\boldsymbol{\mu}_{i}}\mathcal{L}=\mathbf{0},\forall i\in\mathcal{I}$, and noticing that for a minimal exponential family defined as (\ref{ef}) with constant $h(\mathbf{x})$, we have $\nabla_{\boldsymbol{\mu}}H(p(\mathbf{x};\boldsymbol{\lambda}))=-\boldsymbol{\lambda}$ \cite{bibPML},
\begin{align}
    \boldsymbol{\lambda}_{a}=&\nabla_{\boldsymbol{\mu}_{a}}U_{a}+\sum_{i\in\mathcal{N}_{\rm BP}(a)}\mathbf{L}_{a,i}^{\top}\boldsymbol{\lambda}_{a,i},\ a\in\mathcal{A}_{\rm BP}\label{derivative_lambda_a}\\
    \boldsymbol{\lambda}_{i}=&\sum_{a\in\mathcal{N}(i)}\left(\boldsymbol{\lambda}_{i}-\boldsymbol{\lambda}_{a,i}\right),\ i\in\mathcal{I}_{\rm BP}\label{derivative_lambda_i_BP}\\
    \boldsymbol{\lambda}_{i}=&\sum_{a\in\mathcal{N}(i)}\nabla_{\boldsymbol{\mu}_{i}}U_{a},\ i\in\mathcal{I}_{\rm MF}\label{derivative_lambda_i_MF}
\end{align}
Denote the message from factor node $a$ to variable node $i\in\mathcal{N}(a)$ as $\mathbf{m}_{a\to i}$. For $i\in\mathcal{N}_{\rm BP}(a)$, we define
\begin{align}
  \mathbf{m}_{a\to i}\triangleq \boldsymbol{\lambda}_{i}-\boldsymbol{\lambda}_{a,i}\label{def_m}.
\end{align}
For $i\in\mathcal{N}_{\rm MF}(a)$, define $\mathbf{m}_{a\to i}\triangleq\nabla_{\boldsymbol{\mu}_{i}}U_{a}$ as in (\ref{belief2m_mf_bpmf}). Then, (\ref{derivative_lambda_i_BP}) and (\ref{derivative_lambda_i_MF}) can be written as (\ref{lambda_i_bpmf}). We further define the message from variable node $i\in\mathcal{I}_{\rm BP}$ to factor node $a\in\mathcal{N}(i)$, denoted by $\mathbf{n}_{i\to a}$, as in (\ref{m2n_bpmf}). Due to (\ref{lambda_i_bpmf}), (\ref{m2n_bpmf}), and (\ref{def_m}),
\begin{align}
  \boldsymbol{\lambda}_{a,i}=\mathbf{n}_{i\to a},\ \forall i\in\mathcal{I}_{\rm BP},a\in\mathcal{N}(i).\label{multiplier_n}
\end{align}
Substituting (\ref{multiplier_n}) into (\ref{derivative_lambda_a}), the equation (\ref{lambda_a_bpmf}) is verified. The equation constraint $\boldsymbol{\mu}_{i}=\mathbf{L}_{a,i}\boldsymbol{\mu}_{a}$ is equivalent to $\boldsymbol{\lambda}_{i}=g\left(\mathbf{L}_{a,i}f\left(\boldsymbol{\lambda}_{a}\right)\right)$. Substituting this and (\ref{multiplier_n}) into (\ref{def_m}), the equation (\ref{n2m_bpmf}) is verified. The proof that stationary points of Lagrangian (\ref{Lagrangian}) are fixed points fulfilling (\ref{lambda_a_bpmf}) and (\ref{lambda_i_bpmf}) is completed. Additionally, all above steps are reversible. This completes the proof of Theorem \ref{Theorem_stationary_points}.

\section{Proof of Lemma {\ref{Lemma_simplified_U}}}
\label{Proof_simplified_U}
Substituting factor (\ref{factor_model}) into definition (\ref{U_a}), one has
\begin{align}
  U_{a}=&-\frac{1}{2}\int b_{a}\left(\mathbf{x}_{a}\right)b_{a}(\boldsymbol{\phi}_{a})\left(\mathbf{r}_{a}^{\top}(\mathbf{x}_{a})\mathbf{R}_{a}^{-1}(\boldsymbol{\phi}_{a})\mathbf{r}_{a}(\mathbf{x}_{a})\right.\notag\\
  &-2\log c(\boldsymbol{\phi}_{a})\big){\rm d}\mathbf{x}_{a}{\rm d}\boldsymbol{\phi}_{a}\notag\\
  =&-\frac{1}{2}\int b_{a}\left(\mathbf{x}_{a}\right)\mathbf{r}_{a}^{\top}(\mathbf{x}_{a})\bar{\mathbf{R}}_{a}^{-1}\mathbf{r}_{a}(\mathbf{x}_{a}){\rm d}\mathbf{x}_{a}+c_{1}\notag\\
  =&\int b_{a}\left(\mathbf{x}_{a}\right)\log{\rm N}\left(\mathbf{r}_{a}\left(\mathbf{x}_{a}\right);\mathbf{0},\bar{\mathbf{R}}_{a}\right){\rm d}\mathbf{x}_{a}+c,\notag
\end{align}
where $c_{1}$ is a constant independent of $\mathbf{x}_{a}$ and $\boldsymbol{\mu}_{a}$.

\section{Proof of Lemma \ref{Lemma_gradient_mu}}
\label{Proof_gradient_mu}
Let $\boldsymbol{\mu}_{a}={\rm col}(\boldsymbol{\mu}_{a,1},\boldsymbol{\mu}_{a,2})$, where $\boldsymbol{\mu}_{a,1}$ has the same dimension as that of $\hat{\mathbf{x}}_{a}$. We have $\hat{\mathbf{x}}_{a}=\boldsymbol{\mu}_{a,1}$ and ${\rm vec}(\mathbf{P}_{a})=\boldsymbol{\mu}_{a,2}-{\rm vec}(\boldsymbol{\mu}_{a,1}\boldsymbol{\mu}_{a,1}^{\top})$. By the chain rule,
\begin{align}
  \nabla_{\boldsymbol{\mu}_{a,1}}U_{a}=&\frac{\partial \hat{\mathbf{x}}_{a}^{\top}}{\partial \boldsymbol{\mu}_{a,1}}\frac{\partial U_{a}}{\partial \hat{\mathbf{x}}_{a}}+\frac{\partial {\rm vec}^{\top}(\mathbf{P}_{a})}{\partial \boldsymbol{\mu}_{a,1}}\frac{\partial U_{a}}{\partial {\rm vec}(\mathbf{P}_{a})}\notag\\
  =&\frac{\partial U_{a}}{\partial \hat{\mathbf{x}}_{a}}-\left(\boldsymbol{\mu}_{a,1}\otimes\mathbf{I}+\mathbf{I}\otimes\boldsymbol{\mu}_{a,1}\right)^{\top}{\rm vec}\left(\frac{\partial U_{a}}{\partial \mathbf{P}_{a}}\right)\notag\\
  =&\frac{\partial U_{a}}{\partial \hat{\mathbf{x}}_{a}}-2\frac{\partial U_{a}}{\partial \mathbf{P}_{a}}\hat{\mathbf{x}}_{a},\notag
\end{align}
and
\begin{align}
  \nabla_{\boldsymbol{\mu}_{a,2}}U_{a}=&\frac{\partial {\rm vec}^{\top}(\mathbf{P}_{a})}{\partial \boldsymbol{\mu}_{a,2}}\frac{\partial U_{a}}{\partial {\rm vec}(\mathbf{P}_{a})}={\rm vec}\left(\frac{\partial U_{a}}{\partial \mathbf{P}_{a}}\right).\notag
\end{align}

\begin{IEEEbiography}[{\includegraphics[width=1in,height=1.25in,clip,keepaspectratio]{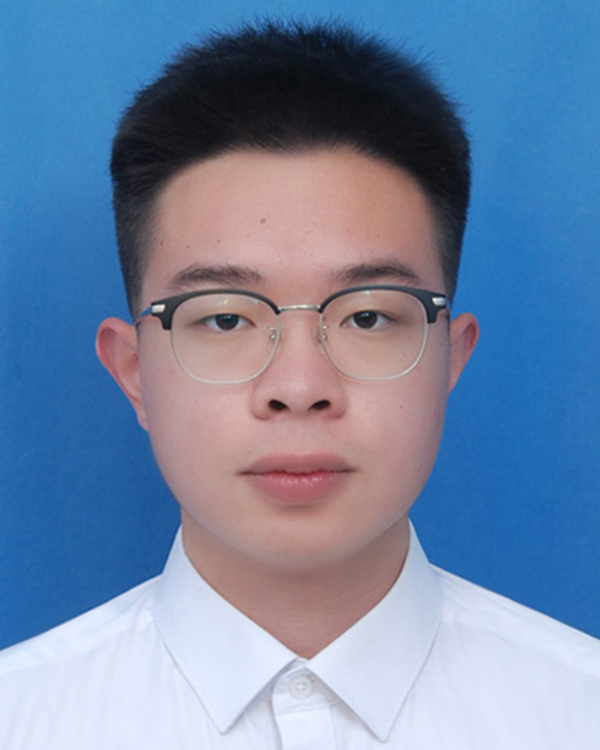}}]{Han Shen} (Student Member, IEEE) received the B.S. degree in automation from Hefei University of Technology, Hefei, China, in 2021, and the M.S. degree in cyber science and engineering from Southeast University, Nanjing, China, in 2024, where he is currently working toward the Ph.D. degree in systems science.

His current research interests include multirobot systems, variational inference, message passing, and state estimation.

Dr. Shen was the recipient of the Young Student Basic Research Project funded by the National Natural Science Foundation of China (NSFC), and the China Association for Science and Technology Talent Cultivation Project for Ph.D. Students.
\end{IEEEbiography}

\begin{IEEEbiography}[{\includegraphics[width=1in,height=1.25in,clip,keepaspectratio]{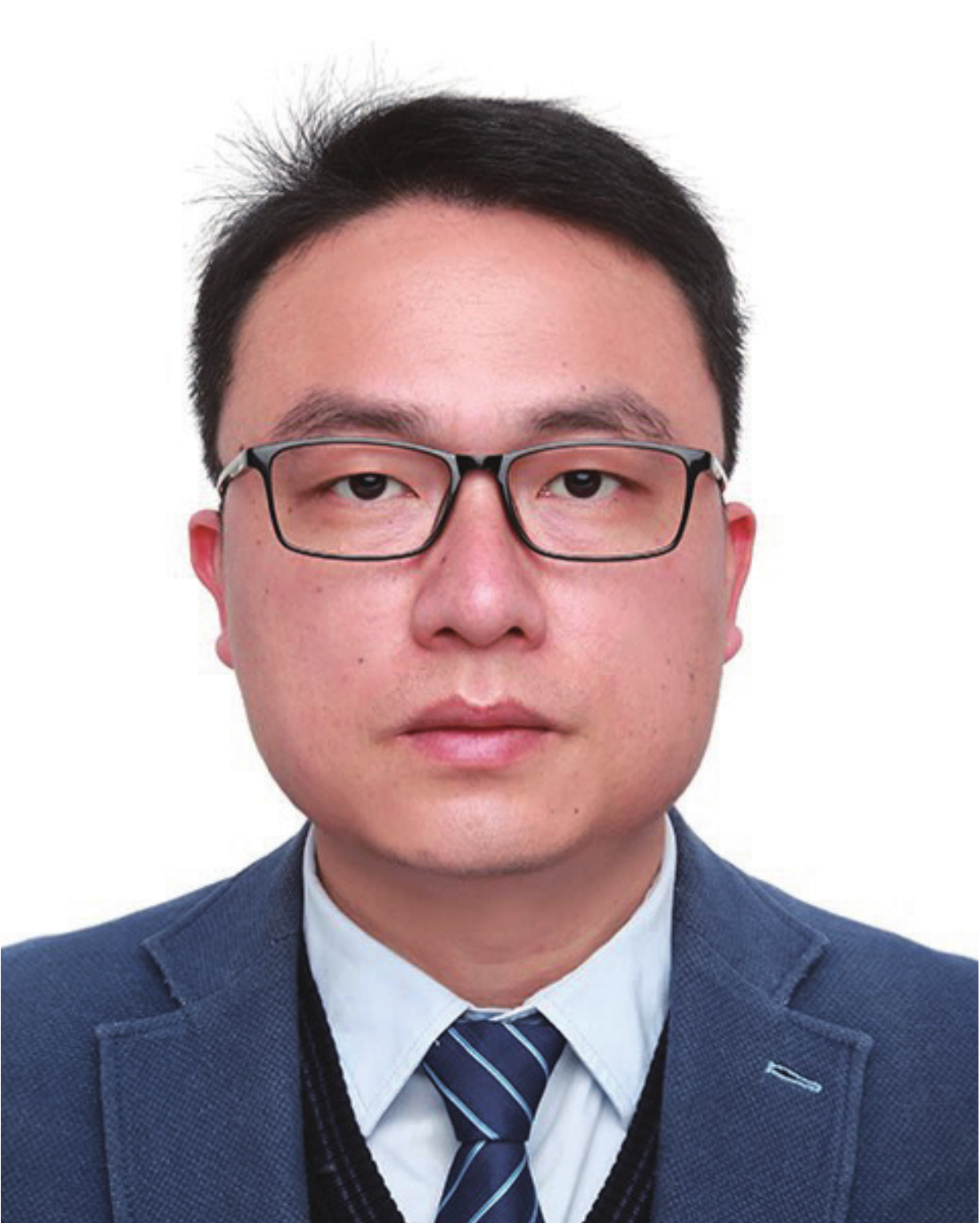}}]{Guanghui Wen}(Senior Member, IEEE) received the Ph.D. degree in mechanical systems and control from Peking University, Beijing, China, in 2012.

Currently, he is an Endowed Chair Professor and the Vice Dean of the School of Automation, Southeast University, Nanjing, China. His current research interests include autonomous intelligent systems, complex networked systems, distributed control and optimization, resilient control, and distributed reinforcement learning.

Prof. Wen was the recipient of the National Science Fund for Distinguished Young Scholars, the China Youth Science and Technology Award, the Australian Research Council Discovery Early Career Researcher Award, and the Asia Pacific Neural Network Society Young Researcher Award. He currently serves as the Editor-in-Chief of \textit{Embodied Intelligence} and a Technical Editor of the \textsc{IEEE/ASME Transactions on Mechatronics}. He is also an Associate Editor for multiple journals, including \textit{Automatica}, the \textsc{IEEE Transactions on Industrial Informatics}, the \textsc{IEEE Transactions on Neural Networks and Learning Systems}, the \textsc{IEEE Transactions on Control of Network Systems}, etc. He was the recipient of the Outstanding Associate Editor Award of \textsc{IEEE Transactions on Industrial Informatics} and \textit{Asian Journal of Control}. He is an IET Fellow.
\end{IEEEbiography}

\begin{IEEEbiography}[{\includegraphics[width=1in,height=1.25in,clip,keepaspectratio]{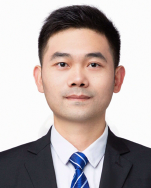}}]{Liangming Chen} (Senior Member, IEEE) received the B.E. degree in automation from Southwest Jiaotong University, Chengdu, China, in 2015. 
  
From 2015 to 2021, he was enrolled in the Ph.D. Programs of Harbin Institute of Technology, Harbin, China, and University of Groningen, Groningen, The Netherlands. From 2021 to 2022, he was a Research Fellow with Nanyang Technological University, Singapore. He is currently an Associate Professor at Southern University of Science and Technology, Shenzhen, China. His research interests include formation control and cooperative localization of multiagent systems.

Dr. Chen is an Associate Editor of \textsc{IEEE Transactions on Robotics} and \textsc{IEEE Transactions on Systems, Man and Cybernetics: Systems}.
\end{IEEEbiography}

\begin{IEEEbiography}[{\includegraphics[width=1in,height=1.25in,clip,keepaspectratio]{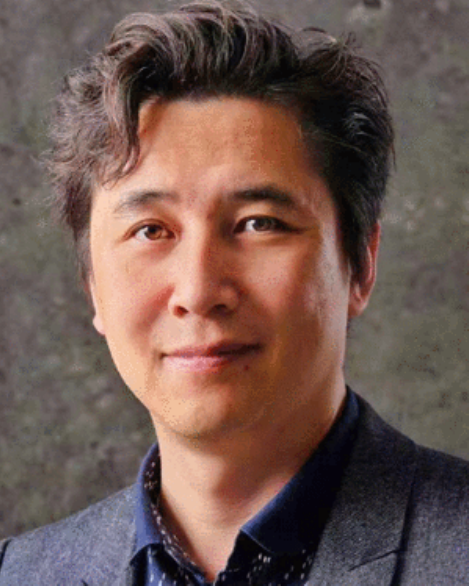}}]{Ming Cao} (Fellow, IEEE) received the bachelor's and master's degrees in electrical engineering from Tsinghua University, Beijing, China, in 1999 and 2002, respectively, and the Ph.D. degree in electrical engineering from Yale University, New Haven, CT, USA, in 2007.

Since 2016, he has been a Professor of networks and robotics with the Engineering and Technology Institute (ENTEG), University of Groningen, Groningen, The Netherlands, where he started as an Assistant Professor in 2008. Since 2022, he has been the Director of the Jantina Tammes School of Digital Society, Technology and AI, University of Groningen. From 2007 to 2008, he was a Research Associate with Princeton University, Princeton, NJ, USA. He was a Research Intern in 2006 with the IBM T. J. Watson Research Center, USA. His research interests include autonomous robots and multi-agent systems, complex networks and decision-making processes.

Dr. Cao was the recipient of the 2017 and inaugural Manfred Thoma medal from the International Federation of Automatic Control (IFAC) and the 2016 recipient of the European Control Award sponsored by the European Control Association (EUCA). He is a Senior Editor of \textit{Systems and Control Letters}, and has been an Associate Editor for \textsc{IEEE Transactions on Automatic Control}, \textsc{IEEE Transactions on Control of Network Systems}, \textsc{IEEE Transactions on Neural Networks and Learning Systems}, \textsc{IEEE Transactions on Circuits and Systems}, \textsc{IEEE Robotics \& Automation Magazine}, and \textsc{IEEE Circuits and Systems Magazine}. He is a Member of the IFAC Council.
\end{IEEEbiography}

\vfill


\begin{thebibliography}{00}

\bibitem{chen1}	L. Chen, C. Liang, S. Yuan, M. Cao, and L. Xie, ``Relative localizability and localization for multirobot systems," \textit{IEEE Trans. Robot.}, vol. 41, pp. 2931--2949, 2025.

\bibitem{chen2}L. Chen, R. Ruan, J. Li, W. Yuan, and Z. Lin, ``An angle-based control law for target location stabilization in 3-D space," \textit{IEEE Trans. Autom. Control}, vol. 71, no. 2, pp. 1310--1317, Feb. 2026.

\bibitem{PGO2}Y. Tian, Y. Chang, F. Herrera Arias, C. Nieto-Granda, J. P. How, and L. Carlone, ``Kimera-Multi: Robust, distributed, dense metric-semantic SLAM for multi-robot systems," \textit{IEEE Trans. Robot.}, vol. 38, no. 4, pp. 2022--2038, Aug. 2022.

\bibitem{PGO4}X. Liu, J. Lei, A. Prabhu, Y. Tao, I. Spasojevic, P. Chaudhari, N. Atanasov, and V. Kumar, ``SlideSLAM: Sparse, lightweight, decentralized metric-semantic SLAM for multirobot navigation," \textit{IEEE Trans. Robot.}, vol. 41, pp. 6529--6548, 2025.

\bibitem{FGO2}F. Zhu, Y. Ren, L. Yin, F. Kong, Q. Liu, R. Xue, W. Liu, Y. Cai, G. Lu, H. Li, and F. Zhang, ``Swarm-LIO2: Decentralized efficient LiDAR-inertial odometry for aerial swarm systems," \textit{IEEE Trans. Robot.}, vol. 41, pp. 960--981, 2025.

\bibitem{FGO3}H. Xu, Y. Zhang, B. Zhou, L. Wang, X. Yao, G. Meng, and S. Shen, ``Omni-Swarm: A decentralized omnidirectional visual-inertial-UWB state estimation system for aerial swarms," \textit{IEEE Trans. Robot.}, vol. 38, no. 6, pp. 3374--3394, Dec. 2022.

\bibitem{ADMM}H. Xu, P. Liu, X. Chen, and S. Shen, ``$D^{2}$SLAM: Decentralized and distributed collaborative visual-inertial SLAM system for aerial swarm," \textit{IEEE Trans. Robot.}, vol. 40, pp. 3445--3464, 2024.

\bibitem{BP1}F. Meyer, O. Hlinka, H. Wymeersch, E. Riegler, and F. Hlawatsch, ``Distributed localization and tracking of mobile networks including noncooperative objects," \textit{IEEE Trans. Signal Inf. Process. Netw.}, vol. 2, no. 1, pp. 57--71, Mar. 2016.

\bibitem{BP2}Y. Li, W. Yu, and X. Guan, ``Hybrid TOA-AOA cooperative localization for multiple AUVs in the absence of anchors," \textit{IEEE Trans. Ind. Inform.}, vol. 20, no. 2, pp. 2420--2431, Feb. 2024.

\bibitem{BP3}R. Murai, J. Ortiz, S. Saeedi, P. H. J. Kelly, and A. J. Davison, ``A robot web for distributed many-device localization," \textit{IEEE Trans. Robot.}, vol. 40, pp. 121--138, 2024.

\bibitem{BP4}R. Murai, I. Alzugaray, P. H. J. Kelly, and A. J. Davison, ``Distributed simultaneous localisation and auto-calibration using Gaussian belief propagation," \textit{IEEE Robot. Autom. Lett.}, vol. 9, no. 3, pp. 2136--2143, Mar. 2024.

\bibitem{MFVI1}Y. W. Lv, G. H. Yang, and S. Wasly, ``Cooperative localization and target tracking in mobile sensor networks with uncertainties in measurement models," \textit{IEEE Trans. Veh. Technol.}, vol. 72, no. 12, pp. 15521--15534, Dec. 2023.

\bibitem{bibVMP1}J. Winn and C. M. Bishop, ``Variational message passing," \textit{J. Mach. Learn. Res.}, vol. 6, pp. 661--694, 2005.

\bibitem{bibVMP2}J. Dauwels, ``On variational message passing on factor graphs," in
\textit{Proc. IEEE Int. Symp. Inf. Theory}, 2007, pp. 2546--2550.

\bibitem{bibBPMF1}E. Riegler, G. E. Kirkelund, C. N. Manchon, M.-A. Badiu, and B. H. Fleury, ``Merging belief propagation and the mean field approximation: A free energy approach," \textit{IEEE Trans. Inf. Theory}, vol. 59, no. 1, pp. 588--602, Jan. 2013.

\bibitem{DCS}P. Agarwal, G. D. Tipaldi, L. Spinello, C. Stachniss, and W. Burgard, ``Robust map optimization using dynamic covariance scaling," in \textit{Proc. 2013 IEEE Int. Conf. Robot. Automat.}, 2013, pp. 62--69.

\bibitem{bibVBAKF2}X. Lv, P. Duan, Z. Duan, G. Chen, and L. Shi, ``Stochastic event-triggered variational Bayesian filtering," \textit{IEEE Trans. Autom. Control}, vol. 68, no. 7, pp. 4321--4328, Jul. 2023.

\bibitem{bibVBRKF1}Y. Huang, Y. Zhang, B. Xu, Z. Wu, and J. Chambers, ``A new outlier-robust Student's \textit{t} based Gaussian approximate filter for cooperative localization," \textit{IEEE/ASME Trans. Mechatronics}, vol. 22, no. 5, pp. 2380--2386, Oct. 2017.

\bibitem{bibVBRKF2}Y. Huang, Y. Zhang, Y. Zhao, and J. A. Chambers, ``A novel robust Gaussian-Student's \textit{t} mixture distribution based Kalman filter," \textit{IEEE Trans. Signal Process.}, vol. 67, no. 13, pp. 3606--3620, Jul. 2019.

\bibitem{bibVBRKF3}H. Shen, G. Wen, Y. Lv, and J. Zhou, ``A stochastic event-triggered robust unscented Kalman filter-based USV parameter estimation," \textit{IEEE Trans. Ind. Electron.}, vol. 71, no. 9, pp. 11272--11282, Sep. 2024.

\bibitem{bibVBRKF4}H. Shen and G. Wen, ``Robust collaborative dynamic parameter estimation for multirobot systems: A distributed variational inference-based approach," \textit{IEEE/ASME Trans. Mechatronics}, vol. 31, no. 2, pp. 2209--2220, Apr. 2026.

\bibitem{bibBPMF2}B. Cakmak, D. N. Urup, F. Meyer, T. Pedersen, B. H. Fleury, and F. Hlawatsch, ``Cooperative localization for mobile networks: A distributed belief propagation-mean field message passing algorithm," \textit{IEEE Signal Process. Lett.}, vol. 23, no. 6, pp. 828--832, Jun. 2016.

\bibitem{bibBPMF3}D. Zhang, X. Song, W. Wang, G. Fettweis, and X. Gao, ``Unifying message passing algorithms under the framework of constrained Bethe free energy minimization," \textit{IEEE Trans. Wireless Commun.}, vol. 20, no. 7, pp. 4144--4158, Jul. 2021.

\bibitem{bibNGVI1}J. Hensman, M. Rattray, and N. Lawrence, ``Fast variational inference in the conjugate exponential family," in \textit{Proc. Int. Conf. Neural Inf. Process. Syst.}, 2012, vol. 25, pp. 2888--2896.

\bibitem{bibNGVI2}M. E. Khan and W. Lin, ``Conjugate-computation variational inference: Converting variational inference in non-conjugate models to inferences in conjugate models," in \textit{Proc. Conf. Artif. Intell. Statist.}, 2017, pp. 878--887.

\bibitem{bibNGVI3}W. Lin, M. Schmidt, and M. E. Khan, ``Handling the positive-definite constraint in the Bayesian learning rule," in \textit{Proc. Int. Conf. Mach. Learn.}, 2020, pp. 6116--6126.

\bibitem{bibNGVI4}M. E. Khan and H. Rue, ``The Bayesian learning rule," \textit{J. Mach. Learn. Res.}, vol. 24, no. 281, pp. 1--46, 2023.

\bibitem{bibNGVIKF1}H. Lan, S. J. Zhao, J. J. Hu, Z. F. Wang, and J. Fu, ``Joint state estimation and noise identification based on variational optimization," \textit{IEEE Trans. Autom. Control}, vol. 70, no. 7, pp. 4500--4515, Jul. 2025.

\bibitem{bibNGVIKF2}H. Lan, S. J. Zhao, Y. X. Mao, Z. F. Wang, Q. Cheng, and Z. A. Liu, ``Noise adaptive Kalman filtering with stochastic natural gradient variational inference," \textit{IEEE Trans. Aerosp. Electron. Syst.}, vol. 61, no. 4, pp. 9959--9976, Aug. 2025.

\bibitem{bibGE}S. Mohamed, M. Rosca, M. Figurnov, and A. Mnih, ``Monte Carlo gradient estimation in machine learning," \textit{J. Mach. Learn. Res.}, vol. 21, no. 132, pp. 1--62, 2020.

\bibitem{bibSFGE}J. P. Kleijnen and R. Y. Rubinstein, ``Optimization and sensitivity analysis of computer simulation models by the score function method," \textit{Eur. J. Oper. Res.}, vol. 88, no. 3, pp. 413--427, Feb. 1996.

\bibitem{bibPGE}D. J. Rezende, S. Mohamed, and D. Wierstra, ``Stochastic backpropagation and approximate inference in deep generative models," in \textit{Proc. Int. Conf. Mach. Learn.}, 2014, pp. 1278--1286.

\bibitem{bibMVGE}G. C. Pflug. \textit{Optimization of Stochastic Models: The Interface between Simulation and Optimization}. Boston, MA, USA: Kluwer Academic Publishers, 1996.

\bibitem{bibNFS1}T. Müller, B. McWilliams, F. Rousselle, M. Gross, and J. Novák, ``Neural importance sampling," \textit{ACM Trans. Graph.}, vol. 38, no. 5, pp. 1--19, 2019.

\bibitem{bibNFS2}F. Noé, S. Olsson, J. Köhler, and H. Wu, ``Boltzmann generators: Sampling equilibrium states of many-body systems with deep learning," \textit{Science}, vol. 365, no. 6457, Art. no. eaaw1147, 2019.

\bibitem{bibNFS3}L. A. Kruse, A. Tzikas, H. Delecki, M. Arief, and M. J. Kochenderfer, ``Enhanced importance sampling through latent space exploration in normalizing flows," in \textit{Proc. AAAI Conf. Artif. Intell.}, 2025, pp. 17983--17989.

\bibitem{bibNF1}D. Rezende, and S. Mohamed, ``Variational inference with normalizing flows," in \textit{Proc. Int. Conf. Mach. Learn.}, 2015, pp. 1530--1538.

\bibitem{bibRealNVP}L. Dinh, J. Sohl-Dickstein, and S. Bengio, ``Density estimation using real NVP," in \textit{Proc. Int. Conf. Learn. Representations}, 2017.

\bibitem{bibNF2}G. Papamakarios, E. Nalisnick, D. J. Rezende, S. Mohamed, and B. Lakshminarayanan, ``Normalizing flows for probabilistic modeling and inference," \textit{J. Mach. Learn. Res.}, vol. 22, no. 57, pp. 1--64, 2021.

\bibitem{bibBP}J. S. Yedidia, W. T. Freeman, and Y. Weiss, ``Constructing free-energy approximations and generalized belief propagation algorithms," \textit{IEEE Trans. Inf. Theory}, vol. 51, no. 7, pp. 2282--2312, Jul. 2005.

\bibitem{bibMFVI}C. Zhang, J. Butepage, H. Kjellstrom, and S. Mandt, ``Advances in variational inference," \textit{IEEE Trans. Pattern Anal. Mach. Intell.}, vol. 41, no. 8, pp. 2008--2026, Aug. 2019.

\bibitem{bibPML}K. P. Murphy, \textit{Probabilistic Machine Learning: Advanced Topics}. Cambridge, MA, USA: MIT Press, 2023.

\bibitem{bibestrob}T. D. Barfoot, \textit{State Estimation for Robotics}. Cambridge, U.K.: Cambridge Univ. Press, 2017.

\bibitem{bibLie}C. Forster, L. Carlone, F. Dellaert, and D. Scaramuzza, ``On-manifold preintegration for real-time visual-inertial odometry," \textit{IEEE Trans. Robot.}, vol. 33, no. 1, pp. 1--21, Feb. 2017.

\bibitem{bibGTSAM}F. Dellaert, ``Factor graphs and GTSAM: A hands-on introduction," Georgia Institute of Technology, Atlanta, Georgia, Tech. Rep. GT-RIM-CP\&R-2012-002, 2012.

\bibitem{bibPyTorch}A. Paszke et al., ``PyTorch: An imperative style, high-performance deep learning library," in \textit{Proc. Int. Conf. Neural Inf. Process. Syst.}, 2019, Art. no. 721.

\bibitem{bibPyTorch3D}J. Johnson et al., ``Accelerating 3D deep learning with PyTorch3D," in \textit{Proc. ACM SIGGRAPH Asia Courses}, 2020, Art. no. 10.

\bibitem{bibNumPy}C. R. Harris et al., ``Array programming with NumPy," \textit{Nature}, vol. 585, pp. 357--362, Sep. 2020.

\bibitem{bibSciPy}P. Virtanen et al., ``SciPy 1.0: Fundamental algorithms for scientific computing in Python," \textit{Nat. Methods}, vol. 17, pp. 261--272, Mar. 2020.

\bibitem{biblegoloam}T. Shan and B. Englot, ``LeGO-LOAM: Lightweight and ground optimized Lidar odometry and mapping on variable terrain," in \textit{Proc. 2018 IEEE/RSJ Int. Conf. Intell. Robots Syst.}, 2018, pp. 4758--4765.



\end{thebibliography}
\end{document}